\documentclass[10pt,journal,compsoc]{IEEEtran}
\usepackage{graphicx}
\usepackage{amsmath}
\usepackage{amssymb}
\usepackage{booktabs}
\usepackage{mathrsfs}

\usepackage{algorithm}
\usepackage{algorithmic}
\usepackage{multirow}
\usepackage[table,xcdraw]{xcolor}
\usepackage{amssymb}
\usepackage{bbding}
\usepackage{autobreak}
\usepackage{subfigure}
\usepackage{listings}
\usepackage{wrapfig}
\usepackage{booktabs}
\usepackage{amsmath}
\usepackage{threeparttable}
\usepackage{amsthm}

\newtheorem{Theorem}{Theorem}
\newtheorem{Lemma}{Lemma}

\newtheorem{Assumption}{Assumption}
\newtheorem{Definition}{Definition}

\newtheorem{Proposition}{Proposition}
\usepackage[breaklinks,colorlinks]{hyperref}

\usepackage[capitalize]{cleveref}

\newtheorem{assumption}{Assumption}
\newtheorem{lemma}{Lemma}

\ifCLASSOPTIONcompsoc
  \usepackage[nocompress]{cite}
\else
  \usepackage{cite}
\fi
\ifCLASSINFOpdf
\else
\fi
\begin{document}
%
% paper title
% Titles are generally capitalized except for words such as a, an, and, as,
% at, but, by, for, in, nor, of, on, or, the, to and up, which are usually
% not capitalized unless they are the first or last word of the title.
% Linebreaks \\ can be used within to get better formatting as desired.
% Do not put math or special symbols in the title.
\title{DAC-MR: Data Augmentation Consistency Based Meta-Regularization for Meta-Learning}
%
%
% author names and IEEE memberships
% note positions of commas and nonbreaking spaces ( ~ ) LaTeX will not break
% a structure at a ~ so this keeps an author's name from being broken across
% two lines.
% use \thanks{} to gain access to the first footnote area
% a separate \thanks must be used for each paragraph as LaTeX2e's \thanks
% was not built to handle multiple paragraphs
%
%
%\IEEEcompsocitemizethanks is a special \thanks that produces the bulleted
% lists the Computer Society journals use for "first footnote" author
% affiliations. Use \IEEEcompsocthanksitem which works much like \item
% for each affiliation group. When not in compsoc mode,
% \IEEEcompsocitemizethanks becomes like \thanks and
% \IEEEcompsocthanksitem becomes a line break with idention. This
% facilitates dual compilation, although admittedly the differences in the
% desired content of \author between the different types of papers makes a
% one-size-fits-all approach a daunting prospect. For instance, compsoc
% journal papers have the author affiliations above the "Manuscript
% received ..."  text while in non-compsoc journals this is reversed. Sigh.

\author{Jun~Shu,
	Xiang Yuan,
	Deyu Meng,
	and Zongben Xu% <-this % stops a space
	\IEEEcompsocitemizethanks{\IEEEcompsocthanksitem Jun Shu, Xiang Yuan, Deyu Meng (corresponding author) and Zongben Xu are with School of Mathematics and Statistics and Ministry of Education Key Lab of Intelligent Networks and Network Security, Xi'an Jiaotong University, Shaanxi, P.R.China.\protect\\
Email: {xjtushujun,relojeffrey}@gmail.com, {dymeng,zbxu}@mail.xjtu.edu.cn}}

\markboth{Journal of \LaTeX\ Class Files,~Vol.~14, No.~8, August~2015}%
{Shell \MakeLowercase{\textit{et al.}}: Bare Demo of IEEEtran.cls for Computer Society Journals}
% The only time the second header will appear is for the odd numbered pages
% after the title page when using the twoside option.
%
% *** Note that you probably will NOT want to include the author's ***
% *** name in the headers of peer review papers.                   ***
% You can use \ifCLASSOPTIONpeerreview for conditional compilation here if
% you desire.

% The publisher's ID mark at the bottom of the page is less important with
% Computer Society journal papers as those publications place the marks
% outside of the main text columns and, therefore, unlike regular IEEE
% journals, the available text space is not reduced by their presence.
% If you want to put a publisher's ID mark on the page you can do it like
% this:
%\IEEEpubid{0000--0000/00\$00.00~\copyright~2015 IEEE}
% or like this to get the Computer Society new two part style.
%\IEEEpubid{\makebox[\columnwidth]{\hfill 0000--0000/00/\$00.00~\copyright~2015 IEEE}%
%\hspace{\columnsep}\makebox[\columnwidth]{Published by the IEEE Computer Society\hfill}}
% Remember, if you use this you must call \IEEEpubidadjcol in the second
% column for its text to clear the IEEEpubid mark (Computer Society jorunal
% papers don't need this extra clearance.)

% use for special paper notices
%\IEEEspecialpapernotice{(Invited Paper)}

% for Computer Society papers, we must declare the abstract and index terms
% PRIOR to the title within the \IEEEtitleabstractindextext IEEEtran
% command as these need to go into the title area created by \maketitle.
% As a general rule, do not put math, special symbols or citations
% in the abstract or keywords.
\IEEEtitleabstractindextext{%
\begin{abstract}
Meta learning recently has been heavily researched and helped advance the contemporary machine learning. However, achieving well-performing meta-learning model requires a large amount of training tasks with high-quality meta-data representing the underlying task generalization goal, which is sometimes difficult and expensive to obtain for real applications. Current meta-data-driven meta-learning approaches, however, are fairly hard to train satisfactory meta-models with imperfect training tasks. To address this issue, we suggest a meta-knowledge informed meta-learning (MKIML) framework to improve meta-learning by additionally integrating compensated meta-knowledge into meta-learning process. We preliminarily integrate meta-knowledge into meta-objective via using an appropriate meta-regularization (MR) objective to regularize capacity complexity of the meta-model function class to facilitate better generalization on unseen tasks. As a practical implementation, we introduce data augmentation consistency to encode invariance as meta-knowledge for instantiating MR objective, denoted by DAC-MR. The proposed DAC-MR is hopeful to learn well-performing meta-models from training tasks with noisy, sparse or unavailable meta-data. We theoretically demonstrate that DAC-MR can be treated as a proxy meta-objective used to evaluate meta-model without high-quality meta-data. Besides, meta-data-driven meta-loss objective combined with DAC-MR is capable of achieving better meta-level generalization. 10 meta-learning tasks with different network architectures and benchmarks substantiate the capability of our DAC-MR on aiding meta-model learning. Fine performance of DAC-MR are obtained across all settings, and are well-aligned with our theoretical insights. This implies that our DAC-MR is problem-agnostic, and hopeful to be readily applied to extensive meta-learning problems and tasks.

\end{abstract}

% Note that keywords are not normally used for peerreview papers.
\begin{IEEEkeywords}
Data augmentation consistency, meta-regularization, meta-knowledge, meta learning, meta-data, generalization.
\end{IEEEkeywords}}

% make the title area
\maketitle

% To allow for easy dual compilation without having to reenter the
% abstract/keywords data, the \IEEEtitleabstractindextext text will
% not be used in maketitle, but will appear (i.e., to be "transported")
% here as \IEEEdisplaynontitleabstractindextext when the compsoc
% or transmag modes are not selected <OR> if conference mode is selected
% - because all conference papers position the abstract like regular
% papers do.
\IEEEdisplaynontitleabstractindextext
% \IEEEdisplaynontitleabstractindextext has no effect when using
% compsoc or transmag under a non-conference mode.

% For peer review papers, you can put extra information on the cover
% page as needed:
% \ifCLASSOPTIONpeerreview
% \begin{center} \bfseries EDICS Category: 3-BBND \end{center}
% \fi
%
% For peerreview papers, this IEEEtran command inserts a page break and
% creates the second title. It will be ignored for other modes.
\IEEEpeerreviewmaketitle

\IEEEraisesectionheading{\section{Introduction}\label{sec:introduction}}
% Computer Society journal (but not conference!) papers do something unusual
% with the very first section heading (almost always called "Introduction").
% They place it ABOVE the main text! IEEEtran.cls does not automatically do
% this for you, but you can achieve this effect with the provided
% \IEEEraisesectionheading{} command. Note the need to keep any \label that
% is to refer to the section immediately after \section in the above as
% \IEEEraisesectionheading puts \section within a raised box.

% The very first letter is a 2 line initial drop letter followed
% by the rest of the first word in caps (small caps for compsoc).
%
% form to use if the first word consists of a single letter:
% \IEEEPARstart{A}{demo} file is ....
%
% form to use if you need the single drop letter followed by
% normal text (unknown if ever used by the IEEE):
% \IEEEPARstart{A}{}demo file is ....
%
% Some journals put the first two words in caps:
% \IEEEPARstart{T}{his demo} file is ....
%
% Here we have the typical use of a "T" for an initial drop letter
% and "HIS" in caps to complete the first word.
% You must have at least 2 lines in the paragraph with the drop letter
% (should never be an issue)

\IEEEPARstart{M}{aching} learning has recently demonstrated impressive performance in various fields, e.g., computer vision \cite{he2016deep}, natural
language processing \cite{devlin2018bert}, speech processing \cite{abdel2014convolutional}, etc. However, an effective machine learning method often requires a large amount of high-quality labeled data to properly and sufficiently simulate the testing/evaluating distribution. Collecting such large-scale supervised datasets is notoriously expensive in time and effort for most real applications. Compared with current machine intelligence, humans are able to quickly learn novel concepts from only small amount of examples \cite{lake2015human,lake2017building}. The capability of machine to learn new concepts quickly from small examples is thus desirable, especially for many problems/applications where data are intrinsically rare or expensive, or compute resources are unavailable.

Meta-learning \cite{naik1992meta,schmidhuber1987evolutionary,thrun1998lifelong}, or learning to learn, has been suggested as a promising solution path to assemble machine learning with above capability. The key idea of meta-learning is to distill a meta-model from multiple learning tasks/episodes, and then use this meta-model to improve performance of task-specific model on novel query tasks \cite{hospedales2021meta, shu2021learning}.
Such a learning paradigm is hopeful to bring a variety of benefits, such as finely adapting to query tasks with less computation/data costs (e.g., avoid learning from scratch for novel tasks), as well as fewer human interventions.

Recently, it produces an explosion of researches on meta-learning, due to its potential to advance the frontier of the contemporary machine learning. Especially,
meta-learning has helped machine learning improve the data efficiency \cite{wang2020generalizing,shu2018small}, algorithm automation \cite{he2021automl,karmaker2021automl}, and generalization \cite{maurer2016benefit,tripuraneni2020theory,shu2021learning}. Successful
applications have been demonstrated in areas spanning few/zero-shot learning \cite{finn2017model,snell2017prototypical,vinyals2016matching,sung2018learning,soh2020meta}, neural architecture search (NAS) \cite{liu2018darts,elsken2019neural},  hyperparameter optimization \cite{franceschi2018bilevel}, curriculum learning \cite{ren2018learning,shu2019meta}, domain adaptation/generalization \cite{liu2021cycle,li2018learning}, transfer learning \cite{jang2019learning,sun2020meta},  label noise learning \cite{zheng2021meta,shu2020meta,zhao2021probabilistic}, semi-supervised learning \cite{pham2021meta},  unsupervised learning \cite{metz2018meta}, reinforcement Learning \cite{duan2016rl,wang2016learning}, data/label generation \cite{cubuk2018autoaugment,wu2021learning,shu2022cmw}, loss/regularization learning \cite{balaji2018metareg,yazdanpanah2022revisiting,lee2019meta,shu2020learning,huang2019addressing}, learning to optimize \cite{andrychowicz2016learning,ravi2016optimization,shu2022mlr}, and robustness \cite{collins2020task,killamsetty2022nested}, etc.

These successes largely attribute to the data-based nature of current meta-learning approaches that learn from a tremendous number of training tasks with high-quality meta-data representing the underlying task generalization goal.
However, in most real applications, collecting such high-quality training tasks are difficult, expensive and impractical. This often makes obtained training tasks imperfect. In fact, we always have access to problematic meta-data for some applications. For example, the corresponding ground-truth labels of meta-data are generally noisy in label noise problems \cite{shu2019meta}, or unavailable in unsupervised domain adaptation tasks \cite{liu2021cycle}, or the size of meta-data is limited in few-shot learning issues \cite{finn2017model}.
The purely meta-data-driven approaches tend to reach their limits or lead to unsatisfactory results under these imperfect circumstances.
With meta learning becoming more and more popular in real applications, there is also a growing need for meta-learning to train well-performing and sufficiently generalized meta-models from such imperfect training tasks.

\begin{figure} \vspace{-3mm}
	\centering
	\includegraphics[width=0.49\textwidth]{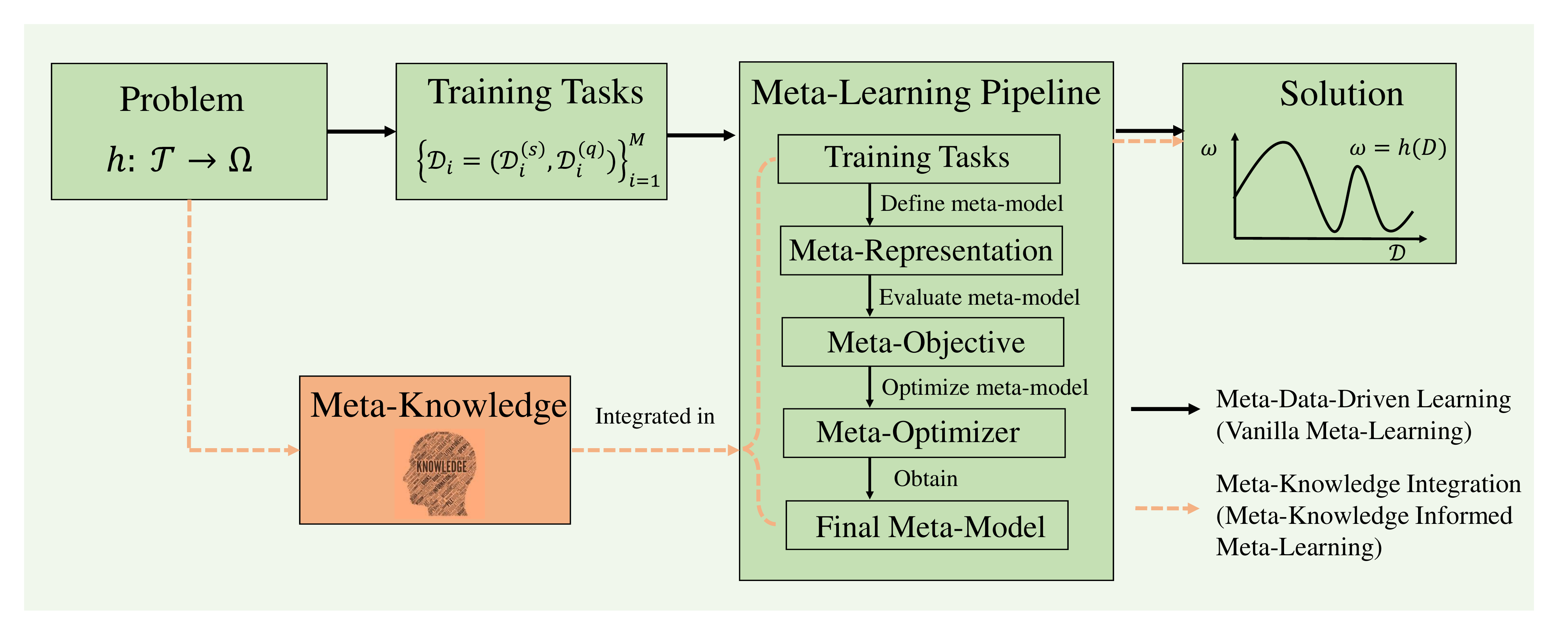}
	\vspace{-8mm}
	\caption{ Illustration of the meta-knowledge informed meta-learning framework. Specifically, an additional branch encoding certain beneficial meta-knowledge is integrated into the meta learning pipeline.
}\label{mkiml}  \vspace{-5mm}
\end{figure}

As a step towards addressing the limitations of purely meta-data-driven meta-learning, we suggest a meta-knowledge informed meta-learning (MKIML) framework, as shown in Fig. \ref{mkiml}, which comprises an additional meta-knowledge integration into the meta learning pipeline. Moreover, such meta-knowledge could be obtained in an external, separated way from the meta-learning problem and the usual training tasks. This framework is expected to be functional in exploring an orthogonal meta-knowledge-driven approaches relative to previous purely meta-data-driven approaches to learn and evaluate meta-model. With the MKIML framework, we attempt to integrate meta-knowledge into meta-objective by means of a meta-regularization (MR) term.
The key insight is that we leverage the benefits of fundamental properties of the meta-model for various training tasks, which should help achieve better generalization of meta-model to unseen tasks and alleviate the critical need of high-quality meta-data.
Specifically, in this study we instantiate MR with the data augmentation consistency (DAC) as a new meta-objective for meta-learning. The DAC stems from recent advances in semi-supervised learning \cite{xie2020unsupervised,sohn2020fixmatch}, and to the best of our knowledge, we exploit it to meta-regularize the complexity of meta-model function class for the first time, which enforces the model facilitated by meta-model to output similar predictions under input data augmentations.

Our contributions can be mainly summarized as follows.

1) We suggest a MKIML framework, as shown in Fig. \ref{mkiml}, aiming to improve capability of previous purely meta-data-driven meta-learning approaches by additionally integrating compensated meta-knowledge into meta learning process. Specifically, we explore to integrate meta-knowledge into meta-objective through designing an appropriate meta-regularizer (MR). The MR is functional on regularising the capacity complexity of meta-model function class, so as to improve its meta-level generalization on unseen tasks.

2) We introduce data augmentation consistency (DAC) to instantiate MR objective for an effective practical implementation (DAC-MR for brevity). The DAC-MR provides a general approach to help apply meta-learning models to tasks with noisy, sparse or unavailable meta-data. Besides, the DAC-MR is problem-agnostic, which can be generally applicable to extensive meta-learning problems and tasks.

3) We theoretically prove that the additional DAC-MR term in meta-objective can bring better meta-level generalization compared with solely meta-data-driven meta-loss objective. Meanwhile, we prove that DAC-MR is able to be regarded as a proxy meta-objective implicitly calculated on high-quality meta-data under some mild conditions.

4) We experimentally demonstrate that DAC-MR aids meta-model learning across various meta-learning problems in computer vision, including few-shot learning (\S\ref{section4}), transfer learning (\S\ref{section5}), continual learning (\S\ref{section6}) and label noise learning (\S\ref{section7}). Specifically, our DAC-MR is substantiated to be valid across 10 meta-learning tasks with different network architectures and testing benchmarks. Furthermore, these empirical results are well-aligned with our theoretical insights.

The paper is organized as follows. \S\ref{relatedwork} discusses related work. \S\ref{section3} presents the proposed MKIML framework, MR objective and our DAC-MR objective as a practical implementation for MKIML as well as its theoretical insights. We evaluate DAC-MR to few-shot learning in \S\ref{section4},
transfer learning in \S\ref{section5}, continual learning in \S\ref{section6} and label noise learning in \S\ref{section7}, respectively. The conclusion is finally made.

%In this paper, we introduce consistency regularization into meta-learning framework as a novel meta-regularizer for the first time to achieve better generalization for meta-model on novel tasks.
%
%We instantiate our meta-regularization framework, referred to as DAC-MR, with data augmentation consistency.
%
%To the best of our knowledge
%
%Yet, despite the motivational difference, each of these problems require learning methodology to be learned from the training tasks that allow for better generalization to novel tasks and data.
%
%
%
%Our approach is a simple, strong baseline that we hope will be used in future research.
%
%
%consequently impairing the performance on unseen tasks
%
%resulting in a severe deterioration of the predictive performance in practice
%
%have a detrimental effect on the performance of
%
%
%under-investigated
\section{Related Work} \label{relatedwork}
\textbf{Meta-Objective design.}
Most current methods define the meta objective using a meta dataset to compute the performance metric, after updating the task
model with the meta-model. This is in line with classic validation set approaches for hyperparameter and model selection. To adapt to various investigated problems, a large corpus of problem-specific meta-learning objectives are proposed, e.g.,
sample efficient few-shot learning \cite{finn2017model,snell2017prototypical}, fast computation \cite{andrychowicz2016learning,ravi2016optimization,shu2022mlr}, rapid online adaptation in non-stationary environment \cite{finn2019online,nagabandi2018deep}, catastrophic interference \cite{riemer2018learning,javed2019meta}, selective forgetting \cite{baik2021learning}, robustness to domain-shift \cite{li2018learning,balaji2018metareg}, label noise \cite{shu2019meta}, adversarial attack \cite{goldblum2020adversarially}. Yet, despite the good progress, these methods often use meta-objective that suits the problem at hand and sometimes tends to be unavailable when given meta-data are noisy or missing.
Instead, we aim to seek the solution of meta-model through possessing some properly defined fundamental meta-knowledge suitable for general meta-learning tasks, which allows for better generalization to unseen tasks and less reliance on meta-data.

\noindent\textbf{Regularization.} Regularization is an important technique in machine learning, which imposes a penalty on model's complexity, allowing for good generalization to unseen data even when being trained on a relatively small training set. Some popular regularizations, such as sparsity \cite{tibshirani1996regression}, low-rankness \cite{candes2009exact}, and smoothness ones \cite{belkin2006manifold}, are effective for eliminating over-fitting and enhance generalization of the learned model. Recently, some regularization methods are proposed to improve generalization of deep learning, e.g., early stopping, weight decay \cite{goodfellow2016deep}, dropout \cite{srivastava2014dropout}, batch normalization \cite{ioffe2015batch}. Besides, some data dependent regularizations make additional assumptions on model with respect to training data, e.g., data augmentation \cite{shorten2019survey}, adversarial training \cite{goodfellow2014explaining}, label smoothing \cite{szegedy2016rethinking}, mixup \cite{zhang2018mixup}, etc. Different from these regularizations aiming to control the \textit{model capacity} for improving its generalization on unseen data, our meta-regularization tries to control the \textit{meta-model capacity} for improving its generalization on unseen tasks.

\noindent\textbf{Consistency regularization.} The idea of consistency regularization has been studied in many settings. It generally enforces the model to output similar predictions under certain input transformations such as adversarial perturbations \cite{miyato2018virtual} and data augmentation \cite{xie2020unsupervised}, or model parameter space transformations such as temporal ensembling \cite{laine2016temporal} and mean teacher \cite{tarvainen2017mean}. Such regularization methods have been successfully applied to semi-supervised learning \cite{xie2020unsupervised,sohn2020fixmatch}, self-supervised learning \cite{chen2020simple,grill2020bootstrap}, unsupervised domain adaptation \cite{berthelot2021adamatch}, adversarial robustness \cite{carmon2019unlabeled,alayrac2019labels}, few-shot learning \cite{kye2020meta}, image generation \cite{zhang2019consistency,sinha2021consistency}, and transfer learning \cite{abuduweili2021adaptive}. Yet as far as we known, these ideas have not been exploited in meta-regularizing meta-models in meta-learning community.

\noindent\textbf{Meta-learning loss/regularization.} The main idea is to meta-learn proxy loss/regularization from data that improves inner-level model optimization from various task-specific goal perspectives, including model generalization \cite{balaji2018metareg,denevi2019learning,denevi2020advantage}, optimization efficiency \cite{li2019lfs,wang2020loss}, differentiable approximation to a true non-differentiable metric \cite{huang2019addressing}, unsupervised update rule \cite{metz2018meta}, robust to domain shift \cite{li2019feature}, label noise \cite{gao2021searching,shu2020learning,shu2020metav2,ding2023improve}, or adversarial attack \cite{goldblum2020adversarially}, and arising in generalizations of unsupervised learning \cite{antoniou2019learning}, self-supervised learning \cite{doersch2017multi}, auxiliary task learning \cite{jaderberg2016reinforcement,liu2019self}, etc. These methods, however, still overlook the meta-objective's design at outer-level learning. Comparatively, we steer the research interests towards designing a meta objective to meta-regularize meta-model's capacity for improving the meta-level generalization to unseen learning tasks.

\noindent\textbf{Knowledge informed machine learning.} The main idea is to integrate beneficial knowledge into machine learning pipelines, to help eliminate limitations of machine learning when it deals with insufficient training data \cite{von2021informed,deng2020integrating,hao2022physics}, hence increasing the reliability and robustness of the obtained model.
Comparatively, meta-knowledge informed meta-learning attempts to integrate useful meta-knowledge into meta-learning pipelines, to help purely meta-data-driven meta-learning approaches deal with imperfect training tasks, e.g., tasks with noisy, sparse or unavailable meta-data.
This framework thus focuses on higher outer-level learning beyond inner-level regular machine learning, and aims to achieve
better meta-level generalization, reliability and robustness of the learned meta-model.

\noindent\textbf{Meta-overfitting and meta-regularization.} Yin et al.,\cite{yin2019meta} found that learning a well-generalized initialization \cite{finn2017model} runs a high risk of inducing a sufficiently expressive initialization that memorizes all meta-training tasks. This phenomenon is called memorization meta-overfitting \cite{yin2019meta}, where meta-learned initialization solves the query set even without much relying on the support set for adaptation. This meta-overfitting meta-model then fails to generalize to meta-test tasks. To eliminate this issue, \cite{yin2019meta} proposed an information bottleneck constraint on the capacity of the initialization.
Afterwards, \cite{liu2020task,rajendran2020meta,ni2021data} presented task augmentation strategies and \cite{jamal2019task} imposed a unbiased task-agnostic prior to address meta-overfitting problem. \cite{shu2021learning} presented to control the range of meta-model's output as the meta-regularization strategy inspired from the derived statistical generalization bound to improve meta-level generalization of meta-model.
%Conversely, Shu et al. \cite{shu2022cmw} proposed the MW-Net algorithm \cite{shu2019meta} for robust deep learning issue, and found that it may suffer from meta-underfitting problem, in which the meta-model lacks the capacity of addressing heterogenous label noises. They further proposed a CMW-Net algorithm \cite{shu2022cmw} to enhance the meta-model's capacity of addressing complicated heterogenous label noises.
However, these methods are delicately designed for specific problems, e.g., few-shot learning, and then use the proposed problem-dependent MR to improve corresponding meta-algorithms.
This makes it hard to generalize such proposed MR to more extensive meta-learning tasks, and thus certainly lack generality among different meta-learning tasks.

Along this research line, the most related work to ours are PACOH \cite{rothfuss2021pacoh} and F-PACOH \cite{rothfuss2021meta}.  PACOH \cite{rothfuss2021pacoh} derives the PAC-optimal hyper-posterior using a KL-divergence between hyper-posterior and hyper-prior to serve as a meta-regularizer.
F-PACOH \cite{rothfuss2021meta} further defines the hyper-prior over the function space to address overconfident predictions in meta-learning.
However, the hyper-prior still needs to be delicately specified for the investigated problems.
Meanwhile, they are usually computationally prohibitive and cumbersome to meta-learning tasks with overparameterized DNNs and large-scale datasets, rendering this meta-level regularization regime always intractable.
Comparatively, our DAC-MR is relatively light-weight, simple and convenient to be implemented and problem-agnostic, which can be potentially applicable to evidently wider range of meta-learning problems and tasks.

\section{Methodology} \label{section3}
In this section, we firstly introduce the vanilla meta-learning model in \S\ref{section31}, and then present the MKIML framework in \S \ref{section321}. In \S\ref{section32}, we provide the MR objective to highlight the utility of MKIML framework and
then present the proposed DAC-MR strategy in \S\ref{section33}. Finally, we show the theoretical guarantee of DAC-MR in terms of improving vanilla meta-learning generalization (\S\ref{section34}), and behaving as a proxy meta-objective used to evaluate meta-model (\S\ref{section35}).

\vspace{-4mm}
\subsection{Preliminary and Vanilla Meta-Learning} \label{section31}
In this paper, we consider the following bi-level optimization formulation of meta-learning \cite{hospedales2021meta}:
\begin{align}
	h_{\phi^*}  = \mathop{\arg\min}_{h_{\phi} \in \mathcal{H}}& \sum_{i=1}^{M} \mathcal{L}^{meta} (\mathcal{D}_i^{(q)}; f_{\theta^*_i}(h_{\phi})),   \label{eqout} \\
	\text{s.t.,} \ f_{\theta^*_i}(h_{\phi}) &= \mathop{\arg\min}_{f_{\theta} \in \mathcal{F}} \mathcal{L}^{task} (\mathcal{D}_i^{(s)};f_{\theta}, h_{\phi} ),\label{eqin}
\end{align}
where $\theta$ and $\phi$ denote the parameters of task-specific model $f:\mathcal{X} \rightarrow \mathcal{Y}$ and meta-model $h:\mathcal{T} \rightarrow \Omega$, respectively, and $\mathcal{X}, \mathcal{Y}$ are the feature and label spaces, respectively,  $\mathcal{T}, \Omega$ are the task and meta-model output spaces, respectively, and $\mathcal{F}, \mathcal{H}$ are the function classes of model and meta-model, respectively. We assume that a set of $M$ training tasks $\mathcal{D}$ are sampled from task distribution $p(\mathcal{T})$, and the $i$-th training task $\mathcal{D}_i = (\mathcal{D}_i^{(s)},\mathcal{D}_i^{(q)})$ consisting of task training data $\mathcal{D}_i^{(s)}$ and meta data $\mathcal{D}_i^{(q)}$. We denote $x\in \mathcal{X}$ as input feature, and $y\in \mathcal{Y}$ as its label (or response), and $\mathcal{D}_i^{(q)} = \{(x_{ij}^{(q)},y_{ij}^{(q)})\}_{j=1}^{n}$, $\mathcal{D}_i^{(s)} = \{(x_{ij}^{(s)},y_{ij}^{(s)})\}_{j=1}^{m}$, where $m,$ $n$ are the sizes of $\mathcal{D}_i^{(s)},\mathcal{D}_i^{(q)}$, respectively. The train and meta datasets here are also respectively called support and query sets in the meta-learning literatures \cite{vinyals2016matching}.
$\omega \in \Omega$ are the hyperparameters in the machine learning \cite{hospedales2021meta,shu2021learning}, which are often pre-specified in the conventional assumption, while are meta-learned by meta-model $h_{\phi}$ from training tasks $\mathcal{D}$ in the meta-learning framework, i.e., $w = h_{\phi}(\mathcal{D})$. We will drop explicit dependence of $h_{\phi}$ on $\mathcal{D}$ for brevity in the following.
$\mathcal{L}^{meta}$ and $\mathcal{L}^{task}$ refer to the outer and inner learning objectives, respectively, such as cross entropy or mean square loss in the case of
classification or regression tasks, and $\mathcal{L}^{task}(\mathcal{D}_i^{(s)}; f_{\theta},h_{\phi}) \!=\! \frac{1}{m}\sum_{j=1}^{m} \mathcal{L}^{task}$ $\left(f_{\theta}(x_{ij}^{(s)}),y_{ij}^{(s)};h_{\phi}\right )$.

Such a bi-level optimization problem is capable of finely delivering the working mechanics of meta-learning. Specifically, at the inner-level learning, the task-specific model $f_{\theta^*_i}(h_{\phi})$ is trained on training dataset $\mathcal{D}_i^{(s)}$ using the meta-model $h_{\phi}$, behaving just like regular machine learning; and the outer-level optimization seeks the meta-model $h_{\phi}$ that ensures produced model $f_{\theta^*_i}(h_{\phi})$ to perform well on its corresponding meta dataset $\mathcal{D}_i^{(q)}$. As a whole, meta-learning aims to learn a shared meta-model that generalizes across $M$ tasks, and ideally this meta-model enables task-specific model to be learned better than from scratch for a new task.

As shown in Fig. \ref{mkiml}, the current meta-learning pipeline mainly contains three components except for given training tasks \cite{hospedales2021meta}: 1) Meta-Representation.
Proper instantiation representation of meta-model indicates the specific meta-learning approaches, e.g., initial condition of the optimal hypothesis \cite{finn2017model}, data curriculum strategy \cite{shu2022cmw}, gradient descend algorithm \cite{andrychowicz2016learning}, learning rate schedules \cite{shu2022mlr}, hyper-parameter setting rule \cite{ding2023improve}. More examples can refer to \cite{hospedales2021meta, shu2021learning}. 2) Meta-Optimizer. The outer-level optimizer for learning meta-model can take a variety of forms, e.g., gradient-descent \cite{finn2017model}, reinforcement learning \cite{bello2017neural}, and evolutionary search \cite{houthooft2018evolved}.
(3) Meta-Objective. It defines the learning objective of meta-model, which achieves different purposes such as
sample efficient few-shot learning \cite{finn2017model,snell2017prototypical}, fast computation \cite{andrychowicz2016learning,ravi2016optimization,shu2022mlr}, robustness to domain-shift \cite{li2018learning,balaji2018metareg}, label noise \cite{shu2019meta}, adversarial attack \cite{goldblum2020adversarially}, etc.

%In practice, such meta-model could be regarded as the initial condition of the optimal hypothesis \cite{finn2017model}, data curriculum strategy \cite{shu2022cmw}, gradient descend algorithm \cite{andrychowicz2016learning}, learning rate schedules \cite{shu2022mlr}, or hyper-parameter setting rule \cite{shu2020learning}, etc. More examples refer to \cite{hospedales2021meta,shu2021learning}.

%Generally, the working mechanism of meta-learning is primarily composed with two levels: an inner- and and outer-level learning \cite{hospedales2021meta, shu2021learning}. At the inner-level learning, a novel query task is arrived, and the machine learning agent attempts to learn the task-specific model from the training data just like regular machine learning. And the quick adaptation of inner-level is facilitated by the shared meta-model which has been extracted from a family of observed tasks at the outer-level learning.

Recently, meta-learning has shown great success in improving machine learning in terms of data efficiency \cite{wang2020generalizing,shu2018small}, algorithm automation \cite{he2021automl,karmaker2021automl}, and generalization \cite{maurer2016benefit,tripuraneni2020theory,shu2021learning}. Most of these success stories are grounded in the data-based nature of the approach that learns from a series of training tasks with high-quality meta-data representing the underlying task generalization goal.
However, high-quality training tasks are difficult and expensive to obtain for some real applications, which would result in imperfect training tasks, e.g., meta-data may be noisy or unavailable. The purely meta-data-driven approaches tend to reach their limits or lead to unsatisfactory results under these imperfect circumstances. With meta learning becoming more and more popular in real-life applications, there is also a growing need for meta-learning to learn reliable and robust meta-models under such imperfect training tasks.

%In this paper, we concentrate on the research of designing meta objective. Although there are various attempts on formulating meta objective, existing meta-learning methods have several deficiencies in practice. On the one hand, most methods define a meta-objective using a performance metric computed on a validation/meta set, after updating the task model with the meta-model. Such validation/meta set represents the underlying distribution of generalization goal, however, we sometimes have access to the problematic data for some applications. For example, the corresponding ground-truth labels of validation/meta set are unavailable in label noise problem \cite{shu2019meta} or unsupervised domain adaptation \cite{liu2021cycle}, or the size of validation/meta set is limited in few-shot learning \cite{finn2017model}, which results in a detrimental effect on the generalization performance of the learned meta-model.
%On the other hand, though there exist a large corpus of problem-specific meta-learning objectives \cite{hospedales2021meta}, current methods mainly design the ad-hoc meta-objective for the investigated problems. This problem-specific meta-objective may limit the applicablity to other meta-learning tasks.

\vspace{-2mm}
\subsection{Meta-Knowledge Informed Meta-Learning (MKIML)}\label{section321}
As a step towards addressing the aforementioned limitations, we suggest a meta-knowledge informed meta-learning (MKIML) framework, aiming to improve the vanilla meta-learning by additionally incorporating meta-knowledge into the meta learning process, as shown in Fig.\ref{mkiml}.
MKIML could hopefully learn from a hybrid information source that contains meta-data and meta-knowledge, or even only from meta-knowledge. Generally, MKIML explores an orthogonal direction through introducing compensated meta-knowledge information relative to purely meta-data-driven approaches, which is expected to reduce the requirement for high-quality training task premise, and thus increase the reliability and robustness of meta-learning.

Such meta-knowledge could be obtained in an external and separated way from the meta-learning problem and the usual training tasks, which could be logic rules, knowledge graphs, equations, invariance, probabilistic relations, etc (cf.\cite{von2021informed,deng2020integrating}).
%The essential characteristic of MKIML is that such meta-knowledge could be explicitly integrated into the whole meta learning pipeline.
Theoretically, such beneficial meta-knowledge could be integrated into training tasks, meta-representation (i.e., the hypothesis set of meta-model), meta-objective and meta-optimizer through elaborate designing on the meta-learning regime, and there should exist a significantly wide range of possible manners to realize such MKIML framework. In this study, we focus on a preliminary MKIML attempt that integrates meta-knowledge into the meta-objective in terms of meta-regularizer, to show its potential power of enhancing meta-learning capability, especially its task generalization capability, and expecting to inspire more considerations along this meaningful research line.

\vspace{-4mm}
\subsection{Meta-Regularization for MKIML}\label{section32}
%Recalling the formulation of meta-learning, the inner optimization Eq.(\ref{eqin}) is conditioned on the meta-model $h_{\phi}$ defined by the outer optimization Eq.(\ref{eqout}), but it cannot change $h_{\phi}$ during its training. This leader-follower asymmetry relationship indicates that the function of $\mathcal{L}^{meta}$ and $\mathcal{L}^{task}$ are essentially different. $\mathcal{L}^{task}$ specifies the learning goal of task-specific model $f_{\theta}$, which could be defined by the meta-model $h_{\phi}$ (cf. ``meta-learning loss/regularization" in \S\ref{relatedwork}). And

Recalling that meta-objective measures different purposes for meta-learning, and often uses a meta-loss computed on paired meta-data with high-quality annotations, after updating the task-specific model facilitated by the meta-model. Such meta-data represent the underlying distribution of the targeted goal. However, as aforementioned, we sometimes only have access to problematic meta-data in real applications. This would yield unreliable meta-objective for such meta-data-driven meta-learning, inclining to result in a detrimental effect on the generalization performance of the learned meta-model to be used on unseen tasks.

%Recalling that meta-objective measures different purposes for meta-model, and it often requires paired meta-data with high-quality annotations, which represent the underlying task generalization goal, to evaluate meta-model (cf. Eq.(\ref{eqout})). However, meta-objective tends to be unreliable when meta-data are noisy \cite{shu2019meta} or unavailable \cite{liu2021cycle} for some practical applications.
%Under these imperfect learning circumstances, such meta-data-driven vanilla meta-learning approaches may reach their limits or lead to unsatisfactory results.

Inspired by the MKIML framework, we try to integrate certain meta-knowledge into meta-objective via designing an appropriate meta-regularizer.
To achieve this, we extend vanilla meta-learning objective in Eq.(\ref{eqout},\ref{eqin}) by incorporating an additional meta-regularization term as follows, expecting to ensure that the solution could possess some specific property (i.e., meta-knowledge) of the meta-model:
\begin{align}
	h_{\phi^*}  = \mathop{\arg\min}_{h_{\phi} \in \mathcal{H}}& \gamma \sum_{i=1}^{M} \mathcal{L}^{meta} (\mathcal{D}_i^{(q)}; f_{\theta^*_i}(h_{\phi}),h_{\phi})  + \lambda \mathcal{MR}(h_{\phi}), \notag \\
	\text{s.t.,} \  f_{\theta^*_i}(h_{\phi}) &= \mathop{\arg\min}_{f_{\theta} \in \mathcal{F}} \mathcal{L}^{task} (\mathcal{D}_i^{(s)};f_{\theta}, h_{\phi} ),\label{eqmr}
\end{align}
where $\mathcal{MR}(h_{\phi})$ is the meta-regularizer that imposes some meta-knowledge of meta-model $h_{\phi}$, aiming to regularize the capacity complexity of the function class $\mathcal{H}$, so as to improve its meta-level generalization to unseen tasks. $\gamma, \lambda \geq 0$ are the hyperparameters making a tradeoff between meta-loss $\mathcal{L}^{meta}$ and meta-regularizer $\mathcal{MR}$.

Recently, some meta-regularization strategies have been proposed (cf. ``meta-overfitting and meta-regularization" in \S\ref{relatedwork}) and helped alleviate the meta-overfitting issues of meta-learning. However, most works still require additional meta-loss computation on high-quality meta-data to jointly learn better meta-model. Besides, they need to be delicately designed for specific problems, certainly limiting their applicability to general meta-learning problems.
In contrast, this study tries to present a novel DAC-MR strategy in the following, which can be treated as a proxy meta-objective used to evaluate meta-model, largely reducing
the essential computation reliance of meta-learning on high-quality meta-data. Meanwhile, DAC-MR is problem-agnostic, which can be potentially useful for more general meta-learning tasks.

\vspace{-4mm}
\subsection{Data Augmentation Consistency Based Meta-Regularization (DAC-MR) } \label{section33}
We attempt to introduce the prediction invariance of the meta-modal under some input perturbations as the meta-knowledge for rectifying a sound learning track for meta-model.
%Specifically, we require the meta-model to maintain such fundamental invariance property for most problems learned from training tasks, in order to achieve better generalization to unseen tasks.
Specifically, in our implementation, we adopt data augmentation consistency \cite{xie2020unsupervised,sohn2020fixmatch} to integrate such invariance knowledge into the meta-regularization term in Eq.(\ref{eqmr}), aiming to enforce the meta-model to produce models capable of outputting similar predictions for all augments from a training sample.
Concretely, we denote $\mathcal{A}: \mathcal{X} \rightarrow \mathcal{X}$ as some set of transformations obtained via data augmentation, and DAC-MR strategy can then be written as follows:
\begin{align}
	h_{\phi^*} & = \mathop{\arg\min}_{h_{\phi} \in \mathcal{H}} \gamma \sum_{i=1}^{M} \mathcal{L}^{meta} (\mathcal{D}_i^{(q)}; f_{\theta^*_i(h_{\phi})}) \notag \\
   & 	  +  \lambda \sum_{i=1}^{M} \mathcal{MR}^{dac} (D_i;f_{\theta^*_i(h_{\phi})},A ),  A \in \mathcal{A} \label{eqdacmr} \\
	\text{s.t.,}  \ & f_{\theta^*_i(h_{\phi})} = \mathop{\arg\min}_{f_{\theta} \in \mathcal{F}} \mathcal{L}^{task} (\mathcal{D}_i^{(s)};f_{\theta}, h_{\phi} ),\label{eqdacin}
\end{align}
where
$\mathcal{MR}^{dac} (D_i;f,A) = \frac{1}{|D_i|} \sum_{j=1}^{|D_i|}\rho \left(f(x_{ij}),f(A(x_{ij}))\right)$, $D_i = \{x_{ij}\}_{j=1}^{k}$, and $x_{ij}$ could be any sample related to the problem like ones from support or query sets. The data augmentation strategies $\mathcal{A}$ in our experiments easily follow the strategies used in \cite{xie2020unsupervised}.
$\rho$ is a metric properly defined on the output space, and we use KL-divergence in our experiments as done in \cite{xie2020unsupervised}. Note that DAC-MR is with an evident difference from the consistency regularization used in machine learning \cite{xie2020unsupervised,sohn2020fixmatch}, which solely enforces model predictions invariant to input perturbations. In comparison, DAC-MR is defined by an outer optimization (meta-level learning) that evaluates the benefit of the meta-model when learning a new task. More comprehensively, such DAC-MR strategy could bring the following potential merits:
\begin{itemize}
	\item  If we set $\gamma \neq 0$, and set $D_i$ as a duplicate of $\mathcal{D}_i^{(q)}$, then additional DAC-MR objective could help yield a smaller generalization error bound than the vanilla meta-learning with respect to meta-model (see \S\ref{section34}), i.e., it inclines to possess better meta-level generalization. Such property demonstrates that it is potential to improve meta-learning algorithms by simply integrating such DAC-MR into the original meta-data-driven meta-objective (see \S\ref{section41},  \S\ref{section53}, \S\ref{section61},\S\ref{section62}).
	\item Note that data involved in computing DAC-MR do not require its corresponding ground-truth labels, in which DAC-MR only enforces predictions for an original sample and its augmented ones to be same. If we set $\gamma=0$, and let $D_i$ additionally sampled/divided from $\mathcal{D}_i^{(s)}$, it would reduce the cost of collecting additional meta data  $\mathcal{D}_i^{(q)}$ with high-quality annotations. This makes it possible to apply meta-learning algorithms to  tasks with noisy or unavailable meta-data scenarios (see \S\ref{section43}, \S\ref{section42}, \S\ref{section51}, \S\ref{section52}, \S\ref{section71},\S\ref{section72}).
	\item When $\gamma = 0$, we can further prove that DAC-MR could be regarded as a proxy meta-objective of the meta-loss calculated with high-quality meta-data under some mild conditions (see \S\ref{section35}). Such property implies that our DAC-MR can be implicitly treated as approximated meta-supervised information for meta-model training without additional guidance from meta-data-driven meta-objective $\mathcal{L}^{(meta)}$.
	%There has a great difference from current meta-regularization strategies, as they may require additional meta-loss computation to jointly learn meta-model for investigative problems.
	\item Our DAC-MR is problem-agnostic, which can be generally applied to different meta-learning problems and tasks. We experimentally demonstrate that DAC-MR finely aids meta-model learning in few-shot learning (\S\ref{section4}), transfer learning (\S\ref{section5}), continual learning (\S\ref{section6}) and label noise learning (\S\ref{section7}), and obtain consistency benefits over corresponding baselines.	
\end{itemize}

\vspace{-4mm}

\subsection{Can DAC-MR Help with Meta-level Generalization?}  \label{section34}
To answer this question, we will demonstrate that DAC-MR can effectively reduce the size of function class $\mathcal{H}$. Formally, we define the following DAC-MR operator $\mathcal{T}^{mr}$ over $\mathcal{H}$:
\begin{Definition} [DAC-MR Operator]
	\begin{align*}
		&\mathcal{T}^{mr}_{\mathcal{A},\mathbf{X}}(\mathcal{H}) \triangleq  \left \{ h | h\in \mathcal{H}, f_{\theta_i(h_{\phi})}(x_{ij}) = f_{\theta_i(h_{\phi})}(A(x_{ij})), f_{\theta_i}  \right. \notag \\
		& \left. \in \mathcal{F}, A \in \mathcal{A}, \mathbf{X} = [x_{11}, \!\cdots\!, x_{ij}, \!\cdots\!, x_{Mn}],  \forall i \in [M], j \in [n] \right\}.
	\end{align*}
\end{Definition}
The DAC-MR operator can be understood as mapping the original function class $\mathcal{H}$ to a potentially smaller subset $\mathcal{T}^{mr}_{\mathcal{A},\mathcal{X}}(\mathcal{H})$, in which every function $h\in \mathcal{T}^{mr}_{\mathcal{A},\mathcal{X}}(\mathcal{H})$ produces consistent predictions for given samples and their corresponding augmentations, i.e., $ f_{\theta^*_i}(h_{\phi})(x_{ij}) = f_{\theta^*_i}(h_{\phi})(A(x_{ij}))$. To illustrate this,
we instantiate our framework for one of the most frequently used classification methods — logistic regression with $\mathcal{X} = \{ x\in \mathbb{R}^d | \|x\|_2 \leq G \}$, $\mathcal{Y} = \{0,1\}$ (the conclusion also holds for $\mathcal{Y} = [K]$). We follow the setting in \cite{tripuraneni2020theory}, and  consider the function class:
\begin{align}
	\mathcal{F} = &\{ f|f(z)  = \alpha^T z, \alpha \in \mathbb{R}^r, \|\alpha\| \leq c \}, \notag \\
	\mathcal{H} = &\{h | h(x) = B^Tx, B = (b_1,\cdots,b_r ) \in \mathbb{R}^{d\times r}, \notag \\
	&B \  \text{is a matrix with orthonormal columns} \},  \label{eqfh}
\end{align}
where task-specific functions $f$s are linear maps, and the underlying meta-representation $h$ is a projection onto a low-dimensional subspace. Such meta-level representation learning would provide a statistical guarantee for several important meta-learning sceneries \cite{maurer2016benefit,tripuraneni2020theory,tripuraneni2021provable,du2020few,sun2021towards,xu2021representation}, e.g., few-shot learning, transfer learning, which will be considered in \S\ref{section4} and \S\ref{section5}, respectively.
We assume $P(y=1|f\circ h(x)) = \sigma(\alpha^TB^Tx)$, where $\sigma(\cdot)$ is the sigmoid function with $\sigma(z) =  1 /(1+\exp(-z))$, and then use the logistic loss $\ell(z,y) =-y\log(\sigma(z)) - (1-y) \log (1-\sigma(z))$ for $\mathcal{L}^{train}$ and $\mathcal{L}^{meta}$.
For the instantiation in Eq.(\ref{eqfh}), \cite{tripuraneni2020theory} recently has theoretically demonstrated that meta-level error bound with respect to meta-model scales as $ C(\mathcal{H}) + M C(\mathcal{F})$, where $C(\cdot)$ captures the complexity of function class, and $M$ denotes some coefficients independent of model complexity. Therefore, we could show Eq.(\ref{eqdacmr}) with additional DAC-MR term (assume $\gamma = \lambda = 1 $) brings better meta-level generalization than Eq.(\ref{eqmr}) through illustrating that the complexity of $\mathcal{T}^{mr}_{\mathcal{A},\mathbf{X}}(\mathcal{H})$ is smaller than $\mathcal{H}$.
To this aim, we use $d_{dac}$ to quantify the strength of data augmentation $\mathcal{A}$ defined by:
\begin{align*}
	d_{dac} \triangleq \mathrm{rank} (A(\mathbf{X})  -\mathbf{X} ), \text{for a fixed} \ A \in \mathcal{A},
\end{align*}
where $A(\mathbf{X}) \!=\! [x_{11}', \!\cdots\!, x_{ij}', \!\cdots\!, x_{Mn}']$,  $x_{ij}'=A(x_{ij})$ represents an augmented sample from $x_{ij}$, and $0 \leq d_{dac} \leq \min (d,Mn)$. As can be seen, $d_{dac}$ measures the number of dimensions perturbed by augmentation, i.e., larger $d_{dac}$ implies that $A(\mathbf{X})$ more evidently perturbs the original dataset, and $d_{dac}=0$ means no augmentations.
\begin{Theorem}\label{th1}\footnote{The theorem shows a relatively concise while informal result. Its formal description is given in Theorem 1 of supplementary material.}
	\begin{align*}
	C(\mathcal{T}^{mr}_{\mathcal{A},\mathbf{X}}(\mathcal{H})) \lesssim \sqrt{\frac{( d-d_{dac})r^2}{Mn}}, C(\mathcal{H})  \lesssim \sqrt{\frac{d r^2 }{Mn}}.
	\end{align*}
\end{Theorem}
By comparing two complexities, we can see that DAC-MR is efficient to reduce the complexity of the meta-model, which decreases the dimensions from $d$ to $d-d_{dac}$ by enforcing DAC-MR. In particular, consider the scenario that data augmentations well perturb the data, e.g., $d_{dac} = d-c, c \ll d$, and then the vanilla meta-learning gives a complexity that scales as $\sqrt{\frac{dr^2}{Mn}}$, while our DAC-MR yields a dimension-free error $\sqrt{\frac{cr^2}{Mn}}$. In practice, we often instantiate $\mathcal{A}$ with strong data augmentations as in \cite{sohn2020fixmatch}, which can ensure that $d_{dac} = d-c$ holds. Besides, the dimension of $\mathcal{X}$ could be large for many real applications, and hence such dimension-free error is promising to bring expected fine improvements. We further empirically demonstrate that incorporating such DAC-MR can generally benefit existing meta-learning algorithms in \S\ref{section41}, \S\ref{section53}, \S\ref{section61},\S\ref{section62}.

\begin{table*}[t] \vspace{-2mm}
	\caption{Classification results of inductive FSL benchmarks on the miniImageNet and CIFAR-FS datasets, compared with previous four typical meta-learning methods. $^\diamond$ denotes results reported by the original paper. $^\dag$ represents results replicated by us to the best effort.   }\label{tablefsl}  \vspace{-4mm}
	\centering
	\scriptsize
	\tabcolsep =1.5mm
	\begin{tabular}{l|l|ll|ll}
		\toprule
		\multirow{2}{*}{\textbf{Model} }	& \multirow{2}{*}{\textbf{Backbone} } &	\multicolumn{2}{c|}{\textbf{miniImageNet 5-way}}   & \multicolumn{2}{c}{\textbf{CIFAR-FS 5-way}} \\
		\cline{3-4} \cline{5-6}
		&  &     \multicolumn{1}{c} {\textbf{1-shot } }  & \multicolumn{1}{c|} {\textbf{5-shot } }  &  \multicolumn{1}{c} {\textbf{1-shot } }  & \multicolumn{1}{c} {\textbf{5-shot } }       \\
		\hline \hline
		MAML \cite{finn2017model} $^\diamond$  &  32-32-32-32     &  48.70 $\pm$ 1.84    &   63.11 $\pm$ 0.92 &  58.90 $\pm$ 1.90  &    71.50 $\pm$ 1.00    \\
		MAML \cite{finn2017model} $^\dag$ &   32-32-32-32     &  46.75 $\pm$ 0.63   & 60.45 $\pm$ 0.57  &  51.97 $\pm$ 0.70  & 69.50 $\pm$ 0.59    \\
		MAML + DAC-MR  &   32-32-32-32     & 47.48 $\pm$ 0.63 \textcolor{red}{0.73$\uparrow$}   &  61.33 $\pm$ 0.57 \textcolor{red}{0.88$\uparrow$} &  53.40 $\pm$ 0.73 \textcolor{red}{1.43$\uparrow$} & 71.60 $\pm$ 0.60 \textcolor{red}{2.10$\uparrow$}   \\
		\hline \hline
		ProtoNet \cite{snell2017prototypical} $^\diamond$  & 64-64-64-64  & 49.42 $\pm$ 0.78   &   68.20 $\pm$ 0.66  &  55.50 $\pm$ 0.70  & 72.00 $\pm$ 0.60   \\
		ProtoNet \cite{snell2017prototypical} $^\dag$ &  64-64-64-64  &  47.73 $\pm$ 0.63  &  70.82 $\pm$ 0.53 &       60.19 $\pm$ 0.72  &       79.67 $\pm$ 0.52        \\
		ProtoNet + DAC-MR &     64-64-64-64        &  48.78 $\pm$ 0.64 \textcolor{red}{1.05$\uparrow$}  &  71.16 $\pm$ 0.52 \textcolor{red}{0.34$\uparrow$} &  61.23 $\pm$ 0.72 \textcolor{red}{1.04$\uparrow$} &  80.81 $\pm$ 0.52 \textcolor{red}{1.14 $\uparrow$}   \\
		\hline
		ProtoNet \cite{snell2017prototypical} $^\dag$ &  ResNet-12  &  55.59 $\pm$ 0.65  &  75.46 $\pm$ 0.53 &       68.90 $\pm$ 0.74  &       83.51 $\pm$ 0.51       \\
		ProtoNet + DAC-MR &     ResNet-12        &  57.15 $\pm$ 0.66 \textcolor{red}{1.56$\uparrow$}  &  76.55 $\pm$ 0.52 \textcolor{red}{1.09$\uparrow$} &  71.48 $\pm$ 0.75 \textcolor{red}{2.58$\uparrow$} &  84.99 $\pm$ 0.50 \textcolor{red}{1.48 $\uparrow$}   \\
		\hline \hline
		R2D2 \cite{bertinetto2018meta} $^\diamond$  & 96-192-384-512  &51.20 $\pm$ 0.60&  68.80 $\pm$ 0.10  &     65.30 $\pm$ 0.20     &       79.40 $\pm$ 0.10 \\
		R2D2 \cite{bertinetto2018meta} $^\dag$  & 96-192-384-512  &55.90 $\pm$ 0.62&  73.17 $\pm$ 0.49  &     67.71 $\pm$ 0.68     &      83.07 $\pm$ 0.50 \\
		R2D2 + DAC-MR  & 96-192-384-512  &56.60 $\pm$ 0.62 \textcolor{red}{0.70$\uparrow$} &  73.18 $\pm$ 0.51 \textcolor{red}{0.01$\uparrow$}   &     69.72 $\pm$ 0.68  \textcolor{red}{2.01$\uparrow$}    &      84.07 $\pm$ 0.49 \textcolor{red}{1.00$\uparrow$} \\
		\hline
		R2D2 \cite{bertinetto2018meta} $^\dag$  & ResNet-12  &58.80 $\pm$ 0.65&  76.44 $\pm$ 0.49  &     71.79 $\pm$ 0.72     &      84.16 $\pm$ 0.51 \\
		R2D2 + DAC-MR  & ResNet-12  &61.31 $\pm$ 0.66 \textcolor{red}{2.51$\uparrow$} &  77.50 $\pm$ 0.49 \textcolor{red}{1.06$\uparrow$}   &     74.22 $\pm$ 0.72  \textcolor{red}{2.43$\uparrow$}    &      86.09 $\pm$ 0.49 \textcolor{red}{1.93$\uparrow$} \\
		\hline \hline
		MetaOptNet-SVM \cite{lee2019meta} $^\diamond$ &  ResNet-12  &   62.64 $\pm$ 0.61 &  78.63 $\pm$ 0.46  &     72.00 $\pm$ 0.70      &            84.20 $\pm$ 0.50       \\
		MetaOptNet-SVM \cite{lee2019meta}  $^\dag$ &  ResNet-12  &   60.68 $\pm$ 0.66 &  77.32 $\pm$ 0.48  &     71.22 $\pm$ 0.71      &            84.38 $\pm$ 0.50       \\
		MetaOptNet-SVM + DAC-MR &  ResNet-12  & 62.26 $\pm$ 0.65 \textcolor{red}{1.58$\uparrow$} &78.56 $\pm$ 0.48  \textcolor{red}{1.24$\uparrow$}     &      74.16 $\pm$ 0.72  \textcolor{red}{2.94$\uparrow$}  &      86.23 $\pm$ 0.49 \textcolor{red}{1.85$\uparrow$} \\
		\bottomrule
	\end{tabular} \vspace{-2mm}
\end{table*}

\vspace{-4mm}

\subsection{Can DAC-MR Work Without Meta-loss $\mathcal{L}^{(meta)}$?  }  \label{section35}
We firstly introduce some basic notations and definitions. Consider $\mathcal{X} = \{ x\in \mathbb{R}^d | \|x\|_2 \leq G\}$, $\mathcal{Y} = [K]$, and define $\mathcal{B}(x) = \{x': \exists A \in \mathcal{A} \ \text{such that} \ \|x'-A(x)\| \leq r \}$ to
be the set of points with distance $r$ from some data augmentations of $x$. The DAC-MR requires that a classifier $f_{\theta(h_{\phi})}$ (we denote $f_{h}$ for brevity below) learned on training data make predictions stably on another partition meta data (unlabelled data) under a suitable set of data augmentations. The DAC-MR objective of $f_{h}$ on the probability measure $P$ can be defined as the fraction of examples where $f_{h}$ is not robust to input data augmentation transformations:
\begin{align*}
	R_{P}^{\mathcal{B}}(f_h) = \mathbb{E}_{P} [\mathbf{1} (\exists x' \in \mathcal{B}(x), \text{s.t.}, f_h(x')\neq f_h(x))].
\end{align*}
Without loss of generality, we consider single training task setting (it is easy to extend the conclusion to multi-task settings), where $S$ and $Q$ are training and meta data distributions.
To the goal, we would establish the relationship between expected meta error $\epsilon^{Q}(f_h) = \mathbb{E}_{Q} \mathbf{1}[f_h(x) \neq g^*(x)]$ and expected DAC-MR loss $R_{Q}^{\mathcal{B}}(f_h)$ in the following, where $g^*(x)$ is the ground-truth label generation function on $S\cup Q$.

 \begin{table*} \vspace{-2mm}
	\caption{Classification results of cross-domain FSL benchmarks on Meta-Dataset (using a multi-domain feature extractor of URT \cite{liu2020universal}).
		The first eight datasets are seen during training and the last five datasets are unseen and used for test only. Results of other methods are copied from TSA \cite{li2022cross}.} \vspace{-4mm} \label{tabletsa}
	\centering
	\scriptsize
	\setlength{\tabcolsep}{1.5mm}
	\begin{tabular}{c|c|c|c|c|c|c|c|c}
		\toprule
		Test datasets &   Simple CNAPS \cite{bateni2020improved} & SUR \cite{dvornik2020selecting}  &URT \cite{liu2020universal} &FLUTE \cite{triantafillou2021learning} &tri-M \cite{liu2021multi} & URL \cite{li2021universal} &   TSA \cite{li2022cross}&  TSA + DAC-MR  \\ \hline
		ImageNet & 58.4 $\pm$ 1.1& 56.2 $\pm$ 1.0  & 56.8 $\pm$ 1.1 &58.6 $\pm$ 1.0 &51.8 $\pm$ 1.1 &58.8 $\pm$ 1.1 & 59.5 $\pm$ 1.0 & \textbf{60.1 $\pm$ 1.0}  \\
		Omniglot &  91.6 $\pm$ 0.6 & 94.1 $\pm$ 0.4 & 94.2 $\pm$ 0.4 & 92.0 $\pm$ 0.6 & 93.2 $\pm$ 0.5 & 94.5 $\pm$ 0.4 & 94.9 $\pm$ 0.4 &  \textbf{95.5 $\pm$ 0.8} \\
		Aircraft & 82.0 $\pm$ 0.7 & 85.5 $\pm$ 0.5 & 85.8 $\pm$ 0.5 & 82.8 $\pm$ 0.7 &  87.2 $\pm$ 0.5 & 89.4 $\pm$ 0.4 & 89.9 $\pm$ 0.4& \textbf{90.7 $\pm$ 0.3}\\
		Birds & 74.8 $\pm$ 0.9&71.0 $\pm$  1.0 &76.2 $\pm$  0.8 &75.3 $\pm$  0.8 &79.2 $\pm$  0.8& 80.7 $\pm$  0.8 &81.1 $\pm$  0.8 & \textbf{82.1 $\pm$ 0.6}\\
		Textures &  68.8 $\pm$ 0.9&  71.0 $\pm$  0.8& 71.6 $\pm$  0.7& 71.2 $\pm$ 0.8 &68.8 $\pm$ 0.8 &77.2 $\pm$ 0.7& 77.5 $\pm$ 0.7 & \textbf{77.9 $\pm$ 0.7}\\
		Quick Draw &76.5 $\pm$ 0.8&  81.8 $\pm$  0.6& 82.4 $\pm$  0.6& 77.3 $\pm$  0.7& 79.5 $\pm$   0.7& \textbf{82.5 $\pm$ 0.6} &81.7 $\pm$ 0.6& 81.7 $\pm$ 0.6\\
		Fungi &  46.6 $\pm$ 1.0 &64.3 $\pm$  0.9 &64.0 $\pm$  1.0& 48.5 $\pm$  1.0 &58.1 $\pm$  1.1 &\textbf{68.1 $\pm$  0.9} &66.3 $\pm$  0.8 & 67.2 $\pm$ 0.5\\
		VGG Flower&90.5 $\pm$ 0.5& 82.9 $\pm$  0.8& 87.9 $\pm$  0.6 &90.5 $\pm$  0.5 &91.6 $\pm$  0.6 &92.0 $\pm$  0.5& 92.2 $\pm$  0.5 & \textbf{93.0 $\pm$ 0.6}\\ \hline
		Traffic Sign & 57.2 $\pm$ 1.0&  51.0 $\pm$  1.1 &48.2 $\pm$  1.1 &63.0 $\pm$  1.0 &58.4 $\pm$  1.1 &63.3 $\pm$  1.1 &82.8 $\pm$  1.0 & \textbf{86.9 $\pm$ 0.7}\\
		MSCOCO & 48.9 $\pm$ 1.1&52.0 $\pm$  1.1 &51.5 $\pm$  1.1 &52.8 $\pm$  1.1& 50.0 $\pm$  1.0 &57.3 $\pm$  1.0 &57.6 $\pm$  1.0 & \textbf{60.5 $\pm$ 0.9}\\
		MNIST &94.6 $\pm$ 0.4& 94.3 $\pm$  0.4& 90.6 $\pm$  0.5 &96.2 $\pm$  0.3& 95.6 $\pm$  0.5& 94.7 $\pm$  0.4 &96.7 $\pm$  0.4& \textbf{97.0 $\pm$ 0.2}\\
		CIFAR-10 & 74.9 $\pm$ 0.7&  66.5 $\pm$  0.9 &67.0 $\pm$  0.8& 75.4 $\pm$  0.8& 78.6 $\pm$  0.7 &74.2 $\pm$  0.8& 82.9 $\pm$  0.7 & \textbf{84.8 $\pm$ 0.5}\\
		CIFAR-100 & 61.3 $\pm$ 1.1& 56.9 $\pm$  1.1& 57.3 $\pm$  1.0& 62.0 $\pm$  1.0& 67.1 $\pm$  1.0 &63.5 $\pm$  1.0 &70.4 $\pm$  0.9 & \textbf{72.7 $\pm$ 0.9}\\  \hline
		Average Seen & 73.7&  75.9 &77.4& 74.5 &76.2 &80.4 &80.4 & \textbf{81.0}  \\
		Average Unseen & 67.4& 64.1& 62.9 &69.9& 69.9& 70.6& 78.1 & \textbf{80.4}\\
		Average All &71.2&  71.4& 71.8 &72.7 &73.8& 76.6 &79.5  &  \textbf{80.8} \\ \hline
		Average Rank &6.6 & 6.3& 5.7 &5.2 & 5.1&3.2 &2.1& \textbf{1.2}\\
		\bottomrule
	\end{tabular}\vspace{-6mm}
\end{table*}

Now, we define the neighborhood function $\mathcal{N}$ as
\begin{align*}
	\mathcal{N}(x) = \{x'\in \mathcal{X} | \mathcal{B}(x) \cap \mathcal{B}(x') \neq \emptyset \},
\end{align*}
and the neighborhood of a set $S \in S\cup Q $ as
\begin{align*}
	\mathcal{N}(S) = \cup_{x\in S}\mathcal{N}(x).
\end{align*}
Let $\mathcal{X}_k \triangleq \{ x\in \mathcal{X} | g^*(x) = k\}$, with $\mathcal{X}_i \cap \mathcal{X}_j = \emptyset, \forall i \neq j$.
We make the following assumption (Assumption 2 in \cite{cai2021theory}).
\begin{Assumption} \label{assumption1}
	Assume the task training and meta data distributions have the following structures: $supp(S) = \cup_{k=1}^K S_k$, $supp(Q) = \cup_{k=1}^K Q_k$, and $S_i \cap S_j = Q_i \cap Q_j = \emptyset, \forall i\neq j $. We further assume the ground truth class $g^*(x)$ for $x \in S_k \cup Q_k$ is consistent, which is denoted as $y_k \in [K]$. Additionally,
	suppose that there exists a constant $\kappa \geq 1$, such that for any $i \in [K]$,
	\begin{align} \label{eqshift}
		P(Q_i \cap A)  \leq \kappa P\left(A \cap \frac{1}{2}\left(S_i\cup Q_i\right)\right), \forall A \subset \mathcal{S\cup Q}.
	\end{align}
\end{Assumption}
Eq.(\ref{eqshift}) presents a quantitative formulation of distribution shift between $S$ and $Q$ \cite{cai2021theory}. To capture connectivity of the data distribution, we further introduce the expansion property \cite{wei2020theoretical,cai2021theory} on the mixed distribution $S\cup Q$ below:
\begin{Definition} [Constant Expansion \cite{wei2020theoretical}]
	We say that the distribution $S\cup Q$ satisfies $(q,\xi)$-constant expansion for some constant $q,\xi \in (0,1)$, if for any $V \subset S\cup Q$ with $P(V) \geq q$ and $P(V\cap (S_k \cup Q_k)) \leq 1/2$ for any $k\in[K]$, we have $P((\mathcal{N}(V) \setminus V)\cap (S_k \cup Q_k)) \geq \min \{P(V\cap (S_k \cup Q_k)), \xi\}$.
\end{Definition}
\begin{Definition} [Multiplicative Expansion \cite{wei2020theoretical}]
	We say that the distribution $S\cup Q$ satisfies $(a,c)$-multiplicative expansion for some constant $a \in (0,1), c>1$, if for any $k\in[K]$ and $V \subset S \cup Q$ with $P(V\cap (S_k \cup Q_k)) \leq a$, we have $P(\mathcal{N}(V) \cap (S_k \cup Q_k)) \geq \min \{c\cdot P(V \cap (S_k \cup Q_k)), 1\}$.
\end{Definition}
This expansion property lower bounds the neighborhood size of low probability sets, and the parameters $(q,\xi)$ or $(a,c)$ quantify the augmentation strength of $\mathcal{A}$. Specifically, the strength of expansion-based data augmentations is characterized by expansion capability of $\mathcal{A}$: for a neighborhood $V \subset \mathcal{X}$ of proper size (characterized by $q$ or $a$ under measure $P$), the stronger augmentation $A$ leads to more expansion in $\mathcal{N}(S)$, and therefore larger $\xi$ or $c$. The following proposition builds a bridge of two expansions.
\begin{Proposition}[Lemma C.6 in \cite{wei2020theoretical}] \label{prop1}
	Suppose the distribution $S\cup Q$ satisfies $(1/2,c)$-multiplicative expansion on $\mathcal{X}$. Then for any choice of $\xi > 0$, $S\cup Q$ satisfies $(\frac{\xi}{c-1}, \xi)$-constant expansion.
\end{Proposition}
We can then have following results for connecting DAC-MR and meta-loss calculated with high-quality meta-data.
\begin{Theorem}[Bounding the Meta Error with Constant Expansion] \label{thm1}
	Supposed that Assumption \ref{assumption1} holds and $\frac{1}{2}(S+Q)$ satisfies $(q,\xi)$-constant expansion, and then we have
	\begin{align*}
			\epsilon^{Q}(f_h) \!\leq\! \left\{\! \frac{\kappa}{1-q}\!+\! \frac{1}{\min \{q,\xi\}}\!\right\} \! R_{Q}^{\mathcal{B}}(f_h)
			 + \frac{2\kappa\max (q, R_{Q}^{\mathcal{B}}(f_h))}{1-q} .
	\end{align*}
\end{Theorem}
\begin{Theorem}[Bounding the Meta Error with Multiplicative Expansion] \label{thm2}
	Supposed that Assumption \ref{assumption1} holds and $\frac{1}{2}(S+Q)$ satisfies $(\frac{1}{2},c)$-multiplicative expansion, and then we have
	\begin{align*}
		\epsilon^{Q}(f_h)\! \leq\! \left\{ \!\max \!\left(\! \frac{\kappa (c+1)}{c-1-\xi}, 3(c-1)\!\right) \!+\! \frac{1}{\min (\frac{\xi}{c-1}, \xi)}\!\right\} \! R_{Q}^{\mathcal{B}}(f_h).
	\end{align*}
\end{Theorem}

As can be seen, the expected meta-loss $\epsilon^{Q}(f_h)$ can be upper bounded by the expected DAC-MR $R_{Q}^{\mathcal{B}}(f_h)$. This implies that when the expected meta-loss is intractable (i.e., $\gamma = 0$ in Eq.(\ref{eqdacmr})), DAC-MR can be regarded as an approximated meta-supervised information to evaluate meta-model. In \S\ref{section43}, \S\ref{section42}, \S\ref{section51}, \S\ref{section52}, \S\ref{section71}, \S\ref{section72}, we will empirically show that DAC-MR can be finely used as a proxy meta-objective to meta-learn a well-performing meta-model when meta-data are noisy or even unavailable.

\section{DAC-MR Benefits Few-Shot Learning}\label{section4}
In this section, we study whether DAC-MR can help improve meta-learning algorithms for three typical few-shot learning tasks, including inductive (\S\ref{section41}), cross-domain (\S\ref{section43}), and transductive/semi-supervised (\S\ref{section42}) situations.

\vspace{-4mm}
\subsection{Inductive Few-Shot Learning}  \label{section41}

\textbf{Formulation.}
We consider four typical meta-learning methods: MAML \cite{finn2017model}, ProtoNet \cite{snell2017prototypical}, R2D2 \cite{bertinetto2018meta}, and MetaOptNet \cite{lee2019metaopt} with SOTA performance. All of them use the meta-objective computed on meta (query) dataset as in Eq.(\ref{eqout}) to evaluate meta-model. We introduce DAC-MR computed on meta-data as in Eq.(\ref{eqdacmr}) to further improve meta-level generalization of meta-model ($\gamma=\lambda=1$).
The implementation is adapted from the official implementations of MetaOptNet available at \url{https://github.com/kjunelee/MetaOptNet}.

\noindent\textbf{Results.} Table \ref{tablefsl} summarizes the results on the 5-way miniImageNet and CIFAR-FS classification tasks with different shots.
Considering different settings of compared methods, we replicate to the best efforts with exactly the same setup following \cite{lee2019metaopt}. It is seen that  DAC-MR does help consistently improve meta-level generalization error from existing meta-learning algorithms in all cases. This implies that our proposed meta-regularization scheme is model-agnostic, in the sense that it can be directly applied to meta-regularize meta-model for different meta-learning algorithms.
In addition to comparing with ProtoNet \cite{snell2017prototypical} and R2D2 \cite{bertinetto2018meta} on their original small backbones, we also compare with both methods with larger convolutional backbones, i.e., ResNet-12. Interestingly, ProtoNet and R2D2 both show competitive results, and especially R2D2 already
performs better than MetaOptNet on CIFAR-FS. Then, DAC-MR can still consistently provide improvements to both (enhanced) baselines with an obvious gap, which yields at least 1.5\% and 2.4\% 1-shot accuracy improvement for ProtoNet and R2D2, respectively. Besides, DAC-MR brings more notable gains to 1-shot accuracy
than to 5-shot in most cases, which is reasonable because dimension-free error of DAC-MR brings more profits than dimension-dependent error of vanilla meta learning when the size of meta-data is small according to Theorem \ref{th1}. This highlights the advantage of DAC-MR for such limited data learning scenario.

 \begin{table*} \vspace{-4mm}
	\caption{Classification results of transductive/semi-supervised benchmarks on three datasets. $^\dag$ denotes results replicated by us to the best effort.} \vspace{-4mm} \label{tableilpc}
	\centering
	\scriptsize
	\setlength{\tabcolsep}{0.5mm}
	\begin{tabular}{c|l|l|l|l|l|l|l}
		\toprule
		\multirow{2}{*}{\textbf{Setting} } &  \multirow{2}{*}{\textbf{Methods} } & \multicolumn{2}{c|}{\textbf{ miniImageNet}}  & \multicolumn{2}{c|}{\textbf{CIFAR-FS}} & \multicolumn{2}{c}{\textbf{CUB}}   \\ \cline{3-8}
		&  &  \multicolumn{1}{c|}{1-shot} &\multicolumn{1}{c|}{5-shot} &\multicolumn{1}{c|}{1-shot} &\multicolumn{1}{c|}{5-shot}& \multicolumn{1}{c|}{1-shot} &\multicolumn{1}{c}{5-shot}      \\ \hline
		& LR+ICI \cite{wang2020instance} &    81.31 $\pm$ 0.84 &  88.53 $\pm$ 0.43 &  86.03 $\pm$ 0.77 & 89.57 $\pm$ 0.53 & 90.82 $\pm$ 0.59 & \multicolumn{1}{c}{-} \\
		{Semi-Supervised} & PT+MAP \cite{hu2021leveraging} & 83.14 $\pm$ 0.72 & 88.95 $\pm$ 0.38 & 87.05 $\pm$ 0.69 &89.98 $\pm$ 0.49 & 91.52 $\pm$ 0.53 & \multicolumn{1}{c}{-} \\
		(WRN-28-10) & iLPC \cite{lazarou2021iterative} & 83.58 $\pm$ 0.79 &89.68 $\pm$ 0.37 & 87.03 $\pm$ 0.72& 90.34 $\pm$ 0.50 &91.69 $\pm$ 0.55 &\multicolumn{1}{c}{-} \\  \cline{2-8}
		& iLPC$^\dag$ & 82.75 $\pm$ 0.80 & 88.32 $\pm$ 0.59 & 87.62 $\pm$ 0.70 & 90.39 $\pm$ 0.49 & 91.27 $\pm$ 0.58 &   \multicolumn{1}{c}{-}  \\
		& iLPC + DAC-MR & \textbf{85.93 $\pm$ 0.71} \textcolor{red}{3.08$\uparrow$} & \textbf{90.61 $\pm$ 0.34} \textcolor{red}{2.29$\uparrow$} & \textbf{88.84 $\pm$ 0.66} \textcolor{red}{1.22$\uparrow$} & \textbf{90.72 $\pm$ 0.48} \textcolor{red}{0.33$\uparrow$} & \textbf{93.07 $\pm$ 0.48} \textcolor{red}{1.80$\uparrow$} &  \multicolumn{1}{c}{-}    \\ \hline
		& EP \cite{rodriguez2020embedding} & 70.74 $\pm$ 0.85 &84.34 $\pm$ 0.53 & \multicolumn{1}{c|}{-}   & \multicolumn{1}{c|}{-}    & \multicolumn{1}{c|}{-}    & \multicolumn{1}{c}{-}   \\
		& SIB \cite{hu2019empirical} & 70.00 $\pm$ 0.60 &79.20 $\pm$ 0.40 &80.00 $\pm$ 0.60 &85.3 $\pm$ 0.40 & \multicolumn{1}{c|}{-}   &  \multicolumn{1}{c}{-}    \\
		Transductive& LaplacianShot \cite{ziko2020laplacian} & 74.86 $\pm$ 0.19 &84.13 $\pm$ 0.14 &  \multicolumn{1}{c|}{-}  & \multicolumn{1}{c|}{-}   &  \multicolumn{1}{c|}{-}  &  \multicolumn{1}{c}{-}    \\
		(WRN-28-10)& PT+MAP \cite{hu2021leveraging}  & 82.88 $\pm$ 0.73& 88.78 $\pm$ 0.40& 86.91 $\pm$ 0.72 &90.50 $\pm$ 0.49 &91.37 $\pm$ 0.61 &93.93 $\pm$ 0.32\\
		& iLPC \cite{lazarou2021iterative} & 83.05 $\pm$ 0.79 &88.82 $\pm$ 0.42 & 86.51 $\pm$ 0.75& 90.60 $\pm$ 0.48& 91.03 $\pm$ 0.63&94.11 $\pm$ 0.30 \\ \cline{2-8}
		& iLPC$^\dag$ & 82.34 $\pm$ 0.79 & 88.42 $\pm$ 0.42 & 87.05 $\pm$ 0.77 & 90.39 $\pm$ 0.51 & 91.16 $\pm$ 0.61 & 94.03 $\pm$ 0.30 \\
		& iLPC + DAC-MR & \textbf{84.38 $\pm$ 0.76} \textcolor{red}{2.04$\uparrow$}&  \textbf{89.50 $\pm$ 0.41} \textcolor{red}{1.08$\uparrow$}& \textbf{87.69 $\pm$ 0.75} \textcolor{red}{0.64$\uparrow$}& \textbf{90.82 $\pm$ 0.51} \textcolor{red}{0.43$\uparrow$}& \textbf{92.05 $\pm$ 0.59} \textcolor{red}{0.89$\uparrow$}& \textbf{94.48 $\pm$ 0.29} \textcolor{red}{0.45$\uparrow$}\\
		\bottomrule
	\end{tabular}\vspace{-2mm}
\end{table*}

\begin{table*} \vspace{-2mm}
	\caption{Classification results of unsupervised domain adaptation benchmark on Office-Home dataset. $^\dag$ indicates our reproduced results to the best effort.} \vspace{-4mm} \label{tableoffice}
	\centering
	\scriptsize
	\setlength{\tabcolsep}{1.5mm}
	\begin{tabular}{c|c|c|c|c|c|c|c|c|c|c|c|c|c}
		\toprule
		Method & Ar-Cl & Ar-Pr& Ar-Rw& Cl-Ar& Cl-Pr& Cl-Rw& Pr-Ar& Pr-Cl& Pr-Rw &Rw-Ar& Rw-Cl& Rw-Pr& Avg. \\ \hline
		DANN \cite{ganin2016domain} & 45.6& 59.3 &70.1& 47.0& 58.5& 60.9& 46.1& 43.7& 68.5& 63.2& 51.8& 76.8 &57.6\\
		CDAN \cite{long2018conditional} & 50.7 &70.6 &76.0 &57.6 &70.0& 70.0 &57.4 &50.9& 77.3 &70.9 &56.7& 81.6& 65.8\\
		CDAN+VAT+Entropy& 52.2& 71.5& 76.4& 61.1& 70.3 &67.8& 59.5& 54.4& 78.6& 73.2 &59.0& 82.7& 67.3 \\
		FixMatch \cite{sohn2020fixmatch} & 51.8& 74.2& 80.1& 63.5& 73.8 &61.3& 64.7& 51.4& 80.0& 73.3& 56.8 &81.7& 67.7 \\
		MDD \cite{zhang2019bridging} & 54.9 &73.7& 77.8& 60.0& 71.4& 71.8& 61.2& 53.6& 78.1& 72.5 &60.2& 82.3& 68.1\\
		SENTRY \cite{prabhu2021sentry}& 61.8& 77.4 &80.1& 66.3 &71.6& 74.7& 66.8 &63.0& 80.9 &74.0 &66.3 &84.1 &72.2 \\
		CST\cite{liu2021cycle}& 59.0 &79.6& 83.4& 68.4 &77.1 &76.7 &68.9 &56.4 &83.0& 75.3& 62.2& 85.1 &73.0\\ \hline
		CST\cite{liu2021cycle}$^\dag$ & 58.4 &80.6& 83.1& 66.9 &76.4 &77.0 &68.1 &55.1 &82.7& 74.4& 61.2& 85.1 &72.4\\
		DAC-MR with $\gamma=0,\lambda=1$ & \underline{58.4}  &  \underline{81.1} & \underline{83.1} &  \underline{67.7} & \underline{77.6} & \underline{77.2} & \textbf{69.0} & \underline{55.5} &  \underline{82.7 } & \underline{75.3} & \textbf{63.1} & \textbf{85.7} & \underline{73.0} \\
		DAC-MR with $\gamma=1,\lambda=1$ & \textbf{59.5}& \textbf{81.2}&\textbf{83.6}&\textbf{68.3}&\textbf{77.9}&\textbf{78.1}&\underline{68.2}&\textbf{56.3}&\textbf{83.2}&\textbf{75.7}&\underline{62.5}&\underline{85.2}  & \textbf{73.7}\\ \bottomrule
	\end{tabular}\vspace{-4mm}
\end{table*}

\begin{table} \vspace{-2mm}
	\caption{Classification results of unsupervised domain adaptation benchmark on VisDA-2017. $^\dag$ indicates the our reproduced results to the best effort.} \vspace{-4mm} \label{tablevisda}
	\centering
	\scriptsize
	\setlength{\tabcolsep}{1.5mm}
	\begin{tabular}{c|cc}
		\toprule
		Method & ResNet-50  & ResNet-101\\ \hline
		DANN \cite{ganin2016domain} &69.3 &79.5 \\
		CDAN \cite{long2018conditional} & 70.0 &80.1 \\
		VAT \cite{miyato2018virtual} & 68.0  & 73.4  \\
		CDAN+VAT+Entropy& 76.5 & 80.4  \\
		FixMatch \cite{sohn2020fixmatch} & 74.5& 79.5  \\
		MDD \cite{zhang2019bridging} & 74.6  & 81.6 \\
		SENTRY \cite{prabhu2021sentry}& 76.7& -  \\
		CST\cite{liu2021cycle}& 80.6 & 86.5 \\ \hline
		CST\cite{liu2021cycle}$^\dag$ & 76.5 & 86.6 \\  \hline
		DAC-MR with $\gamma=0,\lambda=1$ & 76.6  & 87.1  \\
		DAC-MR with $\gamma=1,\lambda=1$ & \textbf{77.2}& \textbf{87.2}\\ \bottomrule
	\end{tabular}\vspace{-5mm}
\end{table}

\vspace{-8mm}
\subsection{Cross-Domain Few-Shot Learning} \label{section43}
\textbf{Formulation.} We evaluate DAC-MR in cross-domain FSL with lastest SOTA method TSA \cite{li2022cross}. Formally, given task-agnostic feature extractor $h_{\phi}$ learned from a large source dataset $D_b$, TSA adapts to target FSL tasks $(\mathcal{D}^{(s)},\mathcal{D}^{(q)})$ by learning task-specific weights, where support set $\mathcal{D}^{(s)}$ and query set $\mathcal{D}^{(q)}$ are sampled from dataset $D_t$ and $D_b$, $D_t$ contain mutually exclusive classes and domain gap.
To achieve better task-specific adaptation, TSA \cite{li2022cross} proposed an adapter $\nu = (\alpha,\beta)$ to adapt the feature extractor and classifier, respectively.
Please refer to \cite{li2022cross} for more details about the architecture of the adapter. To obtain the task-specific weight $v$, they freeze the task-agnostic weight $\phi$ and then minimize cross-entropy loss over the support samples:
\begin{align*}
	\min_{\nu} \frac{1}{N} \sum_{i=1}^N \ell(g_{(\phi,\nu)}(x_i),y_i), \mathcal{D}^{(s)} =\{(x_i,y_i)\}_{i=1}^N
\end{align*}
where $g_{(\phi,\nu)}$ is the predicted probability vector by attaching task-specific weights to a learned task-agnostic model. To better cope with domain gap between source and target tasks, we introduce DAC-MR into TAS to make task-agnostic feature extractor adapt target FSL tasks by:
\begin{align*}
	\min_{{{\phi}}}\	& \mathcal{MR}^{dac} (D;g_{(\phi,\nu^*(\phi) )},A ),	  A \in \mathcal{A},  D = \{x_i\}_{i=1}^N     \notag  \\
	&\nu^*(\phi)  = \arg\min_{\nu} \frac{1}{N} \sum_{i=1}^N \ell(g_{(\phi,\nu)}(x_i),y_i).
\end{align*}
To solve above objective, we use single step approximation of inner optimization and first-order approximation of outer optimization to improve computational efficiency. Our implementation is adapted from the official implementations of TSA available at \url{https://github.com/VICO-UoE/URL}.

\begin{table*} \vspace{-4mm}
	\caption{Worst case (WC) and average (Avg) accuracies of four domain generalization benchmarks. $^\dag$ indicates results replicated by us to the best effort. } \vspace{-4mm} \label{tablearm}
	\centering
	\scriptsize
	\setlength{\tabcolsep}{1.5mm}
	\begin{tabular}{l|c|c|c|c|c|c|c|c}
		\toprule
		\multirow{2}{*}{\textbf{Methods} }  & \multicolumn{2}{c|}{\textbf{Rotated MNIST}} & \multicolumn{2}{c|}{\textbf{FEMNIST}}  & \multicolumn{2}{c|}{\textbf{CIFAR-10-C}} & \multicolumn{2}{c}{\textbf{Tiny ImageNet-C}}   \\ \cline{2-9}
		&   WC &Avg &WC &Avg &WC &Avg& WC &Avg       \\ \hline
		ERM & 74.5 $\pm$ 1.4 & 93.6 $\pm$ 0.4 &62.4 $\pm$ 0.4 &79.1 $\pm$ 0.3& 54.1 $\pm$ 0.3 &70.4 $\pm$ 0.1 &20.3 $\pm$ 0.5 &41.9 $\pm$ 0.1 \\
		UW \cite{sagawa2019distributionally} & 80.3 $\pm$ 1.2  &95.1 $\pm$ 0.1 &65.7 $\pm$ 0.7 &80.3 $\pm$ 0.6 & - & - & - & - \\
		DRNN \cite{sagawa2019distributionally} & 79.9 $\pm$ 0.7 &94.9 $\pm$ 0.1 &57.5 $\pm$ 1.7 &76.5 $\pm$ 1.2 &49.3 $\pm$ 0.9 &65.7 $\pm$ 0.5 &14.2 $\pm$ 0.2 &31.6 $\pm$ 1.0 \\
		DANN \cite{ganin2016domain} & 78.8 $\pm$ 0.8& 94.9 $\pm$ 0.1 &65.4 $\pm$ 1.0& 81.7 $\pm$ 0.3 &53.9 $\pm$ 2.2 &69.8 $\pm$ 0.3& 20.4 $\pm$ 0.7& 40.9 $\pm$ 0.2 \\
		MMD \cite{zhang2019bridging} & 82.4 $\pm$ 0.9& 95.3 $\pm$ 0.3& 62.4 $\pm$ 0.7& 79.8 $\pm$ 0.4 &52.2 $\pm$ 0.3 &69.5 $\pm$ 0.1 &19.7 $\pm$ 0.2 &40.1 $\pm$ 0.1  \\
		BN adaptation \cite{schneider2020improving} & 78.0 $\pm$ 0.3 &94.4 $\pm$ 0.1 &65.7 $\pm$ 1.5& 80.0 $\pm$ 0.5 &60.6 $\pm$ 0.3 &70.9 $\pm$ 0.1& 26.5 $\pm$ 0.3 &42.8 $\pm$ 0.0  \\
		TTT \cite{sun2020test} &  81.1 $\pm$ 0.3 &95.4 $\pm$ 0.1 &68.6 $\pm$ 0.4 &84.2 $\pm$ 0.1& 61.5 $\pm$ 0.3& 71.7 $\pm$ 0.5& 27.6 $\pm$ 0.5 &37.7 $\pm$ 0.3  \\ \hline
		%ARM-CML & 88.0 $\pm$ 0.8 & 96.3 $\pm$ 0.4 & 70.9 $\pm$ 1.4& 86.4 $\pm$ 0.3 &61.2 $\pm$ 0.4 &70.3 $\pm$ 0.2& 29.1 $\pm$ 0.4& 43.3 $\pm$ 0.1 \\
		ARM-CML$^\dag$ & 87.6 $\pm$ 1.3 & 96.0 $\pm$ 0.5 & 70.5 $\pm$ 0.5 & 86.2 $\pm$ 0.3 & 60.5 $\pm$ 0.3 & 70.2 $\pm$ 0.3 & 28.6 $\pm$ 0.2 & 43.1 $\pm$ 0.1 \\
		ARM-CML + DAC-MR & \textbf{88.6 $\pm$ 0.6} & \textbf{96.5 $\pm$ 0.3} & \textbf{71.4 $\pm$ 0.7} & \textbf{86.6 $\pm$ 0.3} & \textbf{62.5 $\pm$ 0.5} & \textbf{72.3 $\pm$ 0.2} & \textbf{30.3 $\pm$ 0.5} & \textbf{43.4 $\pm$ 0.1} \\ \hline
		%ARM-BN & 83.3 $\pm$ 0.5 &95.6 $\pm$ 0.1& 64.5 $\pm$ 3.2 &83.2 $\pm$ 0.5& 61.7 $\pm$ 0.3 &72.4 $\pm$ 0.3 &28.3 $\pm$ 0.3 &43.3 $\pm$ 0.1  \\
		ARM-BN$^\dag$ & 83.6 $\pm$ 0.3 & 95.7 $\pm$ 0.2 & 63.9 $\pm$ 3.2 & 83.0 $\pm$ 0.8 & 61.8 $\pm$ 0.3 & 72.6 $\pm$ 0.3 & \textbf{28.6 $\pm$ 0.3} & 43.4 $\pm$ 0.1 \\
		ARM-BN + DAC-MR & \textbf{83.8 $\pm$ 0.4} & \textbf{95.9 $\pm$ 0.1} & \textbf{68.2 $\pm$ 1.5} & \textbf{84.8 $\pm$ 0.2} & \textbf{64.4 $\pm$ 0.3} & \textbf{74.3 $\pm$ 0.3} & 28.3 $\pm$ 0.4 & \textbf{43.6 $\pm$ 0.4} \\  \hline
		%ARM-LL& 88.9 $\pm$ 0.8 &96.9 $\pm$ 0.2& 67.0 $\pm$ 0.9 &84.3 $\pm$ 0.7& 61.2 $\pm$ 0.7& 72.5 $\pm$ 0.4 &25.4 $\pm$ 0.1& 35.7 $\pm$ 0.4 \\
		ARM-LL$^\dag$ & 88.9 $\pm$ 0.4 & 97.0 $\pm$ 0.1 & 68.9 $\pm$ 1.5 & 84.9 $\pm$ 0.1 & 61.1 $\pm$ 0.5 & 72.2 $\pm$ 0.1 & 21.2 $\pm$ 3.6 & 31.2 $\pm$ 3.8 \\
		ARM-LL + DAC-MR & \textbf{89.3 $\pm$ 0.9}& \textbf{97.0 $\pm$ 0.2} & \textbf{69.5 $\pm$ 2.0} & \textbf{86.3 $\pm$ 0.4} & \textbf{63.1 $\pm$ 0.6} & \textbf{74.3 $\pm$ 0.2} & \textbf{21.2 $\pm$ 3.0} & \textbf{31.6 $\pm$ 3.3} \\
		\bottomrule
	\end{tabular}\vspace{-2mm}
\end{table*}

\begin{table*} \vspace{-2mm}
	\caption{Classification results of standard transfer learning benchmarks (ResNet-50 pre-trained) on three datasets for baselines and our DAC-MR.}  \label{finetuning}\vspace{-4mm}
	\centering
	\scriptsize
	\setlength{\tabcolsep}{1.5mm}
	\begin{tabular}{c|c|cccc|c|cccc|c}
		\toprule
		\multirow{3}{*}{\textbf{Dataset} }	 & \multirow{3}{*}{\textbf{Method} } & \multicolumn{5}{c|}{\textbf{Supervised Pre-trained}}  & \multicolumn{5}{c}{\textbf{MoCo Pre-trained}}  \\  \cline{3-12}
		&   & \multicolumn{5}{c|}{Sampling Rates}  &\multicolumn{5}{c}{Sampling Rates} \\ \cline{3-12}
		&  & 15\%  & 30\% & 50\% & 100\%  &  Avg & 15\%  & 30\% & 50\% & 100\%  &  Avg \\ \hline
		\multirow{8}{*}{\textbf{CUB-200-2011} }	&  Fine-tune & 51.3$\pm$0.1 & 64.4$\pm$0.3 & 74.7$\pm$0.3 & 81.4$\pm$0.6 & 68.0 & 28.0$\pm$0.5 & 47.8$\pm$0.1 & 62.3$\pm$0.5 & 76.2$\pm$0.2 & 53.6 \\
		& Fine-tune + DAC-MR & \textbf{53.2$\pm$0.1} & \textbf{65.3$\pm$0.2} & \textbf{75.6$\pm$0.1} & \textbf{81.9$\pm$0.4} & \textbf{69.0} & \textbf{32.1$\pm$0.1} & \textbf{53.1$\pm$0.4} &\textbf{ 64.9$\pm$0.3} & \textbf{77.7$\pm$0.3} & \textbf{57.0} \\  \cline{2-12}
		& $L^2$-SP \cite{xuhong2018explicit} & 51.1$\pm$0.1 & 65.1$\pm$0.1 & 74.6$\pm$0.5 & 81.7$\pm$0.1 & 68.1 & 27.9$\pm$0.3 & 53.4$\pm$0.4 & 67.8$\pm$0.2 & 78.0$\pm$0.3 & 56.8 \\
		& $L^2$-SP + DAC-MR & \textbf{52.8$\pm$0.2} & \textbf{66.4$\pm$0.1} & \textbf{75.3$\pm$0.2} & \textbf{82.1$\pm$0.3} & \textbf{69.2} & \textbf{30.5$\pm$0.1} & \textbf{55.1$\pm$0.1} & \textbf{68.5$\pm$0.6} & \textbf{78.1$\pm$0.1} & \textbf{58.0} \\ \cline{2-12}
		& DELTA \cite{li2018delta} & 55.0$\pm$0.1 & 67.2$\pm$0.2 & 76.2$\pm$0.2 & 82.3$\pm$0.1 & 70.2 & 26.9$\pm$0.4 & 51.3$\pm$6.1 & 64.5$\pm$0.8 & 75.0$\pm$0.3 & 54.4 \\
		& DELTA  + DAC-MR & \textbf{55.4$\pm$0.1} &\textbf{ 68.1$\pm$0.1} & \textbf{76.5$\pm$0.1} & \textbf{82.4$\pm$0.3} & \textbf{70.6 }& \textbf{31.5$\pm$0.2} & \textbf{55.1$\pm$0.1 }&\textbf{ 66.2$\pm$0.2 }& \textbf{75.2$\pm$0.1} & \textbf{57.1} \\ \cline{2-12}
		& Co-Tuning \cite{you2020co} & 57.6$\pm$0.3 & 70.1$\pm$0.2 & 77.4$\pm$0.5 & 82.7$\pm$0.1 & 72.0 & 29.3$\pm$0.1 & 49.5$\pm$0.1 & 63.1$\pm$0.4 & 76.1$\pm$0.1 & 54.5 \\
		& Co-Tuning  + DAC-MR & \textbf{58.3$\pm$0.1} &\textbf{ 70.2$\pm$0.3} & \textbf{78.2$\pm$0.4} & \textbf{83.1$\pm$0.1} &\textbf{72.5} & \textbf{31.0$\pm$0.2} & \textbf{53.3$\pm$0.3} & \textbf{67.0$\pm$0.3} & \textbf{78.0$\pm$0.2} & \textbf{57.3} \\ \hline \hline
		\multirow{8}{*}{\textbf{Stanford Cars} }	&  Fine-tune  & 41.1$\pm$0.1 & 66.0$\pm$0.2 & 78.2$\pm$0.2 & 87.8$\pm$0.1 & 68.4 & 44.2$\pm$0.6 & 71.6$\pm$0.1 & 83.0$\pm$0.3 & 90.0$\pm$0.1 & 72.2 \\
		& Fine-tune + DAC-MR & \textbf{43.7$\pm$0.2 }& \textbf{68.1$\pm$0.2} & \textbf{80.0$\pm$0.1 }&\textbf{ 88.3$\pm$0.2 }&\textbf{70.0} & \textbf{49.8$\pm$0.1} & \textbf{74.2$\pm$0.2 }& \textbf{84.1$\pm$0.3} & \textbf{90.7$\pm$0.1 }&\textbf{74.7} \\  \cline{2-12}
		& $L^2$-SP \cite{xuhong2018explicit} & 42.4$\pm$0.3 & 68.1$\pm$0.1 & 79.7$\pm$0.1 & 88.4$\pm$0.2 & 70.0 & 46.4$\pm$0.6 & 75.1$\pm$0.2 & 84.0$\pm$0.5 & 89.9$\pm$0.1 & 73.9 \\
		& $L^2$-SP + DAC-MR & \textbf{44.7$\pm$0.2} & \textbf{69.1$\pm$0.3} & \textbf{80.5$\pm$0.1} & \textbf{88.5$\pm$0.1} &\textbf{70.7} & \textbf{50.9$\pm$0.4} & \textbf{76.7$\pm$0.1} & \textbf{84.5$\pm$0.2} & \textbf{90.0$\pm$0.1} & \textbf{75.5} \\ \cline{2-12}
		& DELTA \cite{li2018delta} & 45.0$\pm$0.1 & 68.2$\pm$0.2 & 79.8$\pm$0.2 & 88.2$\pm$0.2 & 70.3 & 45.9$\pm$0.4 & 73.1$\pm$0.1 & 83.2$\pm$0.1 & 89.1$\pm$0.4  & 73.1 \\
		& DELTA + DAC-MR & \textbf{46.4$\pm$0.2} & \textbf{69.2$\pm$0.1} & \textbf{80.6$\pm$0.1} & \textbf{88.3$\pm$0.1} & \textbf{71.1} & \textbf{53.0$\pm$0.4} & \textbf{76.9$\pm$0.4} & \textbf{83.8$\pm$0.2} & \textbf{89.6$\pm$0.2} & \textbf{75.8} \\ \cline{2-12}
		& Co-Tuning \cite{you2020co} & 48.4$\pm$0.5 & 71.1$\pm$0.7 & 81.9$\pm$0.2 & 89.1$\pm$0.1 & 76.6 & 44.0$\pm$0.4 & 72.2$\pm$0.2& 83.3$\pm$0.4 & 90.3$\pm$0.1 & 72.5 \\
		& Co-Tuning  + DAC-MR & \textbf{55.6$\pm$0.2} & \textbf{76.8$\pm$0.4} & \textbf{85.0$\pm$0.3} & \textbf{91.3$\pm$0.4} & \textbf{77.2} & \textbf{49.2$\pm$0.2} & \textbf{77.8$\pm$0.4} & \textbf{86.9$\pm$0.2} & \textbf{91.3$\pm$0.3} & \textbf{76.3} \\ \hline \hline
		\multirow{8}{*}{\textbf{FGVC Aircraft} }	&  Fine-tune  & 41.6$\pm$0.7 & 58.1$\pm$0.1 & 68.4$\pm$0.1 & 80.3$\pm$0.5 & 62.1 & 46.6$\pm$0.5 & 67.8$\pm$0.3 & 78.9$\pm$0.5 & 88.0$\pm$0.1 & 70.3 \\
		& Fine-tune + DAC-MR & \textbf{44.9$\pm$0.1} & \textbf{60.6$\pm$0.2} & \textbf{70.1$\pm$0.1} & \textbf{81.8$\pm$0.3} & \textbf{64.4} & \textbf{50.1$\pm$0.1} & \textbf{71.6$\pm$0.1} & \textbf{80.5$\pm$0.1} & \textbf{88.4$\pm$0.2} & \textbf{72.7} \\  \cline{2-12}
		& $L^2$-SP \cite{xuhong2018explicit} & 43.4$\pm$0.6 & 61.3$\pm$0.1 & 70.5$\pm$0.2 & 82.1$\pm$0.2 & 64.3 & 48.3$\pm$0.1 & 74.0$\pm$0.8 & 81.6$\pm$0.8 & 89.1$\pm$0.1 & 73.3 \\
		& $L^2$-SP + DAC-MR & \textbf{46.1$\pm$0.1} & \textbf{64.1$\pm$0.3} & \textbf{72.4$\pm$0.1} & \textbf{83.1$\pm$0.2} & \textbf{66.4} & \textbf{52.0$\pm$0.1} & \textbf{75.2$\pm$0.1} & \textbf{82.4$\pm$0.1} & \textbf{89.2$\pm$0.1} & \textbf{74.7} \\ \cline{2-12}
		& DELTA \cite{li2018delta} & 44.4$\pm$0.2 & 61.6$\pm$0.3 & 71.3$\pm$0.1 & 82.7$\pm$0.2 & 65.0 & 32.6$\pm$2.7 & 64.1$\pm$0.2 & 70.7$\pm$4.4 &  77.9$\pm$3.7 & 62.2 \\
		& DELTA  + DAC-MR & \textbf{46.0$\pm$0.2} & \textbf{63.1$\pm$0.3} & \textbf{72.3$\pm$0.1} & \textbf{82.8$\pm$0.1} & \textbf{66.1} & \textbf{37.7$\pm$0.5} & \textbf{65.6$\pm$0.5} & \textbf{78.9$\pm$0.3} & \textbf{80.4$\pm$0.4} & \textbf{65.7} \\ \cline{2-12}
		& Co-Tuning \cite{you2020co} & 45.5$\pm$0.8 & 60.8$\pm$0.5 & 71.6$\pm$0.6 & 82.1$\pm$0.4 & 65.0 & 47.0$\pm$0.8 & 68.1$\pm$0.5 & 78.8$\pm$0.4 & 87.8$\pm$0.3 & 70.4 \\
		& Co-Tuning  + DAC-MR & \textbf{52.9$\pm$0.4} & \textbf{68.5$\pm$0.1} & \textbf{76.1$\pm$0.3} & \textbf{85.2$\pm$0.3} & \textbf{70.7} & \textbf{54.5$\pm$0.2} & \textbf{75.4$\pm$0.3} & \textbf{83.8$\pm$0.5} & \textbf{89.0$\pm$0.2} & \textbf{75.7} \\
		\bottomrule
	\end{tabular}\vspace{-6mm}
\end{table*}

\noindent\textbf{Results.} Table \ref{tabletsa} shows that performance comparison with SOTA methods of cross-domain FSL benchmarks on Meta-Dataset \cite{triantafillou2019meta}. We report the average classification accuracy in previously seen domains, unseen domains, all domains and the average rank. As can be seen,
our method outperforms TSA and other competing methods on most domains (11 out of 13), especially obtaining significant improvement on 5 unseen datasets than TSA method, i.e., Average Unseen (+2.3).
Achieving improvement on unseen domains is more challenging due to the large gap between seen and unseen domains and the scarcity of labeled samples for the unseen task.
Nevertheless, DAC-MR can help improve TAS for better adaptation and generalization to unseen domains benefiting from the meta-knowledge information, and hence achieves very competitive results. This fully complies with the theoretical analysis in \S\ref{section35}, and validates that DAC-MR is hopeful to cope with such challenging scenes.

\vspace{-4mm}
\subsection{Transductive / Semi-Supervised Few-Shot Learning} \label{section42}
\textbf{Formulation.} Transductive FSL leverages the distributions of examples in query set $\mathcal{D}^{(q)}$ and given support set $\mathcal{D}^{(s)}$, and the feature extractor $h_{\phi}$ is learned from a large source dataset $D_b$ to make predictions on $\mathcal{D}^{(q)}$, where $(\mathcal{D}^{(s)},\mathcal{D}^{(q)})$ are sampled from a novel target dataset $D_t$. The key idea for transductive FSL is to predict confident pseudo-labels on the query set $\mathcal{D}^{(q)}$ with the help of some semi-supervised techniques.
For semi-supervised FSL, it follows the same solution with $\mathcal{D}^{(q)}$ replaced by unlabelled samples $U$. Here, we consider iLPC \cite{lazarou2021iterative} method due to its recent SOTA transductive/semi-supervised FSL performance.
We aim to meta-optimize feature extractor adapting to target FSL tasks by solving the below objective, and then use the updated feature extractor to produce more confident pseudo-labels on the query (unlabelled) set:
\begin{align*}
	\min_{{{\phi}}}\	& \mathcal{MR}^{dac} (\mathcal{D}^{(s)} \cup \mathcal{D}^{(q)}; f\circ h_{\phi},  A ),	  A \in \mathcal{A},  \notag  \\
s.t. \ &\{ \hat{y}^q_k \}_{k=1}^{|\mathcal{Q}|}  = \text{Alg}_{iLPC}(\mathcal{D}^{(s)}, \mathcal{D}^{(q)}, h_\phi),
\end{align*}
where $\text{Alg}_{iLPC}(\mathcal{D}^{(s)}, \mathcal{D}^{(q)}, h_\phi)$ denotes the set of pseudo-labels predicted by iLPC \cite{lazarou2021iterative} for $\mathcal{D}^{(q)}$, and we choose the best confident pseudo-labels to build a query subset $\mathcal{D}^{(q)'} = \{(x_l^q,\hat{y}_l^q)\}_{l=1}^{|\mathcal{D}^{(q)}|/2}$ with a half size of $\mathcal{D}^{(q)}$. Then we build $f$ by computing prototype from $\mathcal{D}^{(s)} \cup \mathcal{D}^{(q)'}$ as ProtoNet \cite{snell2017prototypical}.
Our implementation is built upon the code of iLPC\cite{lazarou2021iterative} available at \url{https://github.com/MichalisLazarou/iLPC}.

\noindent\textbf{Results.} Table \ref{tableilpc} presents the results of transductive/semi-supervised FSL benchmarks on tieredImageNet, CIFAR-FS and CUB datasets. As can be seen, the performance of DAC-MR is superior to iLPC in all settings. Especially, DAC-MR can bring more notable gains to 1-shot accuracy
than to 5-shot in three datasets.
Though we cannot access labels from query (unlabelled) data, DAC-MR can provide supplemental meta-knowledge information to help seek feature extractor adapting to target FSL tasks, which leads to more confident pseudo-labels, and hence obtain better performance.
This confirms the theoretical analysis in \S\ref{section35}.

\begin{table*}[t] \vspace{-2mm}
	\caption{Results of task-incremental learning benchmarks on CIFAR-100 and TinyImagenet-200. Results of baselines are copied from La-MAML \cite{gupta2020look}.  }\label{tablelamaml}  \vspace{-4mm}
	\centering
	\scriptsize
	\setlength{\tabcolsep}{1.5mm}
	%\resizebox{0.95\textwidth}{13mm}{
		\begin{tabular}{l|c|c|c|c|c|c|c|c}
			\toprule
			\multirow{3}{*}{\textbf{Model} }	& 	\multicolumn{4}{c|}{CIFAR-100}   & 	\multicolumn{4}{c}{TinyImagenet-200}  \\
			\cline{2-9}
			& 	\multicolumn{2}{c|}{Multiple} & \multicolumn{2}{c|}{Single}   & 	\multicolumn{2}{c|}{Multiple} & \multicolumn{2}{c}{Single}  \\  \cline{2-9}
			& RA & BTI & RA & BTI & RA & BTI & RA & BTI  \\ \hline
			iCaRL \cite{rebuffi2017icarl} &  60.47 $\pm$ 1.09 & -15.10 $\pm$ 1.04 & 53.55 $\pm$ 1.69 & -8.03 $\pm$ 1.16 & 54.77 $\pm$ 0.32 & \textbf{-3.93 $\pm$ 0.55} & 45.79 $\pm$ 1.49 & \textbf{-2.73 $\pm$ 0.45} \\
			GEM \cite{lopez2017gradient} & 62.80 $\pm$ 0.55  &-17.00 $\pm$ 0.26  &48.27 $\pm$ 1.10  &-13.70 $\pm$ 0.70 & 50.57 $\pm$ 0.61 & -20.50 $\pm$ 0.10  &40.56 $\pm$ 0.79  &-13.53 $\pm$ 0.65  \\
			AGEM \cite{chaudhry2018efficient} &  58.37 $\pm$ 0.13  &-17.03 $\pm$ 0.72  &46.93 $\pm$ 0.31 & -13.40 $\pm$ 1.44  &46.38 $\pm$ 1.34 & -19.96 $\pm$ 0.61  &38.96 $\pm$ 0.47  &-13.66 $\pm$ 1.73 \\
			MER \cite{riemer2018learning} & - & - & 51.38 $\pm$ 1.05& -12.83 $\pm$ 1.44& -& -& 44.87 $\pm$ 1.43 &-12.53 $\pm$ 0.58   \\
			C-MAML & 65.44 $\pm$ 0.99 & -13.96 $\pm$ 0.86  &55.57 $\pm$ 0.94  &-9.49 $\pm$ 0.45  &61.93 $\pm$ 1.55  &-11.53 $\pm$ 1.11  &48.77 $\pm$ 1.26 & -7.60 $\pm$ 0.52 \\
			SYNC & 67.06 $\pm$ 0.62 & -13.66 $\pm$ 0.50  &58.99 $\pm$ 1.40 & -8.76 $\pm$ 0.95 & 65.40 $\pm$ 1.40  &-11.93 $\pm$ 0.55 & 52.84 $\pm$ 2.55 &-7.30 $\pm$ 1.93  \\ \hline
			LA-MAML & 70.08 $\pm$ 0.66 &-9.36 $\pm$ 0.47 & 61.18 $\pm$ 1.44  &-9.00 $\pm$ 0.20  &66.99 $\pm$ 1.65  &-9.13 $\pm$ 0.90 & 52.59 $\pm$ 1.35  &-3.70 $\pm$ 1.22  \\
			%LA-MAML & 69.33 $\pm$ 1.27  & -9.10 $\pm$ 1.20 & 60.73 $\pm$ 0.57 & -8.75 $\pm$ 0.1 & 65.79 $\pm$ 0.51 & -9.30 $\pm$ 0.40 & 51.53 $\pm$ 0.98 & -2.5 $\pm$ 0.33 \\
			LA-MAML + DAC-MR & \textbf{70.61 $\pm$ 0.30} & \textbf{-7.57 $\pm$ 1.06} & \textbf{62.26 $\pm$ 0.74} & \textbf{-7.80 $\pm$ 0.62} & \textbf{68.86 $\pm$ 0.47} & \textbf{-7.37 $\pm$ 0.85}& \textbf{55.53 $\pm$ 1.14} & \textbf{-2.76 $\pm$ 0.12} \\
			
			\bottomrule
		\end{tabular}	\vspace{-4mm}
	\end{table*}

\begin{table*}[t]
	\caption{Results of FSCIL benchmarks on CIFAR100, MiniImageNet and CUB200 datasets. Other results are copied from the corresponding papers.  }\label{tablefewcontinual}  \vspace{-3mm}
	\centering
	\scriptsize
	%\resizebox{0.95\textwidth}{13mm}{
		\begin{tabular}{l|ccccccccc|c|c}
			\toprule
			\multirow{2}{*}{\textbf{Model} }	& 	\multicolumn{9}{c|}{Sessions (CIFAR-100 5-way 5-shot w/ResNet20)}   & {Average } & Final  \\
			\cline{2-10}
			& 0 &  1  &  2 & 3 & 4  &  5  &  6  &  7  &  8  &  Accuracy & Improvement    \\ \hline
			iCaRL \cite{rebuffi2017icarl} &  64.10 & 53.28 &  41.69 & 34.13 & 27.93 & 25.06 & 20.41 & 15.48 &  13.73 & 32.87 &  +35.51 \\
			TOPIC \cite{tao2020few} &  64.10 & 55.88 & 47.07 & 45.16 & 40.11 &  36.38 & 33.96 & 31.55 & 29.37 & 42.62  & +19.87  \\
			SPPR \cite{zhu2021self} &  64.10 &  65.86 &61.36& 57.34& 53.69 &50.75& 48.58& 45.66& 43.25& 54.51& +5.99\\
			CEC \cite{zhang2021few} & 73.07 & 68.88 &  65.26 &  61.19 & 58.09 & 55.57 & 53.22 & 51.34 & 49.14 & 59.53 & +0.1\\	 \hline	
			CEC + DAC-MR &  \textbf{73.08} & \textbf{69.12} & \textbf{65.30} & \textbf{61.37} & \textbf{58.21} & \textbf{55.64} & \textbf{53.32} & \textbf{51.36} & \textbf{49.24} & \textbf{59.63}  & \\
			\bottomrule
		\end{tabular}	\vspace{1mm}
		
		\begin{tabular}{l|ccccccccc|c|c}
			\toprule
			\multirow{2}{*}{\textbf{Model} }	& 	\multicolumn{9}{c|}{Sessions (MiniImageNet 5-way 5-shot w/ResNet18)}   & Average & Final\\
			\cline{2-10}
			& 0 &  1  &  2 & 3 & 4  &  5  &  6  &  7  &  8  & Accuracy & Improvement     \\ \hline
			iCaRL \cite{rebuffi2017icarl} &  61.31 & 46.32 & 42.94 & 37.63 & 30.49 & 24.00 & 20.89 & 18.80 & 17.21 & 33.29 & +30.74 \\
			TOPIC \cite{tao2020few} &  61.31 & 50.09 & 45.17 &41.16 & 37.48 & 35.52 &32.19 & 29.46 &24.42&  39.64 & +23.53  \\
			SPPR \cite{zhu2021self} &  61.45& 63.80& 59.53 &55.53& 52.50& 49.60 &46.69 &43.79 &41.92& 52.75 & +6.03\\
			CEC \cite{zhang2021few} & 72.00 & 66.83 &62.97 & 59.43 & 56.70 &53.73 & 51.19 & 49.24 & 47.63 & 57.75  & +0.32 \\	 \hline	
			CEC + DAC-MR &\textbf{72.30} & \textbf{67.39} & \textbf{63.30} & \textbf{59.93} &\textbf{57.24} & \textbf{54.15} & \textbf{51.71} & \textbf{49.67} &\textbf{47.95} & \textbf{58.18} & \\
			\bottomrule
		\end{tabular} \vspace{1mm}
		
		\begin{tabular}{l|ccccccccccc|c|c}
			\toprule
			\multirow{2}{*}{\textbf{Model} }	& 	\multicolumn{11}{c|}{Sessions (CUB200 10-way 5-shot w/ResNet18)}   & Average & Final \\
			\cline{2-12}
			& 0 &  1  &  2 & 3 & 4  &  5  &  6  &  7  &  8  & 9 & 10 &  Accuracy & Improvement     \\ \hline
			iCaRL \cite{rebuffi2017icarl} &  68.68 & 52.65 & 48.61 & 44.16 & 36.62 & 29.52 & 27.83 & 26.26 & 24.01 & 23.89 & 21.16 & 36.67  &+31.13 \\
			TOPIC \cite{tao2020few} &   68.68& 62.49& 54.81& 49.99& 45.25& 41.40 & 38.35& 35.36& 32.22& 28.31& 26.28 &43.92 & +26.01\\
			SPPR \cite{zhu2021self} & 68.68& 61.85 &57.43& 52.68 &50.19 &46.88 &44.65& 43.07 &40.17 &39.63 &37.33 &49.32& +14.96 \\
			CEC \cite{zhang2021few} & 75.85 & 71.94 & 68.50 & 63.50 & 62.43 & 58.27 & 57.73 & \textbf{55.81} & \textbf{54.83} & \textbf{53.52} & {52.28} & 61.33  & +0.01\\	 \hline	
			CEC + DAC-MR & \textbf{76.57} & \textbf{72.28} & \textbf{68.67} & \textbf{63.72} & \textbf{62.82} & \textbf{58.40} & \textbf{57.79} & {55.75} & {54.77} & {53.48} & \textbf{52.29} & \textbf{61.50}& \\
			\bottomrule
		\end{tabular} \vspace{-4mm}
	\end{table*}

\vspace{-4mm}
\section{DAC-MR Benefits Transfer Learning}\label{section5}

In this section, we study the influence of DAC-MR on meta-learning for three typical transfer learning tasks, containing unsupervised domain adaptation (\S\ref{section51}), domain generalization (\S\ref{section52}), and transfer learning with fine-tuning (\S\ref{section53}).

\vspace{-4mm}
\subsection{Unsupervised Domain Adaptation}  \label{section51}
\textbf{Formulation.}
We mainly apply DAC-MR on top of CST \cite{liu2021cycle}, due to its SOTA domain adaptation performance, expecting to further boost its performance. Specifically, we have access to $m$ labeled i.i.d. samples $\mathcal{D}^{(s)} = \{x_i^{(s)}, y_i^{(s)}\}_{i=1}^{m}$ from $S$ and $n$ unlabeled i.i.d. samples $\mathcal{D}^{(q)} = \{x_i^{(q)}\}_{i=1}^{n}$ from $Q$.
We denote the shared meta-representation as $h_{\phi}$, and the source and target classifiers trained on top of meta-representation $h_{\phi}$ as $f_s$, $f_q$, respectively. To train $f_q$ on unlabeled target data, it uses trained $f_s$ to generate target pseudo-labels as
\begin{align} \label{eqpse}
{y'}^{(q)} = \mathop{\arg\min}_i \{ f_{s} \circ h_{\phi}(x^{(q)})_{[i]}\},
\end{align}
for each $x^{(q)}$ in the target dataset $\mathcal{D}^{(q)}$. We introduce DAC-MR into CST to encourage learned classifiers to be robust against domain shift and produce reliable pseudo-labels:
\begin{align}
	f^*_s, \phi^*  & = \mathop{\arg\min}_{f_s, \phi}  \mathcal{L} (\mathcal{D}^{(s)}; f_s, h_{\phi}) + \gamma \mathcal{L} (\mathcal{D}^{(s)}; f^*_q(h_{\phi}), h_{\phi})  \notag \\
	& +  \lambda \mathcal{MR}^{dac} (\mathcal{D}^{(q)};f^*_q(h_{\phi}) \circ h_{\phi},A ),	  A \in \mathcal{A},   \notag \\
	& f^*_q(h_{\phi})  = \mathop{\arg\min}_{f_q} \mathcal{L}(\tilde{\mathcal{D}}^{(q)}; f_q,h_{\phi}),  \label{eqdaccst}
\end{align}
where $\tilde{\mathcal{D}}^{(q)} = \{x_i^{(q)}, {y'}_i^{(q)}\}_{i=1}^{n_q}$ is with pseudo-labels generated by the source classifier $f_s$ in Eq.(\ref{eqpse}), and $\mathcal{L}$ is the cross-entropy loss. If we set $\gamma=1, \lambda=0$, it is degenerated to the original CST \cite{liu2021cycle}.
CST focuses on improving the quality of pseudo-labels by optimizing $\mathcal{L} (\mathcal{D}^{(s)}; f^*_q(h_{\phi})$, while we pay more attention on enhancing the tolerance of meta-model to the domain shift of traget domain, i.e., minimizing $\mathcal{MR}^{dac} (\mathcal{D}^{(q)};f^*_q(h_{\phi}) \circ h_{\phi},A )$, which
enforces the model to predict stably under simulated domain shift by some data augmentations. In our experiments, we consider two novel cases compared with CST: 1) $\gamma=0,\lambda=1$, we only require meta-model to be capable of addressing the domain shift; 2) $\gamma=1,\lambda=1$, we require meta-model to struggle against the domain shift and produce reliable pseudo-labels. We just follow the training algorithm of CST to optimize Eq.(\ref{eqdaccst}), and implement our method based on the official implementations of CST available at \url{https://github.com/Liuhong99/CST}.

\noindent\textbf{Results.} Table \ref{tableoffice} shows the results on 12 pairs of Office-Home tasks. Due to the domain shift, standard self-training methods (e.g., VAT and Fixmatch) may fail to produce reliable pseudo-label, and CST \cite{liu2021cycle} would improve the pseudo-label quality and obtain fine results. However, it does not consider feature adaptation for domain shift. Our DAC-MR encourages the feature representation to be robust against domain shift adapting to target domain. When
we only focus on feature adaptation ($\gamma=0,\lambda=1$), it obtains similar or even slightly better performance than CST. Once we  require to struggle against the domain shift and produce reliable pseudo-labels ($\gamma=1,\lambda=1$), it consistently improves CST on all tasks, and outperforms other methods significantly in 12 tasks. Note that we do not involve careful hyperparameter tuning process, and it sometimes suffers from slight decay in performance compared to only feature adaptation.
Table \ref{tablevisda} shows the results on VisDA-2017, which further shows the effectiveness of our DAC-MR.
Particularly, we do not use ground-truth labels of target data to evaluate meta-model $h_{\phi}$. This supports the theoretical analysis in \S\ref{section35}.

\vspace{-2mm}
\subsection{Domain Generalization}  \label{section52}
\textbf{Formulation.} We study the setting that learns models adapting to domain shift at the testing stage. We apply DAC-MR on top of the SOTA meta-learning method ARM \cite{zhang2021adaptive} for domain generalization, aiming to further boost its performance. Formally, ARM optimizes the following objective:
\begin{align*}
\min_{\theta,\phi} \varepsilon(\theta,\phi) = \frac{1}{K}\sum_{k=1}^K \mathcal{L}(g(\theta');D_k),  D_k=\{(x^k_i,y^k_i)\}_{i=1}^N,
\end{align*}
where $\mathcal{L}(g(\theta');D_k) = \frac{1}{N} \sum_{i=1}^N \ell(g(x_i; \theta'),y_i)$, the prediction model $g(\cdot;\theta): \mathcal{X}\rightarrow \mathcal{Y}$ is parameterized by $\theta \in \Theta$ and predicts $y$ given $x$, and $\theta' = h(\theta,x_1,\cdots,x_K;\phi)$ is adaptive parameters of $g$, which is produced by adaptation model $h: \Theta \times \mathcal{X}^K \rightarrow \Theta$ through inputting parameters $\theta$ and $K$ unlabeled samples. Note that ARM has no access to meta-data representing the domain generalization goal, and we attempt to exploit DAC-MR to achieve this aim by solving:
\begin{align*}
	  \min_{\hat{\theta},\hat{\phi}} \frac{1}{K}\sum_{k=1}^K & \mathcal{MR}^{dac} (D'_k;g(\hat{\theta}'),A ),	  A \in \mathcal{A}, D'_k=\{x^k_i\}_{i=1}^N \notag \\
&	\hat{\theta},\hat{\phi} = \mathop{\arg\min}_{\theta,\phi} \frac{1}{K}\sum_{k=1}^K \mathcal{L}(g(\theta');D_k),
\end{align*}
where $\hat{\theta}'=h(\hat{\theta},x_1,\cdots,x_K;\hat{\phi}) $. Our DAC-MR enforces the trained models to make prediction stably under simulated domain shift by some domain augmentations. We achieve our algorithm via following ARM's official implementations available at \url{https://github.com/henrikmarklund/arm}.

\noindent\textbf{Results.} Table \ref{tablearm} reports the results on four domain generalization benchmarks. Across all testbeds, DAC-MR improves both worst case and average accuracies of three variants of ARM methods in almost all cases, implying that DAC-MR is comparatively less reliant on favorable inductive biases for domain shift and consistently attains better results.
%We also evaluate DAC-MR on the WILDS benchmark \cite{koh2021wilds}, and achieve advantage results over ARM (More details can be found in Appendix C.2).
Though we cannot access data from new domains, DAC-MR can encourage the learned models to behave robust against domain shift by leveraging the meta-knowledge about invariance, validating the properness of our theoretical analysis presented in \S\ref{section35} for domain generalization.

\vspace{-4mm}
\subsection{Transfer Learning with Fine-tuning}  \label{section53}
\textbf{Formulation.} Given a DNN model pre-trained on a source dataset $\mathcal{D}^{(s)}$, transfer learning aims to fine-tune it to fit a target dataset $\mathcal{D}^{(q)}= \{x_i^{(q)}, y_i^{(q)}\}_{i=1}^{n}$. Generally, $\mathcal{D}^{(s)}$ and $\mathcal{D}^{(q)}$ share the same input space $\mathcal{X}$ but have
respective category spaces $\mathcal{Y}_{s}$ and $\mathcal{Y}_{q}$. In computer vision, $\mathcal{D}^{(s)}$ is often large-scale, e.g., ImageNet, and $\mathcal{Y}_{q}$ is the visual classification dataset we concern. To overcome heterogeneous label space, it often splits the pre-trained DNN into two parts: a shared representation function $h_{\phi_0}$ and a task-specific function $f_s$, which builds upon $h_{\phi_0}$. In the fine-tuning stage, the $h_{\phi_0}$ is retained and the $f_s$ is replaced by a randomly initialized function $f_q$, whose output space matches $\mathcal{Y}_{q}$. Then the vanilla fine-tuning method optimizes the following non-convex optimization with a good starting point $h_{\phi_0}$,
\begin{align*}
	\min_{f_q, h_{\phi}} \mathcal{L}(\mathcal{D}^{(q)}; f_q \circ h_{\phi}),
\end{align*}
where $\mathcal{L}(\mathcal{D}^{(q)}; f_q \circ h_{\phi}) = \frac{1}{n} \sum_{i=1}^{n} \ell\left(f_q \circ h_{\phi}(x_i^{(q)}),y_i^{(q)} \right)$. Recently, various regularization techniques are proposed to help alleviate over-fitting, and more details can be found in Appendix C.3. Different from them, we use DAC-MR as a meta-regularizer to promote fine-tuning to perform robust against domain shift from a meta-learning perspective:
%\begin{align}
% \min_{f_q, {{\phi}}}\	& \mathcal{MR}^{dac} (\mathcal{D}^{(q)};f_q' \circ  {\phi}',A ),	  A \in \mathcal{A},  \notag  \\
%	&f_q', {{\phi}'}  = f_q, {{\phi}} -\nabla_{f_q, {{\phi}}} \mathcal{L}(\mathcal{D}^{(q)}; f_q \circ h_{\phi}).     \label{eqtr}
%\end{align}
\begin{align}
 \min_{ {{\phi'}}}\	& \mathcal{MR}^{dac} ({D};  f_{q'} \circ h_{\phi'} ,A ),	  A \in \mathcal{A},  \notag  \\
& f_{q'} , h_{\phi'} = \mathop{\arg\min}_{f_q , h_{\phi}} \mathcal{L}(\mathcal{D}^{(q)}; f_q \circ h_{\phi}),
\end{align}
where ${D} = \{x_i^{(q)}\}_{i=1}^{n}$.
This formulation requires to optimize the representation function such that fine-tuning performance on the target task could produce maximally robust behavior against some data augmentations. We use multiple gradient updates for inner-level optimization to improve its computational efficiency.
%We also use a first-order approximation to improve its computational efficiency.

\begin{table}[t]\vspace{-2mm}
	\caption{Classification results of label noise learning benchmark on CIFAR-10 and CIFAR-100 with synthetic symmetric and asymmetric noises. Other results are taken from ELR+\cite{liu2020early} and AugDesc \cite{nishi2021augmentation}. }\label{tablemwnet}  \vspace{-4mm}
	\centering
	\scriptsize
	\tabcolsep =0.6mm
	%\resizebox{0.49\textwidth}{10mm}{
		\begin{tabular}{l|cccc|c|cccc|c}
			\toprule
			\multirow{3}{*}{\textbf{Model} } &  \multicolumn{5}{c|}{\textbf{CIFAR-10}}   & \multicolumn{5}{c}{\textbf{CIFAR-100}} \\
			\cline{2-11}
			&  \multicolumn{4}{c|} {\textbf{Symmetric } }  & \multicolumn{1}{c|} {\textbf{Asym. } }  &  \multicolumn{4}{c|} {\textbf{Symmetric } }  & \multicolumn{1}{c}{\textbf{Asym. }}     \\ \cline{2-11}
			&   20\%   &  50\%  &  80\% &  90\%   &  40\%  &   20\%   &  50\%  &  80\% &  90\%   &  40\%   \\
			\hline
			ERM &   86.8 &79.4 & 62.9 & 42.7 & 83.2 & 62.0 & 46.7 & 19.9 & 10.1 & - \\
			Forward \cite{patrini2017making} &   86.8&  79.8 & 63.3 & 42.9 &-&  61.5&  46.6 &19.9 & 10.2&- \\
			M-correction \cite{arazo2019unsupervised}& 94.0 & 92.0 &  86.8 & 69.1 & 87.4 & 73.9& 66.1& 48.2&  24.3 & - \\
			PENCIL \cite{yi2019probabilistic} & 92.4 & 89.1 & 77.5&  58.9&88.5 &  69.4&  57.5 & 31.1&  15.3&-\\
			%Meta-Learning  & 92.9 & 89.3 & 77.4&  58.7 & - &  68.5 & 59.2 & 42.4 & 19.5 & - \\
			DivideMix \cite{li2019dividemix} & 96.1& 94.6& 93.2& 76.0 & 93.4 & 77.3 & 74.6 & 60.2 & 31.5 &31.5 \\
			ELR+ \cite{liu2020early} & 94.6 & 93.8 & 91.1 & 75.2 & 92.7 & 77.5 & 72.4 & 58.2 & 30.8 & 76.5 \\
			AugDesc \cite{nishi2021augmentation}  & 96.2 & 95.1 & 93.6 & 91.8 & 94.3 & 79.2 & 77.0 & 66.1 & 40.9 & 76.8 \\
			C2D \cite{zheltonozhskii2022contrast} & 96.2 & 95.1 & 94.3 & 93.4 & 90.8 & 78.3 & 76.1 & 67.4 & 58.5 & 75.1 \\ \hline
			Ours & \textbf{96.7} & \textbf{95.6} & \textbf{94.5} & \textbf{93.5} & \textbf{95.7} & \textbf{81.6} & \textbf{77.5} & \textbf{70.0} & \textbf{64.3} & \textbf{78.2} \\			
			\bottomrule
		\end{tabular} \vspace{-4mm}
	\end{table}

\noindent\textbf{Results.} The classification accuracies are shown in Table \ref{finetuning}. Across all sampling rates and all testing datasets, DAC-MR consistently improves the performance of baseline methods regardless of supervised or self-supervised pre-trained representations. Note that baselines can obtain fine-tuning performance when sufficient data are provided.
It can be easily observed that DAC-MR produces boosts in accuracy by large margins with fewer training data for baseline methods (e.g., 7\% absolute rise on Aircraft with a sampling rate of 15\% and 30\% in terms of Co-Tuning), indicating that DAC-MR is potentially useful for transfer learning when target data are limited. Besides, it is known that self-supervised pre-trained representations and downstream classification tasks may suffer from the large discrepancy, e.g., CUB. Though baselines can hardly perform well, DAC-MR can also yield consistent gains for all fine-tuning settings. This shows DAC-MR can help transfer learning perform robust against domain shift, and brings more profits when there are fewer target data, agreed with the result of Theorem \ref{th1}.

\vspace{-4mm}
\section{DAC-MR Benefits Continual Learning}\label{section6}

In this section, we study the influence of DAC-MR on meta-learning for continual learning, containing task-incremental learning (\S\ref{section61}) and class-incremental learning (\S\ref{section62}) tasks.

\vspace{-4mm}
\subsection{Task-Incremental Learning}  \label{section61}
\textbf{Formulation.} We apply DAC-MR on top of the La-MAML \cite{gupta2020look}, due to its SOTA meta-learning performance on task-incremental learning. Note that La-MAML exploits the samples in the replay-buffer as meta-data to compute meta-loss for updating meta-model. We additionally introduce DAC-MR computed on meta-data to further boost its meta-level generalization. More discussions can be found in Appendix D.1.
The implementation is adapted from the official implementations of LA-MAML available at \url{https://github.com/montrealrobotics/La-MAML}.

\noindent\textbf{Results.} Table \ref{tablelamaml} reports the task-incremental results on CIFAR-100 and TinyImagenet-200.
%Experiments are carried out in both the Single-Pass and Multiple-Pass settings.
As can be seen, our DAC-MR consistently improves the performance of LA-MAML on both datasets across setups, and achieves superior performance compared to other baselines. Though iCARL attains lower BTI in some setups, it takes the cost of lower performance throughout learning. Among the high-performing approaches, our method has the lowest BTI. This shows that DAC-MR is hopeful to alleviate the forgetting issue and favor positive backward transfer for the task.
We highlight the fact that meta-data-driven meta-loss of LA-MAML combined with DAC-MR can boost meta-level generalization indeed, accordant with the result in Theorem \ref{th1}.

\begin{table}[t] \vspace{-2mm}
	\caption{Classification results of label noise learning benchmark on real-world mini-WebVision dataset. Results for baseline methods are copied from C2D \cite{zheltonozhskii2022contrast}. * denotes results trained with Inception-ResNet-v2.} \label{webvision}\vspace{-4mm}
	\centering
	\tabcolsep =0.7mm
	\begin{tabular}{c|c|c|c|c}
		\toprule
		\multirow{2}{*}{Methods} &  \multicolumn{2}{c|}{WebVision} & \multicolumn{2}{c}{ILSVRC12}   \\  \cline{2-5}
		&   top1 & top5 & top1 & top5 \\
		\hline
		Forward* \cite{patrini2017making}& 61.12 &82.68 &  57.36 & 82.36  \\
		MentorNet* \cite{jiang2018mentornet}&  63.00& 81.40 &57.80& 79.92  \\
		Co-teaching* \cite{han2018co}  & 63.58 &85.20&61.48 &84.70 \\
		Interative-CV* \cite{chen2019understanding} &   65.24 &85.34& 61.60& 84.98 \\
		%CMW-Net-SL & 72.72 & 90.44&  69.88 & 89.56  \\ \hline
		DivideMix*  \cite{li2019dividemix} & 77.32 & 91.64& 75.20 & 90.84  \\
		ELR* \cite{liu2020early} & 77.78 & 91.68& 70.29 & 89.76  \\
		DivideMix \cite{li2019dividemix}&   76.32 & 90.65& 74.42 & 91.21  \\
		C2D \cite{zheltonozhskii2022contrast}& 79.42 & 92.32 & 78.57 & 93.04  \\
		\hline
		Ours& \textbf{81.44} &\textbf{94.24} & \textbf{78.76} & \textbf{94.76} \\
		\bottomrule
	\end{tabular} \vspace{-4mm}
\end{table}

\begin{table*}[ht] \vspace{-2mm}
	\caption{Classification results of transition matrix estimation benchmark on CIFAR-10 and CIFAR-100 with synthetic symmetric and pair flipping noises.
	}\label{tabletm}  \vspace{-4mm}
	\centering
	\scriptsize
	%\scriptsize
	%\tabcolsep =0.7mm
	%\resizebox{0.49\textwidth}{10mm}{
		\begin{tabular}{l|cccc|cccc}
			\toprule
			\multirow{2}{*}{\textbf{Model} } &  \multicolumn{4}{c|}{\textbf{CIFAR-10}}   & \multicolumn{4}{c}{\textbf{CIFAR-100}} \\    \cline{2-9}
			&  Sym-20\% &  Sym-50\%  & Pair-20\% &  Pair-45\%   &  Sym-20\% &  Sym-50\%  & Pair-20\% &  Pair-45\%   \\ \cline{2-9}
			\hline
			%ERM  &   86.8 &79.4 & 62.9 & 42.7 & 83.2 & 62.0 & 46.7 & 19.9 \\
			Forward \cite{patrini2017making} &   85.20 $\pm$ 0.80 &  74.82 $\pm$ 0.78 & 88.21 $\pm$ 0.48 & 77.44 $\pm$ 6.89 & 54.90 $\pm$ 0.74 & 41.85 $\pm$ 0.71&  56.12 $\pm$ 0.54  &36.88 $\pm$ 2.32 \\
			T-Revision \cite{xia2019anchor} & 87.95 $\pm$ 0.36& 80.01 $\pm$ 0.62 & 90.33 $\pm$ 0.52 & 78.94 $\pm$ 2.58 & 62.72 $\pm$ 0.69 & 49.12 $\pm$ 0.22& 64.33 $\pm$ 0.49& 41.55 $\pm$ 0.95 \\
			Dual-T \cite{yao2020dual} & 88.35 $\pm$ 0.33 & 82.54 $\pm$ 0.19 & 89.77 $\pm$ 0.25&  76.53 $\pm$ 2.51&62.16 $\pm$ 0.58 &  52.49 $\pm$ 0.37&  67.21 $\pm$ 0.43 & 47.60 $\pm$ 0.43\\
			VolMinNet \cite{li2021provably}  & 89.58 $\pm$ 0.26 & 83.37 $\pm$ 0.25 & 90.37 $\pm$ 0.30 &  88.54 $\pm$ 0.21& 64.94 $\pm$ 0.40 &  53.89 $\pm$ 1.26 & 68.45 $\pm$ 0.69 & 58.90 $\pm$ 0.89 \\  \hline
			Ours & \textbf{90.72 $\pm$ 0.12} & \textbf{83.89 $\pm$ 0.18 }    &  \textbf{91.69 $\pm$ 0.21 }      &  \textbf{ 89.10 $\pm$ 0.23  } &   \textbf{ 69.62 $\pm$ 0.15}   &    \textbf{59.52 $\pm$ 0.25}    &  \textbf{ 74.30 $\pm$ 0.16} &     \textbf{62.07 $\pm$ 0.29}                \\			
			\bottomrule
		\end{tabular} \vspace{-4mm}
	\end{table*}

	\begin{table*}[ht]
		\setlength{\abovecaptionskip}{0.cm}
		\setlength{\belowcaptionskip}{-2cm}
		\caption{Classification results of transition matrix estimation benchmark on real-world Clothing1M dataset. The other results are copy from VolMinNet \cite{li2021provably}.} \vspace{-1mm} \label{tableclothing}
		\centering
		\scriptsize
		\begin{tabular}{c|c|c|c|c|c|c|c|c}
			\toprule
			ERM & GCE \cite{zhang2018generalized} &  Co-teaching \cite{han2018co} & MentorNet \cite{jiang2018mentornet} &  Forward \cite{patrini2017making}& T-Revision \cite{xia2019anchor}  & Dual-T \cite{yao2020dual} &  VolMinNet \cite{li2021provably} &  Ours \\   		\hline
			69.03 & 69.75&  60.15 & 56.79 & 69.91 & 70.97 & 71.49   & 72.42 & \textbf{72.80}  \\
			\bottomrule
		\end{tabular} \vspace{-4mm}
	\end{table*}

\begin{figure} \vspace{-1mm}
	\centering
	\subfigcapskip=-1mm
	\subfigure[Performance on CIFAR-10]{\label{figcifar10}
		\includegraphics[width=0.23\textwidth]{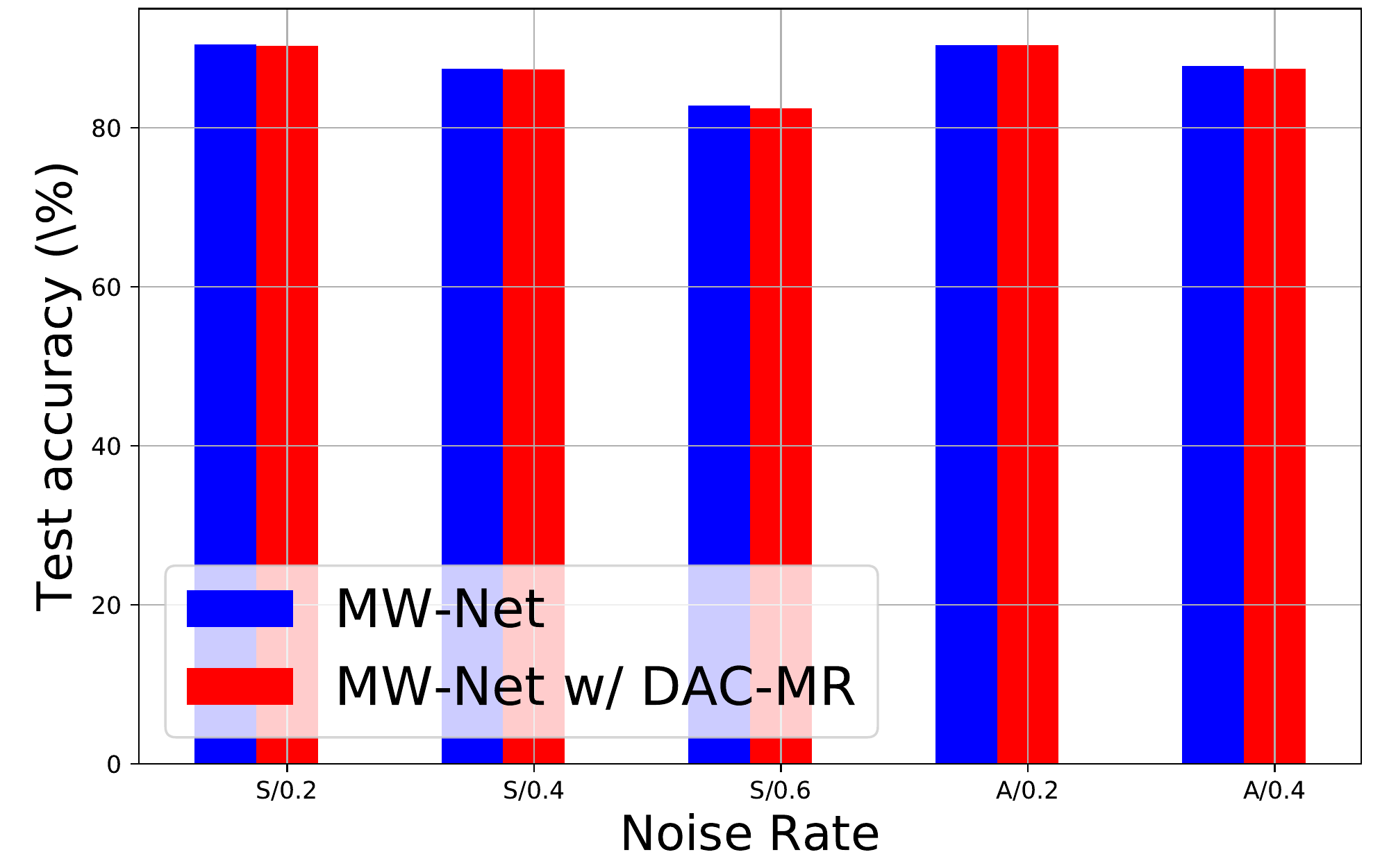} }
	\subfigure[Performance on CIFAR-100]{\label{figcifar100}
		\includegraphics[width=0.23\textwidth]{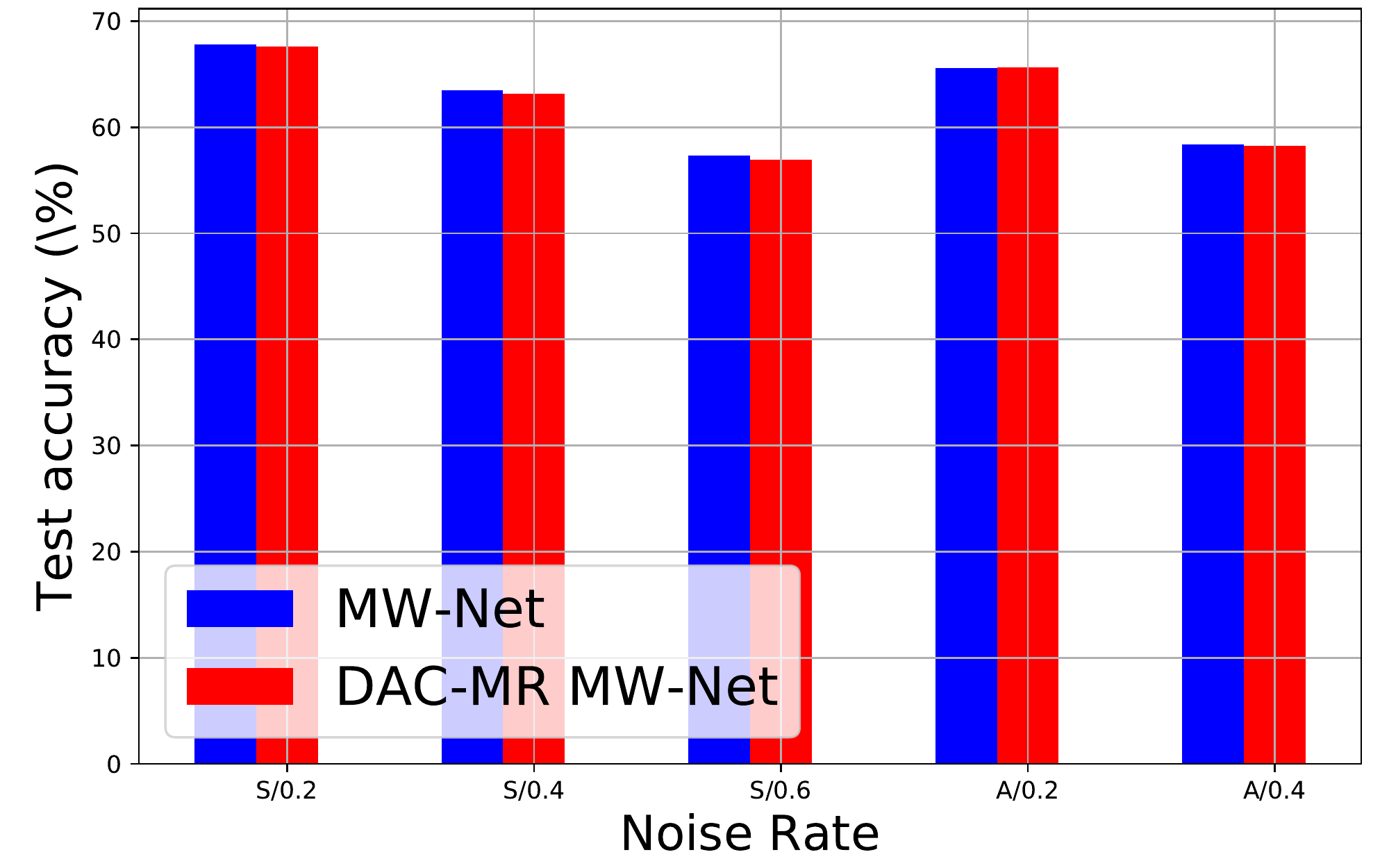} } \vspace{-6mm}
	\caption{Performance comparison for MW-Net with original clean meta samples and DAC-MR on synthetic (a) CIFAR-10, (b) CIFAR-100 noises.}\label{figcomparion}  \vspace{-4mm}
\end{figure}

\vspace{-4mm}
\subsection{Few-Shot Class-Incremental Learning} \label{section62}
\textbf{Formulation.} Different from FSL, few-shot class-incremental learning (FSCIL) learns training sessions in sequence. In this paper, we study the CEC \cite{zhang2021few} algorithm due to its SOTA FSCIL performance. CEC exploits the meta-objective computed on limited meta-data to optimize meta-model.
We introduce DAC-MR into CEC as a meta-regularizer to provide supplemental meta-knowledge information to help improve the performance of CEC.
More discussions can be found in Appendix D.2.
The implementation of our method is adapted from the official implementations of CEC available at \url{https://github.com/icoz69/CEC-CVPR2021}.

\noindent\textbf{Results.} As shown in Table \ref{tablefewcontinual}, our methods outperforms other methods on all three datasets among almost all the incremental
sessions. Our DAC-MR achieves higher average accuracy over all sessions and better final performance compared with CEC \cite{zhang2021few}.
This substantiates the effectiveness of DAC-MR, which further improves the capability of the meta-model to incrementally learn with less forgetting. It is also well-aligned with the theoretical result of Theorem \ref{th1}.

\vspace{-4mm}
\section{DAC-MR Benefits Label Noise Learning}\label{section7}
In this section, we study the influence of DAC-MR on meta-learning for two typical label noise learning tasks, including adaptive sample weighting strategy learning (\S\ref{section71}) and transition matrix estimation (\S\ref{section72}).
\vspace{-4mm}
\subsection{Sample Weighting Learning}   \label{section71}
\textbf{Formulation.} We consider the MW-Net algorithm \cite{shu2019meta}, representing a typical recent meta-learning strategy. Specifically, given a training dataset $\mathcal{D}^{(s)}$ with noisy labels, and a meta dataset $\mathcal{D}^{(q)}$ with clean labels, MW-Net $h_{\phi}(\cdot)$ learns an explicit sample weighting mapping by solving the following bi-level optimization objective:
\begin{align}
		{\phi}^* &= \mathop{\arg\min}_{{\phi }}  \mathcal{L}^{meta} (\mathcal{D}^{(q)};\mathbf{w}^*(\phi)),     \label{eqmeta}\\
		\text{s.t.} \ &\ \mathbf{w}^*(\phi) =\mathop{\arg\min}_{{\mathbf{w} }} \frac{1}{m} \sum_{i=1}^{m} h_{\phi}(L_i^{tr}(\mathbf{w})) L_i^{tr}(\mathbf{w}), \label{eqtrain}
\end{align}
where $L_i^{tr}(\mathbf{w}) \!=\! {\ell}(f_{\mathbf{w}}(x_i^{(s)}),y_i^{(s)})$, and $\ell$ is the cross-entropy loss. MW-Net often requires to collect extra meta dataset with clean labels, which is sometimes unavailable in practice. To reduce the barrier of real-life problem settings where clean samples are often unavailable, we introduce DAC-MR into MW-Net without the need to access clean meta samples. Specifically, we replace Eq.(\ref{eqmeta}) with DAC-MR as below: \vspace{2mm}
\begin{align} \label{eqmwnet}
{\phi}^* &= \mathop{\arg\min}_{{\phi }} \mathcal{MR}^{dac} (D;\mathbf{w}^*(\phi),A ),	  A \in \mathcal{A},
\end{align}
where $D = \{x_{i}\}_{i=1}^{k}$ are additionally sampled/divided from $\mathcal{D}^{(s)}$, so as to reduce the requirement of clean meta samples.
Here, we use the official implementations of MW-Net available at \url{https://github.com/xjtushujun/meta-weight-net}.

\noindent\textbf{Comparison with MW-Net.}  Fig. \ref{figcomparion} shows performances of MW-Net meta-learned with clean meta samples and DAC-MR, respectively, under different noise settings. One can see that DAC-MR behaves similar in almost all cases as clean meta samples supervision. This implies that DAC-MR can learn an adaptive weighting strategy without access to clean meta-data, well-aligned with the theoretical insights in \S\ref{section35}.

\noindent\textbf{Comparison with SOTA methods.} To fairly compare with the SOTA methods, we use pseudo-labels to correct noisy labels to more sufficiently make use of samples inspired by DivideMix \cite{li2019dividemix}, C2D \cite{zheltonozhskii2022contrast} and AugDesc \cite{nishi2021augmentation}. More discussions are given in the Appendix E.1.
Table \ref{tablemwnet} evaluates the performance of MW-Net with DAC-MR on CIFAR-10 and CIFAR-100 with synthetic noisy datasets. It is seen that our method consistently outperforms SOTA methods by an evident margin. Table \ref{webvision} compares our method with SOTA methods trained on the real mini-WebVision dataset and evaluated on both WebVision and ILSVRC12 validation sets. Our method outperforms previous works on WebVision validation set by at least 2\% top-1 accuracy, further validating the effectiveness of DAC-MR objective for learning proper weighting schemes on different noisy label cases.

\vspace{-4mm}
\subsection{Transition Matrix Estimation} \label{section72}
\textbf{Formulation.} The transition matrix plays a key role in building statistically consistent classifiers \cite{patrini2017making,shu2020meta} in label noise learning. Existing consistent estimators for the transition matrix have been developed by
exploiting anchor points \cite{patrini2017making}. However, the anchor-point assumption is not always satisfied in real scenarios, which tends to conduct a poorly estimated transition matrix and a degenerated classifier. To reduce the requirement of anchor points, we explore to use the DAC-MR for estimating transition matrix, i.e.,
\begin{align*}
	{\mathbf{T}}^* & = \mathop{\arg\min}_{ \mathbf{T} \in \mathbb{T}} \mathcal{MR}^{dac} (D;\mathbf{w}^*(\mathbf{T}),A ),	  A \in \mathcal{A},  \\
	\text{s.t.,} \ \mathbf{w}^*(\mathbf{T}) & = \mathop{\arg\min}_{\mathbf{w}} \frac{1}{m}\sum_{i=1}^m \ell(\mathbf{T}f_{\mathbf{w}}(x_i^{(s)}), y_i^{(s)} ),
\end{align*}
where $D = \{x_{i}\}_{i=1}^{k}$ are additionally sampled/divided from $\mathcal{D}^{(s)}$. Note that we treat transition matrix $\mathbf{T}$ as meta-representation, and its estimation should be obtained by minimizing DAC-MR on $D$ in a meta-learning manner \cite{hospedales2021meta}. 

\noindent\textbf{Results.} Table \ref{tabletm} shows the classification accuracies of our and baseline methods on synthetic noisy datasets. Note that T-Revision, Dual-T and VolMinNet are specifically designed based on the knowledge of transition matrix, while our DAC-MR can still outperform them on all noise settings, especially on CIFAR-100 by a significant margin.
Table \ref{tableclothing} shows the results on real noisy Clothing1M dataset. Forward, T-Revision and Dual-T additionally use 50k clean data to help estimate the transition matrix, which is actually not practical in real-world settings. Similar to VolMinNet \cite{li2021provably}, we only use noisy data for transition matrix estimation and model training, and achieve better performance than all baseline methods.
These results demonstrate that DAC-MR produces better transition matrix estimation, and perform superior over baseline methods in dealing with different label noise problems.
Especially, we have no access to clean meta-data, which further confirms the properness of theoretical analysis proposed in \S\ref{section35}.

\vspace{-4mm}
\section{Conclusion}
In this study we have suggested a MKIML framework to increase reliability and robustness of meta-learning for imperfect training tasks. The key insight is to integrate compensated meta-knowledge into the meta-learning process. As a preliminary attempt, we put forward the meta-regularization strategy. Compared to regularization used to improve generalization capability of the extracted model in conventional machine learning, the meta-regularization aims to help ameliorate generalization of the extracted meta-model for meta-learning.
We further use data augmentation consistency to encode data prediction invariance as meta-knowledge as a practical implementation of such MR objective, denoted by DAC-MR, which encourages models facilitated by meta-model to produce similar predictions under some data augmentations.
Two potential theoretical prospects of DAC-MR are illustrated and substantiated. One is that DAC-MR can be regarded as a proxy meta-supervised information to evaluate meta-model, implying a possible solution to handle meta-learning tasks with noisy or unavailable meta-data. The other is that additional DAC-MR objective can boost the meta-level generalization beyond purely meta-data-driven meta-objective. Comprehensive experimental results validate that DAC-MR can improve performance of baseline methods across various meta-learning tasks, network architectures and datasets, implying that DAC-MR is problem-agnostic and potentially useful to help strengthen general meta-learning problems and tasks.

In our future investigation, we will try to develop more meta-regularization strategies benefited from other types of meta-knowledge, like logic rules, knowledge graph, etc, to improve purely meta-data-driven approaches for more comprehensive and diverse meta learning problems. Attributed to its similar principle to conventional regularization strategies, it is also hopeful to develop deeper and more comprehensive statistical learning understanding for meta-regularization theoretically, like bias-variance tradeoff, Bayesian interpretations, etc. Besides, we will make endeavor to build the connection between meta-regularization and effective meta-hypothesis space.
More possible and valid paths to integrate meta-knowledge into other meta-learning components will also be further investigated.
Especially, developing novel approaches of MKIML to effectively learn the meta-model from relatively less training tasks, and well generalize to more complicated meta-test tasks will also be considered in our future research.

\bibliographystyle{IEEEtran}
\bibliography{IEEEabrv,mylib}

\clearpage
\onecolumn
\appendices

\section{Proof Details in Section 3} \label{sec1}

\subsection{Proof of Theorem 1}
We firstly present some necessary assumptions and lemma for our proof.
\begin{assumption}[Regularity of marginal distribution, \cite{yang2022sample}] \label{assumption11}
	Let $x \in P(x)$ be zero-mean $\mathbb{E}[x] = 0$, with the covairance $\mathbb{E}[xx^T] = \Sigma_x \succ 0$. Assume that $\Sigma_x^{-1/2} x$ is $\rho^2$-subgaussian, and there exist constants $C \geq c = \mathcal{O} (1)$, such that $c\mathbf{I}_d  \preccurlyeq \Sigma_x \preccurlyeq  C\mathbf{I}_d$.
\end{assumption}

\begin{assumption}[Sufficient labeled data, \cite{yang2022sample}]\label{assumption22}
	We assume that $MN \gg \rho^4 (d-d_{dac})$.
\end{assumption}

\begin{lemma} \label{lemma11}
	Let $x \in P(x)$ be zero-mean $\mathbb{E}[x] = 0$, with the covariance $\mathbb{E}[xx^T] = \Sigma_x \succ 0$, and $\Sigma_x^{-1/2} x$ is $\rho^2$-subgaussian.  Given an i.i.d. sample of $x$, $\mathbf{X} = [x_{11}, \cdots, x_{ij}, \cdots, x_{Mn}]$, for any $\delta \in (0,1)$, if $Mn\gg \rho^4 d$, then $0.9\Sigma_x \preccurlyeq \frac{1}{Mn} \mathbf{X}^T \mathbf{X} \preccurlyeq 1.1\Sigma_x$ with high probability.
\end{lemma}
The proof of Lemma \ref{lemma11} can refer to Lemma 5 in \cite{yang2022sample}.
Formally, we could define the following DAC-MR operator $\mathcal{T}^{mr}$ over $\mathcal{H}$:
\begin{Definition} [DAC-MR Operator]
	\begin{align*}
		&\mathcal{T}^{mr}_{\mathcal{A},\mathbf{X}}(\mathcal{H}) \triangleq  \left \{ h | h\in \mathcal{H}, f_{\theta^*_i(h_{\phi})}(x_{ij}) = f_{\theta^*_i(h_{\phi})}(A(x_{ij})), f_{\theta^*_i}  \right. \notag \\
		& \left. \in \mathcal{F}, A \in \mathcal{A}, \mathbf{X} = [x_{11}, \!\cdots\!, x_{ij}, \!\cdots\!, x_{Mn}],  \forall i \in [M], j \in [n] \right\}.
	\end{align*}
\end{Definition}
We use $d_{dac}$ to quantify the strength of data augmentation $\mathcal{A}$ defined by
\begin{align} \label{dac}
	d_{dac} \triangleq \mathrm{rank} (A\mathbf{X}  -\mathbf{X}  ), \text{for fixed} \ A \in \mathcal{A}.
\end{align}
Our theoretical results are built upon the results of \cite{tripuraneni2020theory}. Specifically,
we instantiate our framework for one of the most frequently used classification methods — logistic regression with $\mathcal{X} = \{ x\in \mathbb{R}^d | \|x\|_2 \leq G \}$, $\mathcal{Y} = \{0,1\}$. We follow the setting in \cite{tripuraneni2020theory}, and  consider the function class
\begin{align}
	\mathcal{F} = &\{ f|f(z)  = \alpha^T z, \alpha \in \mathbb{R}^r, \|\alpha\| \leq c \}, \notag \\
	\mathcal{H} = &\{h | h(x) = B^Tx, B = (b_1,\cdots,b_r ) \in \mathbb{R}^{d\times r}, \notag \\
	&B \  \text{is a matrix with orthonormal columns} \},  \label{eqfh}
\end{align}
where task-specific functions $f$s are linear maps, and the underlying meta-representation $h$ is a projection onto a low-dimensional subspace. Such meta-level representation learning would provide a statistical guarantee for several importantly meta-learning scenarios \cite{maurer2016benefit,tripuraneni2020theory,tripuraneni2021provable,du2020few,sun2021towards,xu2021representation}, e.g., transfering learning, few-shot learning.

We assume that $P(y=1|f\circ h(x)) = \sigma(\alpha^TA^Tx)$, where $\sigma(\cdot)$ is the sigmoid function with $\sigma(z) =  1 /(1+\exp(-z))$. We use the logistic loss $\ell(z,y) =-y\log(\sigma(z)) - (1-y) \log (1-\sigma(z))$ for $\mathcal{L}^{train}$ and $\mathcal{L}^{meta}$.
For the instantiation in Eq.(\ref{eqfh}), \cite{tripuraneni2020theory} recently has theoretically proved that meta-level error bound with respect to meta-model scales as $ C(\mathcal{H}) + M C(\mathcal{F})$, where $C(\cdot)$ captures the complexity of function class, and $M$ denotes some coefficients independent of model complexity.

To demonstrate that our additional DAC-MR objective brings better meta-level generalization than purely meta-data-driven meta-objective,
we just need to illustrate that the complexity of $\mathcal{T}^{mr}_{\mathcal{A},\mathbf{X}}(\mathcal{H})$ is smaller than $\mathcal{H}$.
Here we use the Gaussian complexity \cite{bartlett2002rademacher} to measure the complexity of a function class. The following theorem shows the complexities of $\mathcal{T}^{mr}_{\mathcal{A},\mathbf{X}}(\mathcal{H})$ and $\mathcal{H}$.
\begin{Theorem}[Formal restatement of Theorem 1 in Section 3.5] \label{th1}
	Considering the setting in Eq.(\ref{eqfh}), if Assumptions \ref{assumption11}, \ref{assumption22} and the conditions in Lemma \ref{lemma11} hold, then the complexity of $\mathcal{T}^{mr}_{\mathcal{A},\mathbf{X}}(\mathcal{H})$ and $\mathcal{H}$ satisfy
	\begin{align}
		C(\mathcal{T}^{mr}_{\mathcal{A},\mathbf{X}}(\mathcal{H})) \lesssim \sqrt{\frac{( d-d_{dac})r^2}{Mn}}, C(\mathcal{H})  \lesssim \sqrt{\frac{d r^2 }{Mn}},
	\end{align}
	where $d_{dac}$ is defined in Eq.(\ref{dac}).
\end{Theorem}
\begin{proof}
	We firstly calculate the Gaussian complexity of $\mathcal{H}$ as follows:
	\begin{align*}
		\hat{\mathcal{G}}_{\mathbf{X}}(\mathcal{H})
		&= \frac{1}{Mn} \mathbb{E}\left[\sup_{B} \sum_{i=1}^M \sum_{j=1}^n \sum_{k=1}^r g_{kij} b_k^T x_{ij}\right]
		= \frac{1}{Mn} \mathbb{E}\left[\sup_{(b_1,\cdots,b_r) \in \mathcal{H}} \sum_{k=1}^r b_k^T \left( \sum_{i=1}^M \sum_{j=1}^n  g_{kij}  x_{ij}\right) \right] \\
		& \leq  \frac{1}{Mn} \sum_{k=1}^r \mathbb{E} \left[ \left\| \sum_{i=1}^M \sum_{j=1}^n  g_{kij}  x_{ij}\right\| \right]
		\leq \frac{1}{Mn} \sum_{k=1}^r \sqrt{\mathbb{E} \left[ \left\| \sum_{i=1}^M \sum_{j=1}^n  g_{kij}  x_{ij}\right\|^2 \right]}  \\
		& \leq \frac{1}{Mn} \sum_{k=1}^r \sqrt{   \sum_{i=1}^M \sum_{j=1}^n  \left\|  x_{ij}\right\|^2 }
		= \frac{r}{\sqrt{Mn}} \sqrt{\mathrm{tr} \left (\frac{1}{Mn} \mathbf{X}\mathbf{X}^T \right)},
	\end{align*}
	where $\mathbf{X} = [x_{11}, \!\cdots\!, x_{ij}, \!\cdots\!, x_{Mn}]$, and the first and second inequalities hold by Jensen inequality; and the third inequality holds by Cauchy-Schwarz inequality. Furthermore, we take expectation over $\mathbf{X}$ to obtain the population Gaussian complexity as below:
	\begin{align*}
		\mathcal{G}_{Mn}(\mathcal{H}) \leq  \frac{r}{\sqrt{Mn}} \mathbb{E}\left[ \sqrt{\mathrm{tr} \left (\frac{1}{Mn} \mathbf{X}\mathbf{X}^T \right)} \right]  \lesssim  \sqrt{\frac{d r^2 }{Mn}},
	\end{align*}
	where the last inequality is obtained by Theorem 4 in \cite{tripuraneni2020theory}.
	
	We then calculate the Gaussian complexity of $\mathcal{T}^{mr}_{\mathcal{A},\mathcal{X}}(\mathcal{H})$. DAC-MR can be rewritten as the following constraints:
	\begin{align}
		\alpha^T B^T (\mathcal{A}(\mathbf{X}) - \mathbf{X} ) = 0.   \label{eqsp1}
	\end{align}
	Without loss of generality, Eq.(\ref{eqsp1}) can be written as
	\begin{align*}
		B^T (\mathcal{A}(\mathbf{X}) - \mathbf{X} ) = \mathbf{0}.
	\end{align*}
	i.e.,
	\begin{align*}
		(\mathcal{A}(\mathbf{X}) - \mathbf{X} ) b_k  = 0, k \in [r] .
	\end{align*}
	Let $\Delta = \mathcal{A}(\mathbf{X}) - \mathbf{X}$, and then we have
	\begin{align*}
		\mathcal{T}^{mr}_{\mathcal{A},\mathcal{X}}(\mathcal{H}) \triangleq & \left \{ h(x) = B^Tx | b_k \in \mathrm{Null}(\Delta),  k \in [r]   \right\}.
	\end{align*}
	Denote $\mathbf{P}_{\Delta}$ as the orthogonal projector onto $\Delta$, $\mathbf{P}^{\perp}_{\Delta} \triangleq \mathbf{I}_d - \mathbf{P}_{\Delta}$ as the orthogonal complement of $\mathbf{P}_{\Delta}$, and we are ready to bound the Gaussian complexity of $\mathcal{T}^{mr}_{\mathcal{A},\mathcal{X}}(\mathcal{H})$ as follows:
	\begin{align*}
		\hat{\mathcal{G}}_{\mathbf{X}}(\mathcal{T}^{mr}_{\mathcal{A},\mathcal{X}}(\mathcal{H}))
		= & \frac{1}{Mn} \mathbb{E}\left[\sup_{B \in \mathcal{T}^{mr}_{\mathcal{A},\mathcal{X}}(\mathcal{H})} \sum_{i=1}^M \sum_{j=1}^n \sum_{k=1}^r g_{kij} b_k^T x_{ij}\right]
		=  \frac{1}{Mn} \mathbb{E}\left[\sup_{(b_1,\cdots,b_r) \in \mathcal{H}} \sum_{k=1}^r b_k^T \left( \sum_{i=1}^M \sum_{j=1}^n  g_{kij} \mathbf{P}^{\perp}_{\Delta} x_{ij}\right) \right] \\
		\leq &  \frac{1}{Mn} \sum_{k=1}^r \mathbb{E} \left[ \left\| \sum_{i=1}^M \sum_{j=1}^n  g_{kij}\mathbf{P}^{\perp}_{\Delta}  x_{ij}\right\| \right]
		\leq  \frac{1}{Mn} \sum_{k=1}^r \sqrt{\mathbb{E} \left[ \left\| \sum_{i=1}^M \sum_{j=1}^n  g_{kij}\mathbf{P}^{\perp}_{\Delta}  x_{ij}\right\|^2 \right]}  \\
		\leq & \frac{1}{Mn} \sum_{k=1}^r \sqrt{   \sum_{i=1}^M \sum_{j=1}^n  \left\| \mathbf{P}^{\perp}_{\Delta} x_{ij}\right\|^2 }
		=  \frac{r}{\sqrt{Mn}} \sqrt{\mathrm{tr}  \left(  \frac{1}{Mn} \mathbf{P}^{\perp}_{\Delta} \mathbf{X}^T \mathbf{X} \mathbf{P}^{\perp}_{\Delta}  \right )}.
	\end{align*}
	Taking expectation over $\mathbf{X}$,
	%and conditioned on $\mathbf{P}^{\perp}_{\Delta}$,
	we can obtain that the population Gaussian complexity is satisfied with:
	\begin{align*}
		\mathcal{G}_{Mn}(\mathcal{T}^{mr}_{\mathcal{A},\mathcal{X}}(\mathcal{H}))
		\leq  \frac{r}{\sqrt{Mn}} \mathbb{E} \left[\sqrt{\mathrm{tr}  \left(  \frac{1}{Mn} \mathbf{P}^{\perp}_{\Delta} \mathbf{X}^T \mathbf{X} \mathbf{P}^{\perp}_{\Delta}  \right)}  \right]
		\leq  \frac{r}{\sqrt{Mn}} \sqrt{\mathbb{E} \left[\mathrm{tr}  \left(  \frac{1}{Mn} \mathbf{P}^{\perp}_{\Delta} \mathbf{X}^T \mathbf{X} \mathbf{P}^{\perp}_{\Delta}  \right )  \right]}
		\triangleq  \frac{r C_\mathcal{A}}{\sqrt{Mn}}.
	\end{align*}
	
	Under the Assumptions \ref{assumption11}, \ref{assumption22}, and Lemma \ref{lemma11}, we have
	\begin{align*}
		C_\mathcal{A}  \leq \sqrt{d-d_{dac}} \left\|\frac{1}{Mn} \mathbf{P}^{\perp}_{\Delta} \mathbf{X}^T \mathbf{X} \mathbf{P}^{\perp}_{\Delta} \right \|_2
		\leq \sqrt{d-d_{dac}} \left\|\frac{1}{Mn} \mathbf{X}^T \mathbf{X}\right \|_2 \leq 1.1 C \sqrt{d-d_{dac}}
		\lesssim \sqrt{d-d_{dac}},
	\end{align*}
	and thus we have
	\begin{align*}
		\mathcal{G}_{Mn}(\mathcal{T}^{mr}_{\mathcal{A},\mathcal{X}}(\mathcal{H})) \lesssim \sqrt{\frac{( d-d_{dac})r^2}{Mn}}.
	\end{align*}
\end{proof}

\subsection{Proof of Theorem 2 \& 3}\label{sec35}

Firstly, we introduce some additional basic notations and definitions.
For an arbitrary set $\mathcal{X}$, let $\mathcal{Y} = [K]$, and $g^*: \mathcal{X} \rightarrow [K]$ be the ground truth classifier that partition $\mathcal{X}$: for each $k\in [K]$, let $\mathcal{X}_k \triangleq \{ x\in \mathcal{X} | g^*(x) = k\}$, with $\mathcal{X}_i \cap \mathcal{X}_j = \emptyset, \forall i \neq j$. In addition, for an arbitrary classifier $g: \mathcal{X} \rightarrow [K]$, we denote the majority label with respect to $g$ for each class,
\begin{align*}
	\hat{y}_k \triangleq \mathop{\arg\max}_{y \in [K]} P(g(x)=y | x \in \mathcal{X}_k), \forall k \in [K],
\end{align*}
and the class-wise and global minority sets
\begin{align*}
	M_k \triangleq \{x \in \mathcal{X}_k | g(x)\neq \hat{y}_k\}, \forall k \in [K], \ M  \triangleq \bigcup_{k=1}^K  M_k.
\end{align*}

To capture the connectivity of the data distribution, we further introduce the expansion propety \cite{wei2020theoretical,cai2021theory} on the mixed distribution $S\cup Q$ as below:
\begin{Definition} [Constant Expansion]
	We say that the distribution $S\cup Q$ satisfies $(q,\xi)$-constant expansion for some constant $q,\xi \in (0,1)$, if for any $V \subset S\cup Q$ with $P(V) \geq q$ and $P(V\cap (S_k \cup Q_k)) \leq 1/2$ for any $k\in[K]$, we have $P((\mathcal{N}(V) \setminus V)\cap (S_k \cup Q_k)) \geq \min \{P(V\cap (S_k \cup Q_k)), \xi\}$.
\end{Definition}
\begin{Definition} [Multiplicative Expansion]
	We say that the distribution $S\cup Q$ satisfies $(a,c)$-multiplicative expansion for some constant $a \in (0,1), c>1$, if for any $k\in[K]$ and $V \subset S \cup Q$ with $P(V\cap (S_k \cup Q_k)) \leq a$, we have $P(\mathcal{N}(V) \cap (S_k \cup Q_k)) \geq \min \{c\cdot P(V \cap (S_k \cup Q_k)), 1\}$.
\end{Definition}
This expansion property lower bounds the neighborhood size of low probability sets, and the parameters $(q,\xi)$ or $(a,c)$ quantify the augmentation strength of $\mathcal{A}$. Specifically, the strength of expansion-based data augmentations is characterized by expansion capability of $\mathcal{A}$: for a neighborhood $V \subset \mathcal{X}$ of proper size (characterized by $q$ or $a$ under measure $P$), the stronger augmentation $A$ leads to more expansion in $\mathcal{N}(S)$, and therefore larger $\xi$ or $c$. The following proposition builds a bridge between two expansions.
\begin{Proposition}[Lemma C.6 in \cite{wei2020theoretical}]  \label{prop1}
	Suppose that the distribution $S\cup Q$ satisfies $(1/2,c)$-multiplicative expansion on $\mathcal{X}$. Then for any choice of $\xi > 0$, $S\cup Q$ satisfies $(\frac{\xi}{c-1}, \xi)$-constant expansion.
\end{Proposition}

To establish the relatinship between the expected meta-loss $\epsilon^{Q}(g) = \mathbb{E}_{Q} \mathbf{1}[g(x) \neq g^*(x)]$ and the expected DAC-MR $R_{Q}^{\mathcal{B}}(g) = \mathbb{E}_{Q} [\mathbf{1} (\exists x' \in \mathcal{B}(x), \text{s.t.}, g(x')\neq g(x))]$, we present some necessary assumption and lemmas as below.

\begin{Assumption} \label{assumption1}
	Assume that the task training and meta data distributions have the following structures: $supp(S) = \cup_{k=1}^K S_k$, $supp(Q) = \cup_{k=1}^K Q_k$, and $S_i \cap S_j = Q_i \cap Q_j = \emptyset, \forall i\neq j $. We further assume that the ground truth class $g^*(x)$ for $x \in S_k \cup Q_k$ is consistent, which is denoted as $y_k \in [K]$. Additionally,
	suppose that there exists a constant $\kappa \geq 1$, for all $i \in [K]$, such that
	\begin{align*} %\label{eqshift}
		P(Q_i \cap A)  \leq \kappa P\left(A \cap \frac{1}{2}\left(S_i\cup Q_i\right)\right), \forall A \subset \mathcal{S\cup Q}.
	\end{align*}
\end{Assumption}

\begin{Lemma} [Robustness on Sub-Populations with Constant Expansion]  \label{lemma1}
	Suppose that $Q$ satisfies $(q,\xi)$-constant expansion, and we divide $[K]$ into two partitions $S_1$ and $S_2$, where for every $i \in S_1$, $\mathbb{E}_{Q_i} \mathbf{1} [\exists x'\in \mathcal{N}(x), g(x) \neq g(x')] \leq \min\{q,\xi\}$,  and for every $i \in S_2$, $\mathbb{E}_{Q_i} \mathbf{1} [\exists x'\in \mathcal{N}(x), g(x) \neq g(x')] \geq \min\{q,\xi\}$. Under such partition, we then have
	\begin{align*}
		\sum_{k\in S_1} P(Q_k \cap Q) \geq 1 - \frac{R_{Q}^{\mathcal{B}}(g)}{\min \{q,\xi\}}.
	\end{align*}
\end{Lemma}
\begin{proof}
	Suppose $\sum_{k\in S_1} P(Q_k \cap Q) < 1 - \frac{R_{Q}^{\mathcal{B}}(g)}{\min \{q,\xi\}}$. Then we have $\sum_{k\in S_2} P(Q_k \cap Q) > \frac{R_{Q}^{\mathcal{B}}(g)}{\min \{q,\xi\}}$, which implies:
	\begin{align*}
		& \mathbb{E}_{Q} \mathbf{1} [\exists x'\in \mathcal{N}(x), g(x) \neq g(x')] \\
		& = \sum_{k \in [K]} \mathbb{E}_{Q_k} \mathbf{1} [\exists x'\in \mathcal{N}(x), g(x) \neq g(x')] P(Q_k \cap Q) \\
		& \geq \sum_{k \in S_2} \mathbb{E}_{Q_k} \mathbf{1} [\exists x'\in \mathcal{N}(x), g(x) \neq g(x')] P(Q_k \cap Q) \\
		& > \min \{q, \xi\} \sum_{k\in S_2} P(Q_k \cap Q)  \\
		& = R_{Q}^{\mathcal{B}}(g).
	\end{align*}
	This leads to  $R_{Q}^{\mathcal{B}}(g) < \mathbb{E}_{Q} \mathbf{1} [\exists x'\in \mathcal{N}(x), g(x) \neq g(x')]$, which induces a contradiction.
\end{proof}
\begin{Lemma} [Robustness on Sub-Populations with Multiplicative Expansion] \label{lemma2}
	Suppose that $Q$ satisfies $(1/2,c)$-multiplicative expansion, and we divide $[K]$ into two partitions $S_1$ and $S_2$, where for every $i \in S_1$, $\mathbb{E}_{Q_i} \mathbf{1} [\exists x'\in \mathcal{N}(x), g(x) \neq g(x')] \leq \min\{\frac{\xi}{c-1},\xi\}$,  and for every $i \in S_2$, $\mathbb{E}_{Q_i} \mathbf{1} [\exists x'\in \mathcal{N}(x), g(x) \neq g(x')] \geq \min\{\frac{\xi}{c-1},\xi\}$. Under such partition, we then have
	\begin{align*}
		\sum_{k\in S_1} P(Q_k \cap Q) \geq 1 - \frac{R_{Q}^{\mathcal{B}}(g)}{\min \{\frac{\xi}{c-1},\xi\}}.
	\end{align*}
\end{Lemma}
Based on Proposition \ref{prop1}, Lemma \ref{lemma2} can be obtained from Lemma \ref{lemma1} by plugging in $q = \frac{\xi}{c-1}$.

\begin{Lemma}  [Accuracy propagates on Sub-Populations with Constant Expansion]\label{lemma3}
	Supposed that $Q$ satisfies $(q,\xi)$-constant expansion, and the subpopulation $Q_k$ satisfies $R_{Q_k}^{\mathcal{B}}(g) < \min \{q, \xi\}$, and we then have
	\begin{align*}
		\mathbb{E}_{Q_k} \mathbf{1} [g(x) \neq \hat{y}_k] \geq 1-q.
	\end{align*}
\end{Lemma}
\begin{proof}
	We claim that $\mathbb{E}_{Q_k} \mathbf{1} [g(x) \neq \hat{y}_k] \geq 1-q$.
	Suppose $\mathbb{E}_{Q_k} \mathbf{1} [g(x) \neq \hat{y}_k] < 1-q$. Let $A = \{x| g(x) = \hat{y}_k\}$, and then we have $P(A \cap Q_k) > q $. By the $(q,\xi)$-constant expansion property, $P(\mathcal{N}(A) \setminus A \cap Q_k) <  \min \{q,\xi\}$. Observing that for $x$ in $\mathcal{N}(A) \setminus A$, we have $g(x) \neq \hat{y}_k$. Therefore, for $x$ in $\mathcal{N}(A) \setminus A$, there exists $x' \in \mathcal{N}(x)$, such that $\hat{y}_k = g(x') \neq g(x)$. Note that
	\begin{align*}
		R_{Q}^{\mathcal{B}}(g) & = \mathbb{E}_{Q} [\mathbf{1} (\exists x' \in \mathcal{B}(x), \text{s.t.}, g(x')\neq g(x))] \\
		& \geq \mathbb{E}_{Q} [\mathbf{1} (\exists x' \in \mathcal{B}(x), \text{s.t.}, g(x')\neq g(x))] \mathbf{1} [x \in \mathcal{N}(A) \setminus A ] \\
		& = P((\mathcal{N}(A) \setminus A ) \cap Q) > \min \{q,\xi\},
	\end{align*}
	which contradicts the condition that $R_{Q}^{\mathcal{B}}(g) \leq \min \{q,\xi\}$.
\end{proof}

Based on Proposition \ref{prop1}, we can easily obtain the following Lemma under the multiplicative expansion assumption.
\begin{Lemma}  [Accuracy propagation on Sub-Populations with Multiplicative Expansion]\label{lemma4}
	Supposed that $Q$ satisfies $(1/2,c)$-multiplicative expansion, and the subpopulation $Q_k$ satisfies $R_{Q_k}^{\mathcal{B}}(g) < \min \{\frac{\xi}{c-1}, \xi\}$, and then we have
	\begin{align*}
		\mathbb{E}_{Q_k} \mathbf{1} [g(x) \neq \hat{y}_k] \geq 1-\frac{\xi}{c-1}.
	\end{align*}
\end{Lemma}

\begin{Lemma} [Upper Bound of Minority Set \cite{cai2021theory}, Lemma A.1] \label{lemm5}
	Under the Assumption \ref{assumption1}, $P(M\cap\frac{1}{2}(S+Q))$ can be bounded as follows:
	
	(a) Under $(\frac{1}{2}, c)$-multiplicative expansion, we have $P(M\cap\frac{1}{2}(S+Q)) \leq \max \left(\frac{c+1}{c-1}, 3\right) R_{Q}^{\mathcal{B}}(g)$;
	
	(b) Under $(q, \xi)$--constant expansion, we have $P(M\cap\frac{1}{2}(S+Q)) \leq 2 \max (q, R_{Q}^{\mathcal{B}}(g)) + R_{Q}^{\mathcal{B}}(g)$.
\end{Lemma}

\begin{Theorem}[Bounding the Meta Error with Constant Expansion] \label{thm1}
	Suppose that Assumption \ref{assumption1} holds and $\frac{1}{2}(S+Q)$ satisfies $(q,\xi)$-constant expansion, and then we have
	\begin{align*}
		\begin{split}
			\epsilon^{Q}(g) \leq & \left\{ \frac{\kappa}{1-q}+ \frac{1}{\min \{q,\xi\}}\right\}  R_{Q}^{\mathcal{B}}(g) \\
			& + \frac{2\kappa}{1-q} \max (q, R_{Q}^{\mathcal{B}}(g)).
		\end{split}
	\end{align*}
\end{Theorem}

\begin{proof}
	Suppose that $Q$ satisfies $(q,\xi)$-constant expansion, and we divide $[K]$ into two partitions $S_1$ and $S_2$, where for every $i \in S_1$, $\mathbb{E}_{Q_i} \mathbf{1} [\exists x'\in \mathcal{N}(x), g(x) \neq g(x')] \leq \min\{q,\xi\}$,  and for every $i \in S_2$, $\mathbb{E}_{Q_i} \mathbf{1} [\exists x'\in \mathcal{N}(x), g(x) \neq g(x')] \geq \min\{q,\xi\}$. Then we have
	\begin{align*}
		\epsilon^{Q}(g) & =  \sum_{k=1}^K \epsilon_k^{Q}(g) P(Q_k \cap Q) \\
		& = \sum_{k\in[S_1]} \epsilon_k^{Q}(g) P(Q_k \cap Q) + \sum_{k\in[S_2]} \epsilon_k^{Q}(g) P(Q_k \cap Q) \\
		& \leq \sum_{k\in[S_1]} \epsilon_k^{Q}(g) P(Q_k \cap Q) +  \sum_{k\in[S_2]}  P(Q_k \cap Q).
	\end{align*}
	(1) For $k \in S_1$, we consider the following two cases:
	
	(a) If $y_k = \hat{y}_k$, we have
	\begin{align*}
		\epsilon_k^{Q}(g) = P(M_k \cap Q)  \leq \kappa P(M_k \cap \frac{1}{2}(S+Q)).
	\end{align*}
	
	(b) If $y_k \neq \hat{y}_k$, according to Lemma \ref{lemma3}, we have
	\begin{align*}
		\frac{P(M_k \cap Q)}{P(Q_k \cap Q)} = P(M_k \cap Q_k)  = \mathbb{E}_{Q_k} \mathbf{1} [g(x) \neq \hat{y}_k]  \geq 1-q.
	\end{align*}
	Then we have
	\begin{align*}
		\epsilon_k^{Q}(g) &\leq P(Q_k \cap Q) \leq \frac{P(M_k \cap Q)}{1-q} \\
		&\leq \frac{\kappa}{1-q} P(M_k \cap \frac{1}{2}(S+Q)).
	\end{align*}
	Combining the two cases (a) and (b), we always have
	\begin{align*}
		\epsilon_k^{Q}(g) \leq \frac{\kappa}{1-q} P(M_k \cap \frac{1}{2}(S+Q)),
	\end{align*}
	and thus
	\begin{align*}
		\sum_{k\in[S_1]} \epsilon_k^{Q}(g) P(Q_k \cap Q) & \leq \sum_{k\in[S_1]} \epsilon_k^{Q}(g) \\
		& \leq \sum_{k\in[S_1]} \frac{\kappa}{1-q} P(M_k \cap \frac{1}{2}(S+Q)) \\
		& \leq \frac{\kappa}{1-q} P(M \cap \frac{1}{2}(S+Q))  \\
		& =\frac{\kappa}{1-q} \left( 2 \max (q, R_{Q}^{\mathcal{B}}(g)) + R_{Q}^{\mathcal{B}}(g)\right),
	\end{align*}
	where the last equality holds based on Lemma \ref{lemm5}(b).
	
	(2) For $k \in S_2$, according to Lemma \ref{lemma2}, we have
	\begin{align*}
		\sum_{k\in[S_2]}  P(Q_k \cap Q)  \leq \frac{R_{Q}^{\mathcal{B}}(g)}{\min \{q,\xi\}}.
	\end{align*}
	
	Therefore,
	\begin{align*}
		\epsilon^{Q}(g) & \leq \frac{\kappa}{1-q} \left( 2 \max (q, R_{Q}^{\mathcal{B}}(g)) + R_{Q}^{\mathcal{B}}(g)\right)  + \frac{R_{Q}^{\mathcal{B}}(g)}{\min \{q,\xi\}}\\
		& =  \left\{ \frac{\kappa}{1-q}+ \frac{1}{\min \{q,\xi\}}\right\}  R_{Q}^{\mathcal{B}}(g) + \frac{2\kappa}{1-q} \max (q, R_{Q}^{\mathcal{B}}(g)).
	\end{align*}	
\end{proof}
We can easily obtain the following Theorem \ref{thm2} by plugging in  $q = \frac{\xi}{c-1}$ in Theorem \ref{thm1}.

\begin{Theorem}[Bounding the Meta Error with Multiplicative Expansion] \label{thm2}
	Suppose that Assumption \ref{assumption1} holds and $\frac{1}{2}(S+Q)$ satisfies $(q,\xi)$-multiplicative expansion, and then we have
	\begin{align*}
		\epsilon^{Q}(g) \leq \left\{ \max \left( \frac{\kappa (c+1)}{c-1-\xi}, 3(c-1)\right) + \frac{1}{\min (\frac{\xi}{c-1}, \xi)}\right\}  R_{Q}^{\mathcal{B}}(g).
	\end{align*}
\end{Theorem}

\section{More Experimental Setting details in Section 4}

\subsection{Inductive Few-Shot Learning}

\noindent\textbf{Datasets.} We adopt two typical few-shot image classification benchmarks. The \textit{miniImageNet} dataset \cite{vinyals2016matching} consists of 100 randomly chosen classes from ImageNet \cite{russakovsky2015imagenet}. The meta-training, meta-validation, and meta-testing sets contain 64, 16 and 20 classes randomly split from 100 classes, respectively. Each class contains 600 images of size $84\times 84$. We use the commonly-used split proposed by \cite{ravi2016optimization}.
The \textit{CIFAR-FS} dataset \cite{bertinetto2018meta} consists of all 100 classes from CIFAR-100. The classes are randomly split into 64, 16 and 20 for meta-training, meta-validation, and meta-testing, respectively. Each class contains 600 images of size $32\times 32$.

\noindent\textbf{Experimental setup.} We identically follow the practice in \cite{lee2019meta} for fair comparison. We uses a ResNet-12 backbone as meta-model to achieve better performance as suggested in \cite{lee2019meta}. We also follow the regularization tricks such as DropBlock \cite{ghiasi2018dropblock} to avoid the overfitting risk.
We use SGD with a Nesterov momentum 0.9 and a weight decay 0.0005. Each mini-batch consists of 8 episodes. The model was meta-trained for 60 epochs, with each epoch consisting of 1000 episodes. The learning rate was initially set to 0.1, and then changed to 0.006, 0.0012, and 0.00024 at epochs 20, 40 and 50, respectively. We adopt horizontal flip, random crop, and color (brightness, contrast, and saturation) jitter data augmentation techniques. We use 5-way classification in both meta-training and meta-test stages. Each class contains 6 query samples during meta-training and 15 test samples during meta-testing. Our meta-trained model was chosen based on 5-way 5-shot test accuracy on the meta-validation set. Meanwhile, we set training shot to 15 for miniImageNet and 5 for CIFAR-FS.
We set $\gamma=\lambda=1$ in Eq.(4) of the main paper.

\subsection{Cross-Domain Few-Shot Learning}

\noindent\textbf{Datasets.} We use the Meta-Dataset \cite{triantafillou2019meta} to evaluate our method, which is the standard benchmark for FSL. It contains images from 13 diverse datasets and we follow the standard protocol in \cite{triantafillou2019meta}, and see \cite{triantafillou2019meta} for details.

\noindent\textbf{Baselines.} We identically follow the baselines in \cite{triantafillou2019meta} for fair comparison, and see \cite{triantafillou2019meta} for details.

\noindent\textbf{Experimental setup.} We use ResNet-18 to train a multi-domain feature extractor over eight training subdatasets by following \cite{li2021universal} with the same hyperparameters in our experiments. We adopt horizontal flip,  random crop, and color (brightness, contrast and saturation) jitter data augmentation techniques. To finetune the feature extrator with DAC-MR, we use SGD with Nesterov momentum 0.9, a learning rate 0.001 and a weight decay 0.0005 on the support samples in meta-test stage. For learning task-specific weights, including the pre-classifier transformation and the adapter parameters, we directly attach them to the task-agnostic weights and learn them on the support samples in meta-test by using Adadelta optimizer following \cite{li2021universal} with the same hyperparameters.
We report the few-shot classification accuracy in previously seen domains and unseen domains along with their average accuracy. We also report average accuracy over all domains and the average rank.

\subsection{Transductive / Semi-Supervised Few-Shot Learning}

\noindent\textbf{Datasets.} We use four common few-shot claudication benchmark datasets, miniImageNet \cite{vinyals2016matching}, tieredImageNet \cite{ren2018meta}, CUB \cite{chen2018closer} and CIFAR-FS \cite{bertinetto2018meta}. More details see iLPC \cite{lazarou2021iterative}.

\noindent\textbf{Baselines.} We identically follow the baselines in \cite{lazarou2021iterative} for fair comparison, including LR+ICI \cite{wang2020instance}, PT+MAP \cite{hu2021leveraging} for semi-supervised FSL and EP \cite{rodriguez2020embedding}, SIB \cite{hu2019empirical}, LaplacianShot \cite{ziko2020laplacian}, PT+MAP \cite{hu2021leveraging} for transductive FSL.

\noindent\textbf{Experimental setup.} We use pre-trained weights of a WRN28-10 for transductive/semi-supervised FSL, which are the same to those used by \cite{hu2021leveraging,lazarou2021iterative,ziko2020laplacian,boudiaf2020information}. The experimental setting for producing pseudo-labels on query set is same as iLPC \cite{lazarou2021iterative}, and we fine-tune feature extractor using a SGD with momentum 0.9, weight decay 0.0005, learning rate 0.0001 for 10 epochs. We report mean accuracy and 95\% confidence interval on the 1000 5-way $K$-shot test tasks, $K \in \{1,5\}$. The query set contains 15 examples per class.

\section{More Experimental Setting details in Section 5}
\subsection{Unsupervised Domain Adaptation}

\noindent\textbf{Datasets.} We evaluate DAC-MR over two visual object recognition datasets: \textit{Office-Home} \cite{venkateswara2017deep} has 65 classes from four kinds of environment with large domain gap: Artistic (Ar), Clip Art (Cl), Product (Pr), and Real-World (Rw); \textit{VisDA-2017} \cite{peng2017visda} is a large-scale UDA dataset
with two domains named Synthetic and Real. The datasets consist of over 200k images from 12 categories of objects.

\noindent\textbf{Baselines.} We identically follow the baselines in \cite{liu2021cycle} for fair comparison, including DANN \cite{ganin2016domain}, CDAN \cite{long2018conditional}, VAT \cite{miyato2018virtual}, FixMatch \cite{sohn2020fixmatch}, MDD \cite{zhang2019bridging}, and SENTRY \cite{prabhu2021sentry}.

\noindent\textbf{Experimental setup.} We use ResNet-50 \cite{he2016deep} (pretrained on ImageNet \cite{russakovsky2015imagenet}) as feature extractors, and we also provide results of ResNet101 for VisDA-2017 to include more baselines. We adopt SGD with initial learning rate $2\times 10^{-3}$, and decay the learning rate exponentially until 30 epochs. We add the DAC-MR after the 10 epochs. We run all the tasks 3 times and report mean in top-1 accuracy. For VisDA-2017, we report the mean class accuracy. We also use sharpness-aware regularization \cite{foret2020sharpness} to enhance performance following CST \cite{liu2021cycle}. For the data augmentation techniques we adopt horizontal flip, random crop, and color (brightness, contrast, and saturation) jitter.

\subsection{Domain Generalization}
\noindent\textbf{Datasets.} We just follow benchmarks in ARM \cite{zhang2021adaptive} to evaluate our DAC-MR, including four image classification problems: Rotated MNIST, FEMNIST, CIFAR-10-C, Tiny ImageNet-C, and the WILDS benchmark \cite{koh2021wilds}.

\noindent\textbf{Baselines.} We identically follow baselines in \cite{zhang2021adaptive} for fair comparison, including BN adaptation \cite{schneider2020improving}, TTT \cite{sun2020test}, UW \cite{sagawa2019distributionally}, DRNN \cite{sagawa2019distributionally}, DANN \cite{ganin2016domain}, MDD \cite{zhang2019bridging}. We also compare CORAL \cite{sun2016return} and IRM \cite{arjovsky2019invariant} for WILDS benchmark.

\noindent\textbf{Experimental setup.} We set $\lambda=1$ for all experiments, and follow the exact experimental settings as ARM \cite{zhang2021adaptive}. We add the DAC-MR after the 20 epochs.For the Ratated MNIST dataset we adopt random rotation data augmentation techniques. For FEMNIST dataset we adopt random crop and blur data augmentation techniques. For the CIFAR-10-C dataset and Tiny ImageNet-C dataset we adopt horizontal flip, random crop, and color (brightness, contrast, and saturation) jitter data augmentation techniques.

\subsection{Transfer Learning with Fine-tuning}

\noindent\textbf{Datasets.} We consider several extensively investigated transfer learning benchmarks, consisting of CUB-200 (11, 788 images for 200 bird species) \cite{welinder2010caltech}, Stanford Cars (16, 185 images for 196 car categories) \cite{krause20133d}, and FGVC Aircraft (10, 000 images for 100 aircraft variants) \cite{maji2013fine}.

\noindent\textbf{Baselines.} We compared against several state-of-the-art fine-tuning methods: Fine-tuning \cite{yosinski2014transferable}, $L^2$-SP \cite{xuhong2018explicit}, DELTA \cite{li2018delta} and Co-Tuning \cite{you2020co}. The implementation of this paper is adapted from the transfer learning library \cite{jiang2022transferability,tllib}.

\noindent\textbf{Experimental setup.} We use ResNet-50 \cite{he2016deep} (supervised pretrained or self-supervised MoCo \cite{he2020momentum} pre-trained on ImageNet \cite{russakovsky2015imagenet}) as the source model. We optimize all models by SGD with a momentum 0.9, and learning rate for task-specific classifier is ten times of the learning rate for pre-trained parameters, following the common fine-tuning practice \cite{yosinski2014transferable}. We set batch size as 48. To explore the impact of negative transfer with different numbers of training examples, we create four configurations for each dataset, which respectively have 15\%, 30\%, 50\%, and 100\% randomly sampled training examples for each category. Each experiment is repeated three times with different random seeds to collect mean and standard deviation of the performance.
We finetune the feature shared representation function by SGD with a momentum 0.9, and the learning rate is the same as the learning rate for pre-trained parameter. And for the target data we dynamicly use the data whose maximum predicted probability is higher than 0.99 invovling in computing DAC-MR. we adopt horizontal flip, random crop, and color (brightness, contrast, and saturation) jitter for the data augmentations.

\section{More Experimental Setting details in Section 6}
\subsection{Task-Incremental Learning}
\noindent\textbf{DAC-MR amelioration manner for the task.} 
Following the La-MAML \cite{gupta2020look} method, we consider a setting where a sequence of $T$ tasks $[\tau_1, \tau_2, \cdots, \tau_T]$ is learned by observing their training data $[\mathcal{D}_1, \mathcal{D}_2, \cdots, \mathcal{D}_T]$ sequentially. We define $(X^i,Y^i) = \{(x_n^i,y_n^i)\}_{n=0}^{N_i}$ as the set of $N_i$
input-label pairs randomly drawn from $\mathcal{D}_i$. For any time-step $j$ during online learning, we aim to minimize the empirical risk of the model on all the $t$ tasks seen so far $(\tau_{1:t})$, given limited access to data $(X^i,Y^i)$ from previous tasks $\tau_i (i<t)$.
%La-MAML \cite{gupta2020look} consists of an inner level and an outer level of optimization.
%In the inner level, we start at some meta-parameters $\theta_0$ and try to minimize the task-specific loss function $\ell_t$ via $k$-steps of SGD, arriving at $\theta_k$.
The learning objective of La-MAML \cite{gupta2020look} is defined as:
\begin{align} \label{eqcon}
	\min_{\theta_0^j, \alpha^j}   \sum_{\mathcal{S}_k^j \sim \mathcal{D}_t} \left[L_t (U_k(\alpha_j, \theta_0^j, \mathcal{S}_k^j)) \right],
\end{align}
where meta loss $L_t \!=\! \sum_{i=1}^t\! \ell_i$ backpropagates gradients with respect to the weights $\theta_0^j$ and learning rate $\alpha^j$,
which is evaluated on $\theta_k^j = U_k(\alpha_j, \theta_0^j )$.
$U_k(\alpha_j, \theta_0^j )$ denotes $k$ steps of gradient descent with learning rate $\alpha_j$ on the inner level task-specific loss function. The meta loss is computed on the samples from replay-buffer indicating the performance of parameters $\theta_k^j$ on all the tasks $\tau_{1:t}$ seen till time $j$. Details please see La-MAML \cite{gupta2020look}.
We additionally introduce DAC-MR into Eq.(\ref{eqcon}) as supplemental meta-knowledge to produce better continual learning performance of La-MAML \cite{gupta2020look}, i.e.,
\begin{align*}
	\min_{\theta_0^j, \alpha^j}   \sum_{\mathcal{S}_k^j \sim \mathcal{D}_t} \left[L_t (U_k(\alpha_j, \theta_0^j, \mathcal{S}_k^j))  + \mathcal{MR}^{dac} (\mathcal{S}_k^j;U_k(\alpha_j, \theta_0^j),A )\right],	  A \in \mathcal{A}.
\end{align*}

\noindent\textbf{Baselines.} We identically follow the baselines in La-MAML \cite{gupta2020look} for fair comparision, including MER \cite{riemer2018learning}, iCaRL \cite{rebuffi2017icarl}, GEM \cite{lopez2017gradient}, AGEM \cite{chaudhry2018efficient}. We also compare C-MAML (base algorithm of LA-MAML) and SYNC (without meta-updating learning rate in LA-MAML).

\noindent\textbf{Datasets.} We conduct experiments on the CIFAR-100 dataset \cite{krizhevsky2009learning} in a task-incremental manner where 20 tasks comprising of disjoint 5-way classification problems are streamed. We also evaluate on the TinyImagenet-200 dataset by partitioning its 200 classes into 40 5-way classification tasks.

\noindent\textbf{Experimental setup.} Following La-MAML \cite{gupta2020look}, we conduct experiments in both the Single-Pass and Multiple-Pass settings. Each method is allowed a replay-buffer, containing upto 200 and 400 samples for CIFAR-100 and TinyImagenet respectively. We report the retained accuracy (RA) metric and backward-transfer and interference (BTI) value, which computes the average accuracy of the model across tasks at the end of training and the average change in accuracy of each task from when it was learnt to the end of the last task. We adopt horizontal flip, random crop, and color (brightness, contrast, and saturation) jitter for the data augmentations.

\subsection{Few-Shot Class-Incremental Learning}
\noindent\textbf{DAC-MR amelioration manner for the task.} Different from FSL, few-shot class-incremental learning (FSCIL) learns training sessions in sequence. Let $\{\mathcal{D}_0, \mathcal{D}_1, \cdots, \mathcal{D}_T\}$ denote the training sets of different training sessions, and the corresponding label space of $\mathcal{D}_i$ is denoted by $\mathcal{Y}_i$. Different sessions have no overlapped classes, i.e., $\mathcal{Y}_i \cap \mathcal{Y}_j = \emptyset, \forall i \neq j$. FSCIL aims to develop an algorithm  that can sequentially train a model from all tasks to possibly avoid the catastrophic forgetting issue, i.e., when the model is trained on the $i$-th task it should still provide possibly accurate predictions for all tasks $j < i$ seen in the past. In this paper, we study the CEC \cite{zhang2021few} algorithm due to its SOTA FSCIL performance. It mainly contains three learning stages: feature pre-training, pseudo incremental learning and classifier learning. To ensure the classifier learning incorporates the global context information of all individual tasks in previous sessions, CEC \cite{zhang2021few} proposes a continually evolved classifier as shown in Algorithm \ref{alg} which includes a classifier adaptation module $\mathcal{G}_{\theta}$ to update the classifier weights learned on each individual session based on the global context of previous sessions. We introduce DAC-MR into CEC as a meta-regularizer to provide supplemental meta-knowledge information to help improve the performance of CEC. The modification of CEC Algorithm is shown in red.

\begin{algorithm}[t]
	\vspace{0mm}
	\renewcommand{\algorithmicrequire}{\textbf{Input:}}
	\renewcommand{\algorithmicensure}{\textbf{Output:}}
	\caption{Pseudo incremental learning \cite{zhang2021few} incorporated with DAC-MR. }
	\label{alg}
	\begin{algorithmic}[1]  \small
		\REQUIRE  Base classes datasets $\mathcal{D}_0$, pre-trained model $\mathcal{R}$, a randomly initialized GAT model $\mathcal{G}_{\theta}$.
		\ENSURE  A trained GAT model $\mathcal{G}_{\theta}$.
		\WHILE{not done}
		\STATE $\{\mathcal{S}_b, \mathcal{Q}_b\} \leftarrow$ Sample the the support and query set for pseudo base classes from $\mathcal{D}_0$.
		\STATE $\mathbf{W}_b \leftarrow$ Learn FC layer upon $\mathcal{R}$ with $\mathcal{S}_b$.
		\STATE $\{\mathcal{S}_i, \mathcal{Q}_i \}\leftarrow$ Sample the the support and query set for pseudo incremental classes from $\mathcal{D}_0$.
		\FOR{class $c$ {\bfseries in} $C_i$}
		\STATE $\gamma \leftarrow$ Randomly select angle in $\{90^{\circ}, 180^{\circ}, 270^{\circ}\}$;
		\STATE $\{\mathcal{S}'_i, \mathcal{Q}'_i\} \leftarrow$ Rotate $\{\mathcal{S}_i, \mathcal{Q}_i\}$ from class $c$ with the selected angle $\gamma$;
		\ENDFOR
		\STATE $\mathbf{W}_i \leftarrow$ Learn FC layer upon $\mathcal{R}$ with pseudo incremental support set $\mathcal{S}'_i$ after rotation.
		\STATE $\mathbf{W}'_b, \mathbf{W}'_i \leftarrow$ Update classifier $\mathbf{W}_b, \mathbf{W}_i $ using $\mathcal{G}_{\theta}$.
		\STATE $\hat{y}^{(q)}\leftarrow$ Predict labels of $\{\mathcal{Q}_b, \mathcal{Q}'_i\}$ using $[\mathcal{R},(\mathbf{W}'_b, \mathbf{W}'_i)]$.
		\STATE loss $\leftarrow$ Compute loss with $\mathcal{L}(y^{(q)},\hat{y}^{(q)})$.
		\STATE {\color{red}MR $\leftarrow$ Compute meta-regularization $\mathcal{MR}^{dac} (\mathcal{Q}_b ;\mathbf{W}'_b \circ \mathcal{R},	  A \in \mathcal{A}) + \mathcal{MR}^{dac} (\mathcal{Q}'_i;\mathbf{W}'_i \circ \mathcal{R},	  A \in \mathcal{A})$.}
		\STATE {\color{red}loss\_all $\leftarrow$ loss $+$ MR.}
		\STATE Optimize $\mathcal{G}_{\theta}$ with SGD
		\ENDWHILE
		%\UNTIL{$noChange$ is $true$}
	\end{algorithmic}
	Note: $C_i$ is the number of classes in pseudo incremental classes; $y^{(q)}$ and $\hat{y}^{(q)}$ indicate the ground truth label and the network prediction, respectively. $\mathcal{L}$ and $\mathcal{MR}^{dac}$ are the cross-entropy loss and DAC-MR, respectively.
\end{algorithm}

\noindent\textbf{Baselines.} We compared with SOTA methods including iCaRL \cite{rebuffi2017icarl}, TOPIC \cite{tao2020few}, SPPR \cite{zhu2021self} and CEC \cite{zhang2021few}. We evaluate the model after each session with test set and report the Top-1 accuracy and the average of all sessions. We also
include the relative improvement for the final session.
%We also define a performance dropping rate (PD) that measures the absolute accuracy drops in the last session w.r.t. the accuracy in the first session.

\noindent\textbf{Datasets.} We evaluate DAC-MR upon CEC on three popular few-shot incremental learning benchmark datasets, including CIFAR100 \cite{krizhevsky2009learning}, miniImageNet \cite{vinyals2016matching} and CUB-200 \cite{welinder2010caltech}. We follow the experimental setting in \cite{tao2020few}.
%For CIFAR100 \cite{krizhevsky2009learning}, we follow the splits in \cite{tao2020few}, where 60 classes and 40 classes are used as base classes and new classes, respectively. The 40 new classes are further divided into 8 new incremental sessions, and each new session is a 5-way 5-shot classification task.
%For miniImageNet \cite{vinyals2016matching}, we follow \cite{tao2020few} to split the 100 classes into 60 base classes and 40 incremental classes. The 40 new classes are further divided equally into 8 sessions with 5 classes in each session, and each class has 5 training images in the incremental sessions. For CUB-200 \cite{welinder2010caltech},  that 200 classes are divided into 100 base classes and 100 new classes, respectively. The 100 new classes are further divided into 10 new sessions where each session is a 10-way 5-shot task.

\noindent\textbf{Experimental setup.} Following \cite{tao2020few}, we employ ResNet20 as the backbone for experiments on CIFAR100 and ResNet18 for experiments on miniImageNet and CUB200. We train the GAT model $\mathcal{G}_{\theta}$ for 5000 iterations with the learning rate of 0.0002, and decay it by 0.5 every 1000 iteration.
Random crop, random scale, and random horizontal flip are used for data augmentation at training time. We add the DAC-MR after the 20 epochs.

\section{More Experimental Setting details in Section 7}

\subsection{Sample Weighting Learning}

\subsubsection{Comparison with MW-Net}
This subsection compares against MW-Net meta-learned with clean meta dataset to show our novel meta-objective is capable of learning proper weighting strategy.

\noindent\textbf{Datasets.} We use CIFAR-10 and CIFAR-100 \cite{krizhevsky2009learning} for comparing robust learning methods. We identically follow MW-Net \cite{shu2019meta}, and apply the symmetric and asymmetric noise models. We randomly select different 1000 images at each epoch from training set as $D$ to compute DAC-MR.

\noindent\textbf{Experimental setup.} We use ResNet-32 as the classifier network, and the weighting network is a single layer MLP with 100 hidden nodes and ReLU activations.
All classifier networks were trained using SGD with a momentum 0.9, a weight decay $5\times 10^{-4}$ and an initial learning rate 0.1. The learning rate is divided by 10 after 80 and 100 epochs (for a total 120 epochs). We use Adam optimizer to train MW-Net with learning rate 0.001. We use a batch size of 100 for both the training samples and the meta ones. We repeat experiments with three different seeds for corrupting samples with label noise and initializing the classifier networks.
We adopt horizontal flip, random crop, and color (brightness, contrast, and saturation) jitter for the data augmentations. 

\subsubsection{Comparison with SOTA methods.}
This subsection compares with SOTA robust learning methods against both synthetic and real-world noisy labels. To fairly compare with the these SOTA methods, we use pseudo-labels to correct noisy labels to more sufficiently make use of samples inspired by DivideMix \cite{li2019dividemix}, C2D \cite{zheltonozhskii2022contrast} and AugDesc \cite{nishi2021augmentation}.

Specifically, we use the following novel bi-level optimization objective to learn MW-Net:
\begin{align}
	{\phi}^* &= \mathop{\arg\min}_{{\phi }}  \mathcal{L}^{meta} (\mathcal{D}^{(q)};\mathbf{w}^*(\phi)),     \label{eqmeta}\\
	\text{s.t.} \ &\ \mathbf{w}^*(\phi) =\mathop{\arg\min}_{{\mathbf{w} }} \frac{1}{m}\sum_{i=1}^{m} \left[ h_{\phi}(L_i^{tr}(\mathbf{w})) \ell(f_{\mathbf{w}}(x_i^{(s)}),y_i^{(s)})+(1-h_{\phi}(L_i^{tr}(\mathbf{w}))) \ell(f_{\mathbf{w}}(x_i^{(s)}),z_i^{(s)}) \right], \label{eqtrain}
\end{align}
where $L_i^{tr}(\mathbf{w}) \!=\! {\ell}(f_{\mathbf{w}}(x_i^{(s)}),y_i^{(s)})$, $\ell$ is the cross-entropy loss, and $z$ is pseudo-label. In our experiments, we apply EMA \cite{tarvainen2017mean} and temporal ensembling \cite{laine2016temporal} techniques to produce pseudo-labels in our algorithm, which has been verified to be effective in tasks like semi-supervised learning \cite{laine2016temporal,laine2016temporal} and robust learning \cite{arazo2019unsupervised,liu2020early}. And we add a warm-up self-supervised pre-training step and impose a data augmentation based consistency regularization as in C2D and AugDesc to boost our method. The complete algorithm is summarized in the Algorithm \ref{alg:example1}.

\begin{algorithm}
	\vspace{0mm}
	\renewcommand{\algorithmicrequire}{\textbf{Input:}}
	\renewcommand{\algorithmicensure}{\textbf{Output:}}
	\caption{Label correction with MW-Net algorithm}
	\label{alg:example1}
	\begin{algorithmic}[1]  \small
		% \STATE {\bfseries Input:} Training data $\mathcal{D}$, meta set $\mathcal{D}^{(m)}$, batch size $n,n^{(m)}$, max iterations $T$.
		%   \STATE {\bfseries Output:} Student model parameters $w^{(T)}$
		\REQUIRE  Training data $\mathcal{D}^{(s)}$, batch size $n$, temporal ensembling momentum $\alpha \in [0,1)$, weight averaging momentum $\beta\in [0,1)$, mixup hyperparameter $\gamma >0$, learning rates $\eta_1,\eta_2$, data augmentations $A$.
		\ENSURE  Classifier parameter $\mathbf{w}^{*}$%, CMW-Net parameter $\Theta^{(T)}$
		%\REPEAT
		%\STATE Solve Eq.(8) by k-means to obtain $S$, then determine the structure of the CMW-Net $\mathcal{V}(L^{tr}(\mathbf{w}),S_{y};\Theta)$, and revise the training dataset as $\tilde{\mathcal{D}}^{tr} = \{x_i,y_i,S_{y_i}\}_{i=1}^N$.
		\STATE Initialize classifier network parameter $\mathbf{w}^{(0)}$ with a self-supervised pre-training model. Initialize averaged predictions with noisy labels $\mathbf{z}^{(0)} = \mathbf{\hat{y}}_{[N\times C]}$, and averaged weights (untrainable) $\mathbf{w}_{WA}^{(0)} =\mathbf{w}^{(0)}$.
		\FOR{$t=0$ {\bfseries to} $T-1$}
		\STATE $\{x,y\} \leftarrow$ SampleMiniBatch($\mathcal{D}^{(s)},n$).
		\STATE $\{x^{meta}\} \leftarrow$ SampleMiniBatch($\mathcal{D}^{(s)},m$).
		\STATE 	Generate mixing coefficient $\lambda \sim Beta(\gamma,\gamma), \lambda=\max(\lambda,1-\lambda)$.
		\STATE Calculate weight averaging: $\mathbf{w}_{WA}^{(t+1)} = \beta \mathbf{w}_{WA}^{(t)} + (1-\beta) \mathbf{w}^{(t)}$.
		\STATE Calculate temporal ensembling: $\mathbf{z}^{(t+1)} = \alpha \mathbf{z}_i^{(t)} + (1-\alpha) f(x; \mathbf{w}_{WA}^{(t+1)})$.
		\STATE Generate new index sequence $\text{idx = torch.randperm(n)}$.
		\STATE Generate $\tilde{x} = \lambda' x +(1-\lambda')x[\text{idx}]$, and let $\tilde{y} = y[\text{idx}], \tilde{z}^{(t+1)} = {z}^{(t+1)}[\text{idx}]$. Calculate $N_i$ and $\tilde{N}_i$, representing the numbers of samples contained in the classes to which $x_i$ and $x[\text{idx}]_i$ belong, respectively.
		\STATE Formulate the learning manner of classifier network:
		\begin{align*}
			&\hat{\mathbf{w}}^{(t+1)}(\phi)=
			\mathbf{w}^{(t)} \\& - \eta_1
			\sum_{i=1}^n  \{ \lambda \left[ h({\ell}(f(\tilde{x}_i; \mathbf{w}_{WA}^{(t)} ),y_i); \phi^{(t)}) \nabla_{\mathbf{w}} {\ell}(f(\tilde{x}_i;\mathbf{w}^{(t)}),y_i)\Big|_{\mathbf{w}^{(t)}}   +(1- h({\ell}(f(\tilde{x}_i;\mathbf{w}_{WA}^{(t)}),y_i) ; \phi^{(t)}) \nabla_{\mathbf{w}} {\ell}(f(\tilde{x}_i;\mathbf{w}^{(t)}),\mathbf{z}_i^{(t+1)})\Big|_{\mathbf{w}^{(t)}} \right] \\ &+(1-\lambda) \left[ h({\ell}(f(\tilde{x}_i;\mathbf{w}_{WA}^{(t)}),\tilde{y}_i); \phi^{(t)}) \nabla_{\mathbf{w}} {\ell}(f(\tilde{x}_i;\mathbf{w}^{(t)}),\tilde{y}_i)\Big|_{\mathbf{w}^{(t)}}   +(1-h({\ell}(f(\tilde{x}_i;\mathbf{w}_{WA}^{(t)}),\tilde{y}_i)); \phi^{(t)}) \nabla_{\mathbf{w}} {\ell}(f(\tilde{x}_i;\mathbf{w}^{(t)}),\tilde{\mathbf{z}}_i^{(t+1)})\Big|_{\mathbf{w}^{(t)}} \right]  \}.
		\end{align*}
		\STATE Update parameters of MW-Net $\phi^{(t+1)}$ by
		\begin{align*}
			{\phi}^{(t+1)} =  {\phi}^{(t)} -\eta_2 \frac{1}{m}\sum_{i=1}^{m} \nabla_{ \phi} \mathcal{MR}^{dac} (\{x^{meta}\};\hat{\mathbf{w}}^{(t+1)}(\phi),A ) \Big|_{\phi^{(t)}}.
		\end{align*}
		\STATE Update parameters of classifier $\mathbf{w}^{(t+1)}$ by
		\begin{align*}
			&{\mathbf{w}}^{(t+1)}=
			\mathbf{w}^{(t)} \\& - \eta_1
			\sum_{i=1}^n  \{ \lambda \left[ h({\ell}(f(\tilde{x}_i; \mathbf{w}_{WA}^{(t)} ),y_i); \phi^{(t+1)}) \nabla_{\mathbf{w}} {\ell}(f(\tilde{x}_i;\mathbf{w}^{(t)}),y_i)\Big|_{\mathbf{w}^{(t)}}   +(1- h({\ell}(f(\tilde{x}_i;\mathbf{w}_{WA}^{(t)}),y_i) ; \phi^{(t+1)}) \nabla_{\mathbf{w}} {\ell}(f(\tilde{x}_i;\mathbf{w}^{(t)}),\mathbf{z}_i^{(t+1)})\Big|_{\mathbf{w}^{(t)}} \right] \\ &+(1-\lambda) \left[ h({\ell}(f(\tilde{x}_i;\mathbf{w}_{WA}^{(t)}),\tilde{y}_i); \phi^{(t+1)}) \nabla_{\mathbf{w}} {\ell}(f(\tilde{x}_i;\mathbf{w}^{(t)}),\tilde{y}_i)\Big|_{\mathbf{w}^{(t)}}   +(1-h({\ell}(f(\tilde{x}_i;\mathbf{w}_{WA}^{(t)}),\tilde{y}_i)); \phi^{(t+1)}) \nabla_{\mathbf{w}} {\ell}(f(\tilde{x}_i;\mathbf{w}^{(t)}),\tilde{\mathbf{z}}_i^{(t+1)})\Big|_{\mathbf{w}^{(t)}} \right] \}.
		\end{align*}
		\ENDFOR
		%\UNTIL{$noChange$ is $true$}
	\end{algorithmic}
\end{algorithm}

\noindent\textbf{Datasets.} For the real-world noisy labels, we employ mini-WebVision dataset, which contains the top 50 classes from the Google image subset of WebVision \cite{li2017webvision}. And We randomly select different 10 images every class at each epochs from training set as $D$ to compute DAC-MR.

\noindent\textbf{Baselines.} The comparison methods include: 1) ERM, 2) Forward \cite{patrini2017making}, 3) M-correction \cite{arazo2019unsupervised}, 4) PENCIL \cite{yi2019probabilistic}, 5) DivideMix \cite{li2019dividemix}, 6) ELR+ \cite{liu2020early}, 7) AugDesc \cite{nishi2021augmentation}, and 8) C2D \cite{zheltonozhskii2022contrast}.

\noindent\textbf{Experimental setup.} We adopt horizontal flip, random crop, and color (brightness, contrast, and saturation) jitter for the data augmentations. For CIFAR-10 and CIFAR-100, we use an 18-layer PreAct-ResNet and train it using SGD with a momentum of 0.9, a weight decay of 0.0005, and a batch size of 128. The network is trained for 300 epochs. We set the initial learning rate as 0.02, and reduce it by a factor of 10 after 150 epochs. For mini-Webvision, we use ResNet-50 and train it using SGD with a momentum of 0.9, a weight decay of 0.0005, and a batch size of 64. We set the initial learning rate as 0.01 and reduce it by a factor of 10 after 50 epochs (for a total 90 epochs). We train MW-Net using Adam with a learning rate of 0.001, a weight decay of 0.001.
To make our results comparable to the existing literature, we introduce pseudo labels to correct noisy labels to possibly make sufficient use of samples as done by M-correction, PENCIL and DivideMix, and we add a warm-up self-supervised pre-training step and impose a data augmentation based consistency regularization as in C2D and AugDesc to boost our method.

\subsection{Transition Matrix Estimation}
\noindent\textbf{DAC-MR amelioration manner for the task.} We consider the label noise generation process studied in previous works \cite{patrini2017making,shu2020meta}.
Specifically, the clean class-posterior $P(Y|X = x) \triangleq [P(Y=1|X = x),\cdots, (Y=C|X = x)]$ can be inferred by utilizing the noisy class-posterior $P(\tilde{Y}|X = x)$ and the transition matrix $\mathbf{T}\in [0,1]^{C\times C}$ where $\mathbf{T}_{ij} = P(\tilde{Y}=j|Y=i)$, i.e.,
\begin{align}\label{eqmatrix}
	p(\widetilde{Y}|X=x) = \mathbf{T}^T p(Y|X=x).
\end{align}
This formulation has facilitated progress to some statistically consistent robust learning methods \cite{patrini2017making,goldberger2016training,hendrycks2018using,shu2020meta}. However, they all heavily rely on the success of estimating transition matrices. Early attempts develop to estimate the transition matrices under the anchor-point assumption \cite{patrini2017making,hendrycks2018using}. However, the violation of the assumption in some cases could lead to a poorly estimated transition matrix and a degenerated classifier.
This motivates the development of algorithms without exploiting anchor points \cite{xia2019anchor,li2021provably}. Here, we consider the SOTA estimator for the transition matrix, VolMinNet \cite{li2021provably}, which requires the volume of the simplex formed by the columns of the transition matrix to be small. The main idea is that the true simplex has the minimum volume in geometry. The objective function is
\begin{align*}
	{\mathbf{T}}^* & = \mathop{\arg\min}_{ \mathbf{T} \in \mathbb{T}} \log \det (\mathbf{T}),  \\
	\text{s.t.,} \ \mathbf{w}^* & = \mathop{\arg\min}_{\mathbf{w}} \frac{1}{m}\sum_{i=1}^m \ell({\mathbf{T}}^*f_{\mathbf{w}}(x_i^{(s)}), y_i^{(s)} ),
\end{align*}
where $\mathbb{T}\! =\! \{\mathbf{T} \!\in\! [0,1]^{C\times C} | \sum_{j=1}^C \mathbf{T}_{ij}\!=\!1, \mathbf{T}_{ii} > \mathbf{T}_{ij}, \forall i \neq j \}$ is the set of diagonally dominant column stochastic matrices, and $\det$ denotes the determinant of matrix. The volume regularization is delicately designed for the estimation of transition matrix. In contrast, we explore to use the problem-agnostic DAC-MR for estimating transition matrix, i.e.,
\begin{align*}
	{\mathbf{T}}^* & = \mathop{\arg\min}_{ \mathbf{T} \in \mathbb{T}} \mathcal{MR}^{dac} (D;\mathbf{w}^*(\mathbf{T}),A ),	  A \in \mathcal{A},  \\
	\text{s.t.,} \ \mathbf{w}^*(\mathbf{T}) & = \mathop{\arg\min}_{\mathbf{w}} \frac{1}{m}\sum_{i=1}^m \ell(\mathbf{T}f_{\mathbf{w}}(x_i^{(s)}), y_i^{(s)} ),
\end{align*}
where $D = \{x_{i}\}_{i=1}^{n}$ are additionally sampled/divided from $\mathcal{D}^{(s)}$. Note that we treat transition matrix $\mathbf{T}$ as meta-representation, and its estimation is obtained by minimizing the DAC-MR on the collected $D$ in a meta-learning manner \cite{hospedales2021meta}.

\noindent\textbf{Datasets.} We evaluate the proposed method on two synthetic noisy datasets: CIFAR-10 and CIFAR-100 \cite{krizhevsky2009learning}, and one real-world noisy dataset: Clothing1M \cite{xiao2015learning}. We conduct experiments with two commonly used types of noise: symmetry flipping \cite{patrini2017making} and pair flipping \cite{han2018co}. For CIFAR-10 and CIFAR-100, we randomly select different 1000 images at each epochs from training set as $D$ to compute DAC-MR. For Clothing1M, we randomly select different 10 images every class at each epochs from training set as $D$ to compute DAC-MR.

\noindent\textbf{Baselines.} We compared with several SOTA transition matrix estimation methods, including Forward \cite{patrini2017making}, T-Revision \cite{xia2019anchor}, Dual-T \cite{yao2020dual} and VolMinNet \cite{li2021provably}.

\noindent\textbf{Experimental setup.} For a fair comparison, we identically follow the practice in \cite{li2021provably}. We adopt horizontal flip, random crop, and color (brightness, contrast, and saturation) jitter for the data augmentations.
We train a ResNet-18 network for CIFAR10, and a ResNet-34 network for CIFAR-100. We train them using SGD with a batch size 128, a momentum 0.9, a weight decay 0.001 and an initial learning rate 0.01. The learning rate is divided by 10 after the 30th and 60th epochs for a total 150 epoch. We adopt Adam with learning rate $10^{-3}$ and weight decay $10^{-4}$ to update the transition matices. For Clothing1M, we train a ResNet-50 pre-trained on ImageNet using SGD with a momentum 0.9, a weight decay 0.001, a batch size 32 and a learning rates $0.002$. We set the learning rates as $0.00002$ after 5-th epoch for a total 10 epoch. We adopt Adam with learning rate $10^{-4}$ and weight decay $10^{-5}$ to update the transition matices.

\end{document}